\documentclass{article}

\usepackage{times}
\usepackage{graphicx} 
\usepackage{subfigure}

\usepackage{natbib}

\usepackage{algorithm}
\usepackage{algorithmic}
\usepackage[algo2e,linesnumbered,lined,boxed,vlined]{algorithm2e}

\usepackage{amsfonts}
\usepackage{amssymb}

\usepackage{enumitem}

\usepackage{hyperref}

\usepackage[accepted]{icml2016}

\usepackage{comment}
\usepackage{amsthm}
\usepackage{mathtools}
\usepackage{tabularx}
\usepackage{multirow}

\newtheorem{theorem}{Theorem}
\newtheorem{lemma}{Lemma}

\newtheorem{corollary}{Corollary}
\newtheorem{definition}{Definition}

\newtheorem{remark}{Remark}

\def\b{\ensuremath\boldsymbol}

\usepackage{lastpage}
\usepackage{fancyhdr}
\usepackage{url}
\icmltitlerunning{}

\begin{document}

\AddToShipoutPictureBG*{%
  \AtPageUpperLeft{%
    \setlength\unitlength{1in}%
    \hspace*{\dimexpr0.5\paperwidth\relax}
    \makebox(0,-0.75)[c]{\normalsize {\color{black} To appear as a part of an upcoming textbook on dimensionality reduction and manifold learning.}}
    }}

\twocolumn[
\icmltitle{Reproducing Kernel Hilbert Space, Mercer's Theorem, Eigenfunctions, Nystr{\"o}m Method, and Use of Kernels in Machine Learning: \\Tutorial and Survey}

\icmlauthor{Benyamin Ghojogh}{bghojogh@uwaterloo.ca}
\icmladdress{Department of Electrical and Computer Engineering, 
\\Machine Learning Laboratory, University of Waterloo, Waterloo, ON, Canada}
\icmlauthor{Ali Ghodsi}{ali.ghodsi@uwaterloo.ca}
\icmladdress{Department of Statistics and Actuarial Science \& David R. Cheriton School of Computer Science, 
\\Data Analytics Laboratory, University of Waterloo, Waterloo, ON, Canada}
\icmlauthor{Fakhri Karray}{karray@uwaterloo.ca}
\icmladdress{Department of Electrical and Computer Engineering, 
\\Centre for Pattern Analysis and Machine Intelligence, University of Waterloo, Waterloo, ON, Canada}
\icmlauthor{Mark Crowley}{mcrowley@uwaterloo.ca}
\icmladdress{Department of Electrical and Computer Engineering, 
\\Machine Learning Laboratory, University of Waterloo, Waterloo, ON, Canada}

\icmlkeywords{Tutorial}

\vskip 0.3in
]

\begin{abstract}
This is a tutorial and survey paper on kernels, kernel methods, and related fields. We start with reviewing the history of kernels in functional analysis and machine learning. Then, Mercer kernel, Hilbert and Banach spaces, Reproducing Kernel Hilbert Space (RKHS), Mercer's theorem and its proof, frequently used kernels, kernel construction from distance metric, important classes of kernels (including bounded, integrally positive definite, universal, stationary, and characteristic kernels), kernel centering and normalization, and eigenfunctions are explained in detail. Then, we introduce types of use of kernels in machine learning including kernel methods (such as kernel support vector machines), kernel learning by semi-definite programming, Hilbert-Schmidt independence criterion, maximum mean discrepancy, kernel mean embedding, and kernel dimensionality reduction. We also cover rank and factorization of kernel matrix as well as the approximation of eigenfunctions and kernels using the Nystr{\"o}m method. This paper can be useful for various fields of science including machine learning, dimensionality reduction, functional analysis in mathematics, and mathematical physics in quantum mechanics. 
\end{abstract}

\section{Introduction}\label{section_introduction}

\subsection{History of Kernels}


It is 1904 when David Hilbert proposed his work on kernels and defined a definite kernel \cite{hilbert1904grundzuge}, and later, his and Erhard Schmidt's works proposed integral equations such as Fredholm integral equations \cite{hilbert1904grundzuge,schmidt1908auflosung}. This was the introduction of a new space which was named the Hilbert space. Later, the Hilbert space was found to be very useful for the formulations in quantum mechanics \cite{prugovecki1982quantum}. 
After the initial works on Hilbert space by Hilbert and Schmidt \cite{hilbert1904grundzuge,schmidt1908auflosung}, James Mercer improved Hilbert's work and proposed his theorem in 1909 \cite{mercer1909functions} which was named the Mercer's theorem later. 
In the mean time, Stefan Banach, Hans Hahn, and Eduard Helly proposed the concepts of another new space in years 1920–1922 \cite{bourbaki1950certains} which was named the Banach space later by Maurice Ren{\'e} Fr{\'e}chet \cite{narici2010topological}. 
The Hilbert space is a subset of the Banach space. 

Reproducing Kernel Hilbert Space (RKHS) is a special case of Hilbert space with some properties. It is a Hilbert space of functions with reproducing kernels \cite{berlinet2011reproducing}.
The first work on RKHS was \cite{aronszajn1950theory}. Later, the concepts of RKHS were improved further in \cite{aizerman1964theoretical}.
The RKHS remained in pure mathematics until this space was used for the first time in machine learning by introduction of kernel Support Vector Machine (SVM) \cite{boser1992training,vapnik1995nature}.
Eigenfunctions were also developed for eigenvalue problem applied on operators and functions \cite{williams2000effect} and were used in machine learning \cite{bengio2003spectral} and physics \cite{kusse2006mathematical}. This is related to RKHS because it uses weighted inner product in Hilbert space \cite{williams2000effect} and RKHS is a Hilbert space of functions with a reproducing kernel. 

Using kernels was widely noticed when linear SVM \cite{vapnik1974theory} was kernelized \cite{boser1992training,vapnik1995nature}.
Kernel SVM showed off very successfully because of its merits. 
Two competing models, kernel SVM and neural network were the models which could handle nonlinear data (see \cite{fausett1994fundamentals} for history of neural networks). Kernel SVM transformed nonlinear data to RKHS to make the pattern of data linear hopefully and then applied linear SVM on it. However, the approach of neural network was different because the model itself was nonlinear (see Section \ref{section_kernelization_techniques} for more details). 
The success of kernel SVM plus the problem of vanishing gradients in neural networks \cite{goodfellow2016deep} resulted in the winter of neural network around years 2000 to 2006. However, the problems of training deep neural networks started to be resolved \cite{hinton2006reducing} and their success plus two problems of kernel SVM helped neural networks take over kernel SVM gradually. One problem of kernel SVM was not knowing the suitable kernel type for various learning problems. 
In other words, kernel SVM still required the user to choose the type of kernel but neural networks were end-to-end and almost robust to hyperparameters such as the number of layers or neurons. 
Another problem was that kernel SVM could not handle big data, although Nystr{\"o}m method, first proposed in \cite{nystrom1930praktische}, was used to resolve this problem of kernel methods by approximating kernels from a subset of data \cite{williams2001using}. 
Note that kernels have been used widely in machine learning such as in SVM \cite{vapnik1995nature},
Gaussian process classifiers \cite{williams1998bayesian}, and spline methods \cite{wahba1990spline}. 
The types of use of kernels in machine learning will be discussed in Section \ref{section_kernel_in_machine_learning}.

\subsection{Useful Books on Kernels}

There exist several books about use of kernels in machine learning. Some examples are \cite{smola1998learning,scholkopf1999advances,scholkopf2002learning,shawe2004kernel,camps2006kernel,steinwart2008support,rojo2018digital,kung2014kernel}. Some survey papers about kernel-based machine learning are \cite{hofmann2006review,hofmann2008kernel,muller2018introduction}.
In addition to some of the above-mentioned books, there exist some other books/papers on kernel SVM such as \cite{scholkopf1997comparing,burges1998tutorial,hastie2009elements}.



\subsection{Kernel in Different Fields of Science}

The term kernel has been used in different fields of science for various purposes. In the following, we briefly introduce the different uses of kernel in science to clarify which use of kernel we are focusing on in this paper. 
\begin{enumerate}[topsep=0pt,itemsep=-1ex,partopsep=1ex,parsep=1ex]
\item \textbf{Kernel of filter} in signal processing: In signal processing, one can use filters to filter a price of signal, such as an image \cite{gonzalez2002digital}. Digital filters have a kernel which determine the values of filter in the window of filter \cite{schlichtharle2011digital}. In convolutional neural networks, the filter kernels are learned in deep learning \cite{goodfellow2016deep}. 
\item \textbf{Kernel smoothing} for density estimation: Kernel density estimation can be used for fitting a mixture of distributions to some data instances \cite{scott1992multivariate}. For this, a histogram with infinite number of bins is utilized. In limit, this histogram is converged to a kernel smoothing \cite{wand1994kernel} where the kernel determines the type of distribution. For example, if a Radial Basis Function (RBF) kernel is used, a mixture of Gaussian distributions is fitted to data. 
\item Kernelization in \textbf{complexity theory}: Kernelization is a pre-processing technique where the input to an algorithm is replaced by a part of the input named kernel. The output of the algorithm on kernel should either be the same as or be able to be transformed to the output of the algorithm for the whole input \cite{fomin2019kernelization}. An example usage of kernelization is in vertex cover problem \cite{abu2004kernelization}. 
\item Kernel in \textbf{operating system}: Kernel is the core of an operating system, such as Linux, which connects the hardware including CPU, memory, and peripheral devices to applications \cite{anderson2014operating}. 
\item Kernel in \textbf{linear algebra} and graphs: Consider a mapping from the vector space $\mathcal{V}$ to the vector space $\mathcal{W}$ as $L: \mathcal{V} \rightarrow \mathcal{W}$. The kernel, also called the nullspace, of this mapping is defined as $\text{ker}(L) := \{\b{v} \in \mathcal{V}\, |\, L(\b{v}) = \b{0}\}$. For example, for a matrix $\b{A} \in \mathbb{R}^{a \times b}$, the kernel of $\b{A}$ is $\text{ker}(\b{A}) = \{\b{x} \in \mathbb{R}^n\, |\, \b{A}\b{x} = \b{0} \}$. 
The four fundamental subspaces of a matrix are its kernel, row space, column space, and left null space \cite{strang1993fundamental}.
Note that the kernel (nullspace) of adjacency matrix in graphs has also been well developed \cite{akbari2006kernels}. 
\item Kernel in other domains of mathematics: There exist kernel concepts in other domains of mathematics and statistics such as \textbf{geometry of polygon} \cite{icking1995searching}, \textbf{set theory} {\citep[p. 14]{bergman2011universal}}, etc. 
\item Kernel in feature space for \textbf{machine learning}: In statistical machine learning, kernels pull data to a feature space for the sake of better discrimination of classes or simpler representation of data \cite{hofmann2006review,hofmann2008kernel}. In this paper, our focus is on this category which is kernels for machine learning. 
\end{enumerate}

\subsection{Organization of Paper}

This paper is a tutorial and survey paper on kernels and kernel methods. It can be useful for several fields of science including machine learning, functional analysis in mathematics, and mathematical physics in quantum mechanics. 
The remainder of this paper is organized as follows. Section \ref{section_Mercer_kernel_and_spaces} introduces the Mercer kernel, important spaces in functional analysis including the Hilbert and Banach spaces, and Reproducing Kernel Hilbert Space (RKHS). Mercer's theorem and its proof are provided in Section \ref{section_Mercer_theorem}. Characteristics of kernels are explained in Section \ref{section_kernel_characteristics}. We introduce frequently used kernels, kernel construction from distance metric, and important classes of kernels in Section \ref{section_well_known_kernels}. Kernel centering and normalization are explained in Section \ref{section_kernel_centering_normalization}. Eigenfunctions are then introduced in Section \ref{section_eigenfunctions}. We explain two techniques for kernelization in Section \ref{section_kernelization_techniques}. Types of use of kernels in machine learning are reviewed in Section \ref{section_kernel_in_machine_learning}. Kernel factorization and Nystr{\"o}m approximation are introduced in Section \ref{section_factorization_and_Nystrom_method}. Finally, Section \ref{section_conclusion} concludes the paper. 

\section*{Required Background for the Reader}

This paper assumes that the reader has general knowledge of calculus, linear algebra, and basics of optimization. 
The required basics of functional analysis are explained in the paper. 

\section{Mercer Kernel and Spaces In Functional Analysis}\label{section_Mercer_kernel_and_spaces}


\subsection{Mercer Kernel and Gram Matrix}

\begin{definition}[Mercer Kernel \cite{mercer1909functions}]\label{definition_Mercer_kernel}
The function $k: \mathcal{X}^2 \rightarrow \mathbb{R}$ is a Mercer kernel function (also known as kernel function) where:
\begin{enumerate}[topsep=0pt,itemsep=-1ex,partopsep=1ex,parsep=1ex]
\item it is symmetric: $k(\b{x}, \b{y}) = k(\b{y}, \b{x})$,
\item and its corresponding kernel matrix $\b{K}(i,j) = k(\b{x}_i, \b{x}_j), \forall i,j \in \{1, \dots, n\}$ is positive semi-definite: $\b{K} \succeq 0$. 
\end{enumerate}
The corresponding kernel matrix of a Mercer kernel is a Mercer kernel matrix. 
\end{definition}
The two properties of a Mercer kernel will be proved in Section \ref{section_kernel_characteristics}.
By convention, unless otherwise stated, the term kernel refers to Mercer kernel. 
The effectiveness of Mercer kernel will be shown and proven in the Mercer's theorem, i.e., Theorem \ref{theorem_Mercer}.

\begin{definition}[Gram Matrix or Kernel Matrix]\label{definition_Gram_matrix}
The matrix $\b{K} \in \mathbb{R}^{n \times n}$ is a Gram matrix, also known as a Gramian matrix or a kernel matrix, whose $(i,j)$-th element is:
\begin{align}\label{equation_Gram_matrix}
\b{K}(i,j) := k(\b{x}_i, \b{x}_j), \quad \forall i,j \in \{1, \dots, n\}.
\end{align}
\end{definition}
Here, we defined the square kernel matrix applied on a set of $n$ data instances; hence, the kernel is a $n \times n$ matrix. We may also have a kernel matrix between two sets of data instances. This will be explained more in Section \ref{section_kernelization_techniques}. 
Moreover, note that the kernel matrix can be computed using the inner product between pulled data to the feature space. This will be explained in detail in Section \ref{section_feature_map}.

\subsection{Hilbert, Banach, $L_p$, and Sobolev Spaces}\label{section_spaces}

Before defining the RKHS and details of kernels, we need to introduce Hilbert, Banach, $L_p$, and Sobolev spaces, which are well-known spaces in functional analysis \cite{conway2007course}. 

\begin{definition}[Metric Space]
A metric space is a set where a metric, for measuring the distance between instances of set, is defined on it. 
\end{definition}

\begin{definition}[Vector Space]
A vector space is a set of vectors equipped with some available operations such as addition and multiplication by scalars. 
\end{definition}


\begin{definition}[Complete Space]
A space $\mathcal{F}$ is complete if every Cauchy sequence converges to a member of this space $f \in \mathcal{F}$. Note that the Cauchy sequence is a sequence whose elements become arbitrarily close to one another as the sequence progresses (i.e., it converges in limit). 
\end{definition}

\begin{definition}[Compact Space]\label{definition_compact_space}
A space is compact if it is closed (i.e., it contains all its limit points) and bounded (i.e., all its points lie within some fixed distance of one another). 
\end{definition}


\begin{definition}[Hilbert Space \cite{reed1972methods}]\label{definition_Hilbert_space}
A Hilbert space $\mathcal{H}$ is an inner product space that is a complete metric space with respect to the norm or distance function induced by the inner product. 
\end{definition}
The Hilbert space generalizes the Euclidean space to a finite or infinite dimensional space. Usually, the Hilbert space is high dimensional. By convention in machine learning, unless otherwise stated, Hilbert space is also referred to as the \textit{feature space}. By feature space, researchers often specifically mean the RKHS space which will be introduced in Section \ref{section_RKHS}.

\begin{definition}[Banach Space \cite{beauzamy1982introduction}]
A Banach space is a complete vector space equipped with a norm. 
\end{definition}

\begin{remark}[Difference of Hilbert and Banach Spaces]
Hilbert space is a special case of Banach space equipped with a norm defined using an inner product notion. All Hilbert spaces are Banach spaces but the converse is not true. 
\end{remark}

Suppose $\mathbb{R}^n$, $\mathcal{H}$, $\mathbb{B}$, $\mathcal{M}_c$, $\mathcal{M}$, $\mathbb{T}$ denote the Euclidean space, Hilbert space, Banach space, complete metric space, metric space, and topological space (containing both open and closed sets), respectively. Then, we have:
\begin{align}
\mathbb{R}^n \subset \mathcal{H} \subset \mathcal{B} \subset \mathcal{M}_c \subset \mathcal{M} \subset \mathbb{T}.
\end{align}

\begin{definition}[$L_p$ Space]
Consider a function $f$ with domain $[a, b]$. 
For $p >0$, let the $L_p$ norm be defined as:
\begin{align}
\|f\|_p := \Big(\int |f(\b{x})|^p\, d\b{x}\Big)^{\frac{1}{p}}.
\end{align}
The $L_p$ space is defined as the set of functions with bounded $L_p$ norm:
\begin{align}
L_p(a,b) := \{f:[a,b] \rightarrow \mathbb{R}\,\, |\,\, \|f\|_p < \infty \}.
\end{align}
\end{definition}

\begin{definition}[Sobolev Space {\citep[Chapter 7]{renardy2006introduction}}]
A Sobolev space is a vector space of functions equipped with $L_p$ norms and derivatives:
\begin{align}
\mathcal{W}_{m,p} := \{f \in L_p(0,1)\, |\, D^m f \in L_p(0,1) \},
\end{align}
where $D^m f$ denotes the $m$-th order derivative. 
\end{definition}
Note that the Sobolev spaces are RKHS with some specific kernels \cite{novak2018reproducing}. RKHS will be explained in Section \ref{section_RKHS}.

\subsection{Reproducing Kernel Hilbert Space}\label{section_RKHS}

\subsubsection{Definition of RKHS}

Reproducing Kernel Hilbert Space (RKHS), first proposed in \cite{aronszajn1950theory}, is a special case of Hilbert space with some properties. It is a Hilbert space of functions with reproducing kernels \cite{berlinet2011reproducing}. 
After the initial work on RKHS \cite{aronszajn1950theory}, another work \cite{aizerman1964theoretical} developed the RKHS concepts.
In the following, we introduce this space. 

\begin{definition}[RKHS \cite{aronszajn1950theory,berlinet2011reproducing}]\label{definition_RKHS}
A Reproducing Kernel Hilbert Space (RKHS) is a Hilbert space $\mathcal{H}$ of functions $f: \mathcal{X} \rightarrow \mathbb{R}$ with a reproducing kernel $k: \mathcal{X}^2 \rightarrow \mathbb{R}$ where $k(\b{x}, .) \in \mathcal{H}$ and $f(\b{x}) = \langle\,k(\b{x},.),f\rangle$. 
\end{definition}

The RKHS is explained in more detail in the following. 
Consider the kernel function $k(\b{x}, \b{y})$ which is a function of two variables. 
Suppose, for $n$ points, we fix one of the variables to have $k(\b{x}_1, \b{y}), k(\b{x}_2, \b{y}), \dots, k(\b{x}_n, \b{y})$.
These are all functions of the variable $\b{y}$. 
RKHS is a function space which is the set of all possible linear combinations of these functions \cite{kimeldorf1971some}, {\citep[p. 834]{aizerman1964theoretical}}, \cite{mercer1909functions}:
\begin{align}\label{equation_RKHS}
\mathcal{H} :=\! \Big\{f(.) = \sum_{i=1}^n \alpha_i\, k(\b{x}_i, .)\Big\} \!\overset{(a)}{=}\! \Big\{f(.) = \sum_{i=1}^n \alpha_i\, k_{\b{x}_i}(.)\Big\},
\end{align}
where $(a)$ is because we define $k_{\b{x}}(.) := k(\b{x}, .)$.
This equation shows that the bases of an RKHS are kernels.
The proof of this equation is obtained by considering both Eqs. (\ref{equation_Mercer_theorem_kernel_representation}) and (\ref{equation_kernel_inner_product}) together (n.b. for better organization, it is better that we provide those equations later). 
It is also noteworthy that this equation will also appear in Theorem \ref{theorem_representer_theorem}.

According to Eq. (\ref{equation_RKHS}), every function in the RKHS can be written as a linear combination. Consider two functions in this space represented as $f = \sum_{i=1}^n \alpha_i\, k(\b{x}_i, \b{y})$ and $g = \sum_{j=1}^n \beta_j\, k(\b{x}, \b{y}_j)$. Hence, the inner product in RKHS is calculated as:
\begin{align}
\langle f,g \rangle_k &\overset{(\ref{equation_RKHS})}{=} \Big\langle \sum_{i=1}^n \alpha_i\, k(\b{x}_i, .), \sum_{j=1}^n \beta_j\, k(\b{y}_j, .) \Big\rangle_k \nonumber \\
&\overset{(a)}{=} \Big\langle \sum_{i=1}^n \alpha_i\, k(\b{x}_i, .), \sum_{j=1}^n \beta_j\, k(., \b{y}_j) \Big\rangle_k \nonumber \\
&= \sum_{i=1}^n \sum_{j=1}^n \alpha_i\, \beta_j\, k(\b{x}_i, \b{y}_j), \label{equation_RKHS_inner_product}
\end{align}
where $(a)$ is because kernel is symmetric (it will be proved in Section \ref{section_kernel_characteristics}).
Hence, the norm in RKHS is calculated as:
\begin{align}
&\|f\|_k := \sqrt{\langle f,f \rangle_k}.
\end{align}
The subscript of norm and inner product in RKHS has various notations in the research papers. Some most famous notations are $\langle f,g \rangle_k$, $\langle f,g \rangle_\mathcal{H}$, $\langle f,g \rangle_{\mathcal{H}_k}$, $\langle f,g \rangle_\mathcal{F}$ where $\mathcal{H}_k$ denotes the Hilbert space associated with kernel $k$ and $\mathcal{F}$ stands for the feature space because RKHS is sometimes referred to as the feature space.

\begin{remark}[RKHS Being Unique for a Kernel]
Given a kernel, the corresponding RKHS is unique (up to isometric isomorphisms). Given an RKHS, the corresponding kernel is unique. In other words, each kernel generates a new RKHS. 
\end{remark}

\begin{remark}
As we also saw in Mercer's theorem, the bases of RKHS space is the eigenfunctions $\{\psi_i(.)\}_{i=1}^\infty$ which are functions themselves. This, along with Eq. (\ref{equation_RKHS}), show that the RKHS space is a space of functions and not a space of vectors. In other words, the basis vectors of RKHS are basis functions named eigenfunctions. 
Because the RKHS is a space of functions rather than a space of vectors, we usually do not know the exact location of pulled points to the RKHS but we know the relation of them as a function. This will be explained more in Section \ref{section_feature_map} and Fig. \ref{figure_pulling}. 
\end{remark}

\subsubsection{Reproducing Property}

In Eq. (\ref{equation_RKHS_inner_product}), consider only one component for $g$ to have $g(\b{x}) = \sum_{j=1}^n \beta_j\, k(\b{x}_i, \b{x}) = \beta k(\b{x}, \b{x})$ where we take $\beta = 1$ to have $g(\b{x}) = k(\b{x}, \b{x}) = k_{\b{x}}(.)$. 
In other words, assume the function $g(\b{x})$ is a kernel in the RKHS space. 
Also consider the function $f(\b{x}) = \sum_{i=1}^n \alpha_i\, k(\b{x}_i, \b{x})$ in the space. According to  Eq. (\ref{equation_RKHS_inner_product}), the inner product of these functions is:
\begin{align}
\langle f(\b{x}), g(\b{x}) \rangle_k &= \langle f, k_{\b{x}}(.) \rangle_k \nonumber \\
&= \Big\langle \sum_{i=1}^n \alpha_i\, k(\b{x}_i, \b{x}), k(\b{x}, \b{x}) \Big\rangle_k \nonumber \\
&= \sum_{i=1}^n \alpha_i\, k(\b{x}_i, \b{x}) \overset{(a)}{=} f(\b{x}), \label{equation_RKHS_reproducing}
\end{align}
where $(a)$ is because the term before equation is the previously considered function $f$. 
As Eq. (\ref{equation_RKHS_reproducing}) shows, the function $f$ is reproduced from the inner product of that function with one of the kernels in the space. This shows the reproducing property of the RKHS space. 
A special case of Eq. (\ref{equation_RKHS_reproducing}) is $\langle k_{\b{x}}, k_{\b{x}} \rangle_k = k(\b{x}, \b{x})$. 

The name of RKHS consists of several parts becuase of the following justifications:
\begin{enumerate}
\item ``Reproducing": because if the reproducing property of RKHS which was proved above. 
\item ``Kernel": because of the kernels associated to RKHS as stated in Definition \ref{definition_RKHS} and Eq. (\ref{equation_RKHS}). 
\item ``Hilbert Space": because RKHS is a Hilbert space of functions with a reproducing kernel, as stated in Definition \ref{definition_RKHS}.
\end{enumerate}

\subsubsection{Representation in RKHS}

In the following, we provide a proof for Eq. (\ref{equation_RKHS}) and explain why that equation defines the RKHS. 
\begin{theorem}[Representer Theorem \cite{kimeldorf1971some}, simplified in \cite{rudin2012prediction}]\label{theorem_representer_theorem}
For a set of data $\mathcal{X} = \{\b{x}_i\}_{i=1}^n$, consider a RKHS $\mathcal{H}$ of functions $f: \mathcal{X} \rightarrow \mathbb{R}$ with kernel function $k$.
For any function $\ell: \mathbb{R}^2 \rightarrow \mathbb{R}$ (usually called the loss function), consider the optimization problem:
\begin{align}
f^* \in \arg\min_{f \in \mathcal{H}}\, \sum_{i=1}^n \ell(f(\b{x}_i), \b{y}_i) + \eta\, \Omega(\|f\|_k),
\end{align}
where $\eta \geq 0$ is the regularization parameter and $\Omega(\|f\|_k)$ is a penalty term such as $\|f\|_k^2$. 
The solution of this optimization can be expressed as:
\begin{align}\label{equation_representer_theorem}
f^* = \sum_{i=1}^n \alpha_i\, k(\b{x}_i, .) = \sum_{i=1}^n \alpha_i\, k_{\b{x}_i}(.).
\end{align}
P.S.: Eq. (\ref{equation_representer_theorem}) can also be seen in {\citep[p. 834]{aizerman1964theoretical}}.
\end{theorem}
\begin{proof}
Proof is inspired by \cite{rudin2012prediction}. 
Assume we project the function $f$ onto a subspace spanned by $\{k(\b{x}_i, .)\}_{i=1}^n$. 
The function $f$ can be decomposed into components along and orthogonal to this subspace, respectively denoted by $f_{\|}$ and $f_{\perp}$: 
\begin{align}
f &= f_{\|} + f_{\perp} \label{equation_rep_theorey_proof1} \\
&\implies \|f\|_k^2 = \|f_{\|}\|_k^2 + \|f_{\perp}\|_k^2 \geq \|f_{\|}\|_k^2. \label{equation_rep_theorey_proof2}
\end{align}
Moreover, using the reproducing property of the RKHS, we have:
\begin{align}
f(\b{x}_i) &= \langle f, k(\b{x}_i, .) \rangle_k \nonumber \\
&\overset{(\ref{equation_rep_theorey_proof1})}{=} \langle f_{\|}, k(\b{x}_i, .) \rangle_k + \langle f_{\perp}, k(\b{x}_i, .) \rangle_k \nonumber \\
&\overset{(a)}{=} \langle f_{\|}, k(\b{x}_i, .) \rangle_k \overset{(\ref{equation_RKHS_reproducing})}{=} f_{\|}(\b{x}_i), \label{equation_rep_theorey_proof3}
\end{align}
where $(a)$ is because the orthogonal component has zero inner product with the bases of subspace. 
According to Eq. (\ref{equation_rep_theorey_proof3}), we have:
\begin{align}\label{equation_rep_theorey_proof4}
\sum_{i=1}^n \ell(f(\b{x}_i), \b{y}_i) = \sum_{i=1}^n \ell(f_{\|}(\b{x}_i), \b{y}_i).
\end{align}
Using Eqs. (\ref{equation_rep_theorey_proof2}) and (\ref{equation_rep_theorey_proof4}), we can say:
\begin{align*}
\min_{f \in \mathcal{H}}\,& \sum_{i=1}^n \ell(f(\b{x}_i), \b{y}_i) + \eta\, \|f\|_k^2 \\
&= \min_{f \in \mathcal{H}}\, \sum_{i=1}^n \ell(f_{\|}(\b{x}_i), \b{y}_i) + \eta\, \|f_{\|}\|_k^2.
\end{align*}
Hence, for this minimization, we only require the component lying in the space spanned by the kernels of RKHS. Therefore, we can represent the function (solution of optimization) to lie in the space as linear combination of basis vectors $\{k(\b{x}_i, .)\}_{i=1}^n$. Q.E.D.
\end{proof}

\begin{corollary}
In Section \ref{section_spaces}, we mentioned that Hilbert space can be infinite dimensional. According to Definition \ref{definition_RKHS}, RKHS is a Hilbert space so it may be infinite dimensional. The representer theorem states that, in practice, we only need to deal with a finite-dimensional space; although, that finite number of dimensions is usually a large number. 
\end{corollary}

Note that the representer theorem has been used in kernel SVM where $\alpha_i$'s are the dual variables which are non-zero for support vectors \cite{boser1992training,vapnik1995nature}. According to this theorem, kernel SVM only requires to learn the dual variables, $\alpha_i$'s, to find the optimal boundary between classes.

\section{Mercer's Theorem and Feature Map}\label{section_Mercer_theorem}

\subsection{Mercer's Theorem}

\begin{definition}[Definite Kernel \cite{hilbert1904grundzuge}]
A kernel $k: [a,b] \times [a,b] \rightarrow \mathbb{R}$ is a definite kernel where the following double integral:
\begin{align}
J(f) = \int_a^b \int_a^b k(\b{x}, \b{y}) f(\b{x}) f(\b{y})\, d\b{x}\, d\b{y},
\end{align}
satisfies $J(f) > 0$ for all $f(\b{x}) \neq 0$. 
\end{definition}
Mercer improved over Hilbert's work \cite{hilbert1904grundzuge} to propose his theorem, the Mercer's theorem \cite{mercer1909functions}, introduced in the following. 


\begin{theorem}[Mercer's Theorem \cite{mercer1909functions}]\label{theorem_Mercer}
Suppose $k: [a,b] \times [a,b] \rightarrow \mathbb{R}$ is a continuous symmetric positive semi-definite kernel which is bounded:
\begin{align}\label{equation_Mercer_theorem_sup_kernel}
\sup_{\b{x}, \b{y}}\, k(\b{x}, \b{y}) < \infty.
\end{align}
Assume the operator $T_k$ takes a function $f(\b{x})$ as its argument and outputs a new function as:
\begin{align}\label{equation_Mercer_theorem_T_operator}
T_k f(\b{x}) := \int_a^b k(\b{x}, \b{y}) f(\b{y})\, d\b{y},
\end{align}
which is a Fredholm integral equation \cite{schmidt1908auflosung}. 
The operator $T_k$ is called the Hilbert–Schmidt integral operator {\citep[Chapter 8]{renardy2006introduction}}.
This output function is positive semi-definite:
\begin{align}
\int\!\!\! \int k(\b{x}, \b{y}) f(\b{y})\, d\b{x}\, d\b{y} \geq 0.
\end{align}
Then, there is a set of orthonormal bases $\{\psi_i(.)\}_{i=1}^\infty$ of $L_2(a,b)$ consisting of eigenfunctions of $T_K$ such that the corresponding sequence of eigenvalues $\{\lambda_i\}_{i=1}^\infty$ are non-negative: 
\begin{align}\label{equation_Mercer_theorem_eigenfunction_decomposition}
\int k(\b{x}, \b{y})\, \psi_i(\b{y})\, d\b{y} = \lambda_i\, \psi_i(\b{x}).
\end{align}
The eigenfunctions corresponding to the non-zero eigenvalues are continuous on $[a,b]$ and $k$ can be represented as \cite{aizerman1964theoretical}:
\begin{align}\label{equation_Mercer_theorem_kernel_representation}
k(\b{x}, \b{y}) = \sum_{i=1}^\infty \lambda_i\, \psi_i(\b{x})\, \psi_i(\b{y}),
\end{align}
where the convergence is absolute and uniform.
\end{theorem}
\begin{proof}
A roughly high-level proof for the Mercer's theorem is as follows. 



\textbf{Step 1 of proof:}
According to assumptions of theorem, the Hilbert-Schmidt integral operator $T_k$ is a symmetric operator on $L_2(a,b)$ space. 
Consider a unit ball in $L_2(a,b)$ as input to the operator.
As the kernel is bounded, $\sup_{\b{x}, \b{y}}\, k(\b{x}, \b{y}) < \infty$, the sequence $f_1, f_2, \dots$ converges in norm, i.e. $\|f_n - f\| \rightarrow 0$ as $n \rightarrow 0$. 
Therefore, according to the Arzel{\`a}-Ascoli theorem \cite{arzela1895sulle}, the image of the unit ball after applying the operator is compact. In other words, the operator $T_k$ is compact. 

\textbf{Step 2 of proof:}
According to the spectral theorem \cite{hawkins1975cauchy}, there exist several orthonormal bases $\{\psi_i(.)\}_{i=1}^\infty$ in $L_2(a,b)$ for the compact operator $T_k$. This provides a spectral (or eigenvalue) decomposition for the operator $T_k$ \cite{ghojogh2019eigenvalue}:
\begin{align}\label{equation_Mercer_theorem_eigenvalue_decomposition}
T_k \psi_i(\b{x}) = \lambda_i\, \psi_i(\b{x}),
\end{align}
where $\{\psi_i(.)\}_{i=1}^\infty$ and $\{\lambda_i\}_{i=1}^\infty$ are the eigenvectors and eigenvalues of the operator $T_k$, respectively. 
Noticing the defined Eq. (\ref{equation_Mercer_theorem_T_operator}) and the eigenvalue decomposition, Eq. (\ref{equation_Mercer_theorem_eigenvalue_decomposition}), we have: 
\begin{align}
\int k(\b{x}, \b{y})\, \psi_i(\b{y})\, d\b{y} \overset{(\ref{equation_Mercer_theorem_T_operator})}{=} T_k \psi_i(\b{x})  \overset{(\ref{equation_Mercer_theorem_eigenvalue_decomposition})}{=} \lambda_i\, \psi_i(\b{x}).
\end{align}
This proves the Eq. (\ref{equation_Mercer_theorem_eigenfunction_decomposition}) which is the eigenfunction decomposition of the operator $T_k$. Note that the eigenvectors $\{\psi_i(.)\}_{i=1}^\infty$ are referred to as the \textit{eigenfunctions} because the decomposition is applied on a function or operator rather than a matrix. 
Note that eigenfunctions will be explained more in Section \ref{section_eigenfunctions}.

\textbf{Step 3 of proof:} 
According to Parseval's theorem \cite{parseval1806memoires}, the Bessel's inequality can be converted to equality \cite{saxe2002beginning}. For the orthonormal bases $\{\psi_i(.)\}_{i=1}^\infty$ in the Hilbert space $\mathcal{H}$ associated with kernel $k$, we have for any function $f \in L_2(a,b)$:
\begin{align}\label{equation_Mercer_theorem_f_Parseval}
f = \sum_{i=1}^\infty \langle f, \psi_i \rangle_k\, \psi_i.
\end{align}
If we replace $\psi_i$ with $f$ in Eq. (\ref{equation_Mercer_theorem_eigenvalue_decomposition}) and consider Eq. (\ref{equation_Mercer_theorem_f_Parseval}), we will have:
\begin{align}\label{equation_Mercer_theorem_T_as_sum}
T_k f = \sum_{i=1}^\infty \lambda_i \langle f, \psi_i \rangle_k\, \psi_i.
\end{align}
One can consider Eq. (\ref{equation_Mercer_theorem_T_operator}) as $T_k f = k f$. Noticing this and Eq. (\ref{equation_Mercer_theorem_T_as_sum}) results in:
\begin{align}\label{equation_Mercer_theorem_k_f_as_sum}
k f = \sum_{i=1}^\infty \lambda_i \langle f, \psi_i \rangle_k\, \psi_i.
\end{align}
Ignoring $f$ from Eq. (\ref{equation_Mercer_theorem_k_f_as_sum}) gives:
\begin{align}
k(\b{x}, \b{y}) = \sum_{i=1}^\infty \lambda_i\, \psi_i(\b{x})\, \psi_i(\b{y}),
\end{align}
which is Eq. (\ref{equation_Mercer_theorem_kernel_representation}); hence, that is proved. 

\textbf{Step 4 of proof:} 
We define the truncated kernel $r_n$ (with parameter $n$) as:
\begin{align}
r_n(\b{x}, \b{y}) &:= k(\b{x}, \b{y}) - \sum_{i=1}^n \lambda_i\, \psi_i(\b{x})\, \psi_i(\b{y}) \nonumber \\
&= \sum_{i={n+1}}^\infty \lambda_i\, \psi_i(\b{x})\, \psi_i(\b{y}).
\end{align}
As $T_k$ is an integral operator, this truncated kernel has positive kernel, i.e., for every $\b{x} \in [a,b]$, we have:
\begin{align}
&r_n(\b{x}, \b{x}) = k(\b{x}, \b{x}) - \sum_{i=1}^n \lambda_i\, \psi_i(\b{x})\, \psi_i(\b{x}) \geq 0 \nonumber \\
&\implies \sum_{i=1}^n \lambda_i\, \psi_i(\b{x})\, \psi_i(\b{x}) \leq k(\b{x}, \b{x}) \leq \sup_{\b{x} \in [a,b]} k(\b{x}, \b{x}). \label{equation_Mercer_theorem_truncated_kernel_inequality}
\end{align}
By Cauchy-Schwartz inequality, we have:
\begin{align*}
&\Big|\sum_{i=1}^n \lambda_i\, \psi_i(\b{x})\, \psi_i(\b{y})\Big|^2 \\
&\leq \Big( \sum_{i=1}^n \lambda_i\, \psi_i(\b{x})\, \psi_i(\b{x}) \Big) \Big( \sum_{i=1}^n \lambda_i\, \psi_i(\b{y})\, \psi_i(\b{y}) \Big) \\
&\overset{(\ref{equation_Mercer_theorem_truncated_kernel_inequality})}{\leq} \sup_{\b{x} \in [a,b]} \big(k(\b{x}, \b{x})\big)^2.
\end{align*}
Taking second root from the sides of inequality gives:
\begin{align}
\sum_{i=1}^n \lambda_i \big|\psi_i(\b{x})\, \psi_i(\b{x})\big| \leq \sup_{\b{x} \in [a,b]} | k(\b{x}, \b{x}) | \overset{(\ref{equation_Mercer_theorem_sup_kernel})}{\leq} \infty.
\end{align}
This shows that the sequence $\sum_{i=1}^n \lambda_i\, \psi_i(\b{x})\, \psi_i(\b{x})$ converges absolutely and uniformly.
Q.E.D.


\end{proof}

The Mercer's theorem is very important. Many of its equations, such as Eqs. (\ref{equation_Mercer_theorem_T_operator}), (\ref{equation_Mercer_theorem_eigenfunction_decomposition}), and (\ref{equation_Mercer_theorem_kernel_representation}) are used in theory of kernels and kernel methods. 

\subsection{Feature Map and Pulling Function}\label{section_feature_map}

Let $\mathcal{X} := \{\b{x}_i\}_{i=1}^n$ be the set of data in the input space (note that the input space is the original space of data). The $t$-dimensional (perhaps infinite dimensional) feature space (or Hilbert space) is denoted by $\mathcal{H}$. 

\begin{definition}[Feature Map or Pulling Function]
We define the mapping:
\begin{align}
\b{\phi}: \mathcal{X} \rightarrow \mathcal{H}, 
\end{align}
to transform data from the input space to the feature space, i.e. Hilbert space. In other words, this mapping pulls data to the feature space:
\begin{align}\label{equation_pulling_mapping}
\b{x} \mapsto \b{\phi}(\b{x}).
\end{align}
The function $\b{\phi}(\b{x})$ is called the feature map or pulling function. 
The feature map is a (possibly infinite-dimensional) vector whose elements are \cite{minh2006mercer}:
\begin{equation}\label{equation_feature_map}
\begin{aligned}
\b{\phi}(\b{x}) &= [\phi_1(\b{x}), \phi_2(\b{x}), \dots]^\top \\
&:= [\sqrt{\lambda}_1\, \b{\psi}_1(\b{x}), \sqrt{\lambda}_2\, \b{\psi}_2(\b{x}), \dots]^\top,
\end{aligned}
\end{equation}
where $\{\b{\psi}_i\}$ and $\{\lambda_i\}$ are eigenfunctions and eigenvalues of the kernel operator (see Eq. (\ref{equation_Mercer_theorem_eigenfunction_decomposition})). Note that eigenfunctions will be explained more in Section \ref{section_eigenfunctions}.
\end{definition}
Let $t$ denote the dimensionality of $\b{\phi}(\b{x})$.
The feature map may be infinite or finite dimensional, i.e. $t$ can be infinity; it is usually a very large number (recall Definition \ref{definition_Hilbert_space} where we said Hilbert space may have infinite number of dimensions).

Considering both Eqs. (\ref{equation_Mercer_theorem_kernel_representation}) and (\ref{equation_feature_map}) shows that:
\begin{align}\label{equation_kernel_inner_product}
k(\b{x}, \b{y}) = \big\langle \b{\phi}(\b{x}), \b{\phi}(\b{y}) \big\rangle_k = \b{\phi}(\b{x})^\top \b{\phi}(\b{y}). 
\end{align}
Hence, the kernel between two points is the inner product of pulled data points to the feature space. 
Suppose we stack the feature maps of all points $\b{X} \in \mathbb{R}^{d \times n}$ column-wise in:
\begin{align}\label{equation_Phi_X_pulled_matrix}
\b{\Phi}(\b{X}) := [\b{\phi}(\b{x}_1), \b{\phi}(\b{x}_2), \dots, \b{\phi}(\b{x}_n)],
\end{align}
which is $t \times n$ dimensional and $t$ may be infinity or a large number. 
The kernel matrix defined in Definition \ref{definition_Gram_matrix} can be calculated as:
\begin{align}\label{equation_kernel_inner_product_matrix}
\mathbb{R}^{n \times n} \ni \b{K} = \big\langle \b{\Phi}(\b{X}), \b{\Phi}(\b{X}) \big\rangle_k =  \b{\Phi}(\b{X})^\top \b{\Phi}(\b{X}).
\end{align}
Eqs. (\ref{equation_kernel_inner_product}) and (\ref{equation_kernel_inner_product_matrix}) show that there is no need to compute kernel using eigenfunctions but a simple inner product suffices for kernel computation. This is the beauty of kernel methods which are simple to compute. 

\begin{definition}[Input Space and Feature Space \cite{scholkopf1999input}]
The space in which data $\b{X}$ exist is called the input space, also known as the original space. This space is denoted by $\mathcal{X}$ and is usually an $\mathbb{R}^d$ Euclidean space. The RKHS to which the data have been pulled is called the feature space. Data can be pulled from the input to feature space using kernels. 
\end{definition}

\begin{remark}[Kernel is a Measure of Similarity]\label{remark_kernel_is_similarity}
Inner product is a measure of similarity in terms of angles of vectors or in terms of location of points with respect to origin. According to Eq. (\ref{equation_kernel_inner_product}), kernel can be seen as inner product between feature maps of points; hence, kernel is a measure of similarity between points and this similarity is computed in the feature space rather than input space. 
\end{remark}

Pulling data to the feature space is performed using kernels which is the inner product of points in RKHS according to Eq. (\ref{equation_kernel_inner_product}). Hence, the relative similarity (inner product) of pulled data points is known by the kernel. However, in most of kernels, we cannot find an explicit expression for the pulled data points. 
Therefore, the exact location of pulled data points to RKHS is not necessarily known but the relative similarity of pulled points, which is the kernel, is known. 
An exceptional kernel is the linear kernel in which we have $\phi(\b{x}) = \b{x}$. 
Figure \ref{figure_pulling} illustrates what we mean by not knowing the explicit location of pulled points to RKHS.

\begin{figure*}[!t]
\centering
\includegraphics[width=6in]{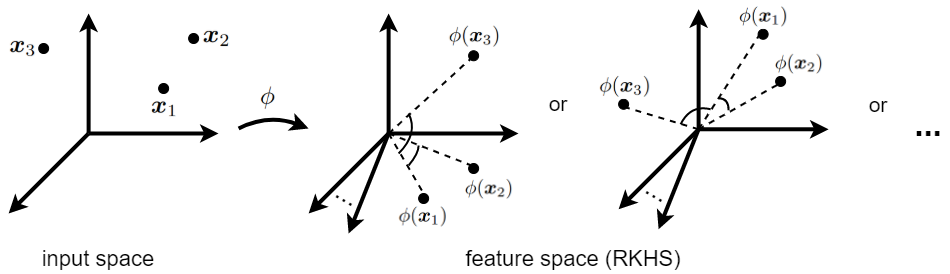}
\caption{Pulling data from the input space to the feature space (RKHS). The explicit locations of pulled points are not necessarily known but the relative similarity (inner product) of pulled data points is known in the feature space.}
\label{figure_pulling}
\end{figure*}

\section{Characteristics of Kernels}\label{section_kernel_characteristics}

In this section, we review some of the characteristics of kernels including the symmetry and positive semi-definiteness properties of Mercer kernel (recall Definition \ref{definition_Mercer_kernel}).

\begin{lemma}[Symmetry of Kernel]\label{lemma_kernel_is_symmetric}
A square Mercer kernel matrix is symmetric, so we have:
\begin{align}
&\langle f, g \rangle_k = \langle g, f \rangle_k, \quad \text{ or } \label{equation_symmetric_inner_product}\\
&\b{K} \in \mathbb{S}^n, \quad \text{i.e.,} \quad k(\b{x}, \b{y}) = k(\b{y}, \b{x}).
\end{align}
\end{lemma}
\begin{proof}
\begin{align*}
k(\b{x}, \b{y}) &\overset{(\ref{equation_kernel_inner_product})}{=} \big\langle \b{\phi}(\b{x}), \b{\phi}(\b{y}) \big\rangle_k = \b{\phi}(\b{x})^\top \b{\phi}(\b{y}) \\
&\overset{(a)}{=} \b{\phi}(\b{y})^\top \b{\phi}(\b{x}) = k(\b{y}, \b{x}), 
\end{align*}
where $(a)$ is because $\b{\phi}(\b{x})^\top \b{\phi}(\b{y})$ and $\b{\phi}(\b{y})^\top \b{\phi}(\b{x})$ are scalars and are equivalent according to the definition of dot product between vectors. Q.E.D.
\end{proof}

\begin{lemma}[Zero Kernel]
We have:
\begin{align}
\langle f, f \rangle_k = 0 \quad \text{iff} \quad f = 0.
\end{align}
\end{lemma}
\begin{proof}
\begin{align*}
0 &\leq f^2(\b{x}) \overset{(\ref{equation_RKHS_reproducing})}{=} \langle f, k_{\b{x}} \rangle_k\,\, \langle f, k_{\b{x}} \rangle_k \\
&\overset{(a)}{\leq} \|f\|_k \|k_{\b{x}}\|_k \,\, \|f\|_k \|k_{\b{x}}\|_k = \|f\|_k^2\, \|k_{\b{x}}\|_k^2 \overset{(b)}{=} 0,
\end{align*}
where $(a)$ is because Cauchy-Schwarz inequality and $(b)$ is because we had assumed $\langle f, f \rangle_k = \|f\|_k = 0$. Hence:
\begin{align*}
0 \leq f^2(\b{x}) = 0 \implies f(\b{x}) = 0.
\end{align*}
\end{proof}

\begin{lemma}[Positive Semi-definiteness of Kernel]\label{lemma_kernel_is_positive_semidefinite}
The Mercer kernel matrix is positive semi-definite:
\begin{align}
\b{K} \in \mathbb{S}_{+}^n, \quad \text{i.e.,} \quad \b{K} \succeq \b{0}.
\end{align}
\end{lemma}
\begin{proof}
Let $\b{v}(i)$ denote the $i$-th element of vector $\b{v}$.
\begin{align*}
\b{v}^\top \b{K} \b{v} &\overset{(\ref{equation_Gram_matrix})}{=} \sum_{i=1}^n \sum_{j=1}^n \b{v}(i)\, \b{v}(j)\, k(\b{x}(i), \b{x}(j)) \\
&\overset{(\ref{equation_RKHS_inner_product})}{=} \sum_{i=1}^n \sum_{j=1}^n \b{v}(i)\, \b{v}(j)\, \Big\langle \b{\phi}\big(\b{x}(i)\big), \b{\phi}\big(\b{x}(j)\big) \Big\rangle_k \\
&= \sum_{i=1}^n \sum_{j=1}^n \Big\langle \b{v}(i)\, \b{\phi}\big(\b{x}(i)\big), \b{v}(j)\, \b{\phi}\big(\b{x}(j)\big) \Big\rangle_k \\
&= \Big\langle \sum_{i=1}^n \b{v}(i)\, \b{\phi}\big(\b{x}(i)\big), \sum_{j=1}^n \b{v}(j)\, \b{\phi}\big(\b{x}(j)\big) \Big\rangle_k \\
&= \Big\| \sum_{i=1}^n \b{v}(i)\, \b{\phi}\big(\b{x}(i)\big) \Big\|_k^2 \geq 0, \quad \forall \b{v} \in \mathbb{R}^n. 
\end{align*}
Hence, according to the definition of positive semi-definiteness \cite{bhatia2009positive}, we have $\b{K} \succeq \b{0}$. Q.E.D.
\end{proof}

\section{Well-known Kernel Functions}\label{section_well_known_kernels}

\subsection{Frequently Used Kernels}


There exist many different kernel functions which are widely used in machine learning \cite{rojo2018digital}. In the following, we list some of the most well-known kernels. 

\textbf{-- Linear Kernel:}

Linear kernel is the simplest kernel which is the inner product of points:
\begin{align}\label{equation_linear_kernel}
k(\b{x}, \b{y}) := \b{x}^\top \b{y}.
\end{align}
Comparing this with Eq. (\ref{equation_kernel_inner_product}) shows that in linear kernel we have $\b{\phi}(\b{x}) = \b{x}$. Hence, in this kernel, the feature map is explicitly known. Note that $\b{\phi}(\b{x}) = \b{x}$ shows that data are not pulled to any other space in linear kernel but in the input space, the inner products of points are calculated to obtain the feature space. 
Moreover, recall Remark \ref{remark_linear_kernel_equivalent_to_non_kernelized} which states that, depending on the kernelization approach, using linear kernel may or may not be equivalent to non-kernelized method. 

\textbf{-- Radial Basis Function (RBF) or Gaussian Kernel:}

RBF kernel has a scaled Gaussian (or normal) distribution where the normalization factor of distribution is usually ignored. Hence, it is also called the Gaussian kernel. The RBF kernel is formulated as:
\begin{align}\label{equation_RBF_kernel}
k(\b{x}, \b{y}) := \exp(-\gamma\, \|\b{x} - \b{y}\|_2^2) = \exp(-\frac{\|\b{x} - \b{y}\|_2^2}{\sigma^2}),
\end{align}
where $\gamma := 1/\sigma^2$ and $\sigma^2$ is the variance of kernel. 
A proper value for this parameter is $\gamma=1/d$ where $d$ is the dimensionality of data.
Note that RBF kernel has also been widely used in RBF networks \cite{orr1996introduction} and kernel density estimation \cite{scott1992multivariate}.

\textbf{-- Laplacian Kernel:}

The Laplacian kernel, also called the Laplace kernel, is similar to the RBF kernel but with $\ell_1$ norm rather than squared $\ell_2$ norm. The Laplacian kernel is: 
\begin{align}
k(\b{x}, \b{y}) := \exp(-\gamma\, \|\b{x} - \b{y}\|_1) = \exp(-\frac{\|\b{x} - \b{y}\|_1}{\sigma^2}),
\end{align}
where $\|\b{x} - \b{y}\|_1$ is also called the Manhattan distance. 
A proper value for this parameter is $\gamma=1/d$ where $d$ is the dimensionality of data.
In some specific fields of science, the Laplacian kernel has been found to perform better than Gaussian kernel \cite{rupp2015machine}. This makes sense because of betting on sparsity principal \cite{hastie2009elements} since $\ell_1$ norm makes algorithm sparse. 
Note that $\ell_2$ norm in RBF kernel is also more sensitive to noise; however, the computation and derivative of $\ell_1$ norm is more difficult than $\ell_2$ norm. 

\textbf{-- Sigmoid Kernel:}

Sigmoid kernel is a hyperbolic tangent function applied on inner product of points. It is formulated as:
\begin{align}
k(\b{x}, \b{y}) := \tanh(\gamma \b{x}^\top \b{y} + c),
\end{align}
where $\gamma >0$ is the slope and $c$ is the intercept. Some proper values for these parameters are $\gamma=1/d$ and $c=1$ where $d$ is the dimensionality of data.
Note that the hyperbolic tangent function is also used widely for activation functions in neural networks \cite{goodfellow2016deep}. 

\textbf{-- Polynomial Kernel:}

Polynomial kernel applies a polynomial function with degree $\delta$ (a positive integer) on inner product of points:
\begin{align}
k(\b{x}, \b{y}) := (\gamma \b{x}^\top \b{y} + c)^d,
\end{align}
where $\gamma >0$ is the slope and $c$ is the intercept. Some proper values for these parameters are $\gamma=1/d$ and $c=1$ where $d$ is the dimensionality of data.

\textbf{-- Cosine Kernel:}

According to Remark \ref{remark_kernel_is_similarity}, kernel is a measure of similarity and computes the inner product between points in the feature space. 
Cosine kernel computes the similarity between points. It is obtained from the formula of cosine and inner product:
\begin{align}\label{equation_cosine_kernel}
k(\b{x}, \b{y}) := \cos(\b{x}, \b{y}) = \frac{\b{x}^\top \b{y}}{\|\b{x}\|_2\, \|\b{y}\|_2}.
\end{align}
The normalization in the denominator projects the points onto a unit hyper-sphere so that the inner product measures the similarity of their angles regardless of their lengths. 
Note that angle-based measures such as cosine are found to work better for face recognition compared to Euclidean distances \cite{perlibakas2004distance}.

\textbf{-- Chi-squared Kernel:}

Assume $\b{x}(j)$ denotes the $j$-th dimension of the $d$-dimensional point $\b{x}$.
The Chi-squared ($\chi^2$) kernel is \cite{zhang2007local}:
\begin{align}
k(\b{x}, \b{y}) := \exp\Big(\!\!-\!\gamma \sum_{j=1}^d \frac{\big(\b{x}(j) - \b{y}(j)\big)^2}{\b{x}(j) + \b{y}(j)}\Big),
\end{align}
where $\gamma >0$ is a parameter (a proper value is $\gamma=1$). 
Note that the summation term inside exponential (without the minus) is the Chi-squared distance which is related to the Chi-squared test in statistics. 

\subsection{Kernel Construction from Distance Metric}

Consider $d_{ij}^2 = ||\b{x}_i - \b{x}_j||_2^2$ as the squared Euclidean distance between $\b{x}_i$ and $\b{x}_j$. We have:
\begin{align*}
d_{ij}^2 &= ||\b{x}_i - \b{x}_j||_2^2 = (\b{x}_i - \b{x}_j)^\top (\b{x}_i - \b{x}_j) \\
&= \b{x}_i^\top \b{x}_i - \b{x}_i^\top \b{x}_j - \b{x}_j^\top \b{x}_i + \b{x}_j^\top \b{x}_j \\
&= \b{x}_i^\top \b{x}_i - 2\b{x}_i^\top \b{x}_j + \b{x}_j^\top \b{x}_j = \b{G}_{ii} - 2 \b{G}_{ij} + \b{G}_{jj},
\end{align*}
where $\mathbb{R}^{n \times n} \ni \b{G} := \b{X}^\top \b{X}$ is the linear Gram matrix. If $\mathbb{R}^n \ni \b{g} := [\b{g}_1, \dots, \b{g}_n] = [\b{G}_{11}, \dots, \b{G}_{nn}] = \textbf{diag}(\b{G})$, we have:
\begin{align*}
& d_{ij}^2 = \b{g}_i -2\b{G}_{ij} + \b{g}_j, \\
& \b{D} = \b{g}\b{1}^\top -2 \b{G} +\b{1}\b{g}^\top = \b{1}\b{g}^\top -2 \b{G} + \b{g}\b{1}^\top,
\end{align*}
where $\b{1}$ is the vector of ones and $\b{D}$ is the distance matrix with squared Euclidean distance ($d_{ij}^2$ as its elements). 
Let $\b{H}$ denote the centering matrix:
\begin{align}\label{equation_centered_matrix}
\mathbb{R}^{n \times n} \ni \b{H} := \b{I} - \frac{1}{n} \b{1}_n\b{1}_n^\top,
\end{align}
and $\b{I}$ is the identity matrix, $\b{1}_n := [1, \dots, 1]^\top \in \mathbb{R}^n$ and $\b{1}_{n \times n} := \b{1}_n \b{1}_n^\top \in \mathbb{R}^{n \times n}$. 
Refer to {\citep[Appendix A]{ghojogh2019unsupervised}} for more details about the centering matrix. 
We double-center the matrix $\b{D}$ as follows \cite{oldford2018lecture}:
\begin{align*}
\b{HDH} &= (\b{I} - \frac{1}{n}\b{1}\b{1}^\top) \b{D} (\b{I} - \frac{1}{n}\b{1}\b{1}^\top) \\
&= (\b{I} - \frac{1}{n}\b{1}\b{1}^\top) (\b{1}\b{g}^\top -2 \b{G} + \b{g}\b{1}^\top) (\b{I} - \frac{1}{n}\b{1}\b{1}^\top) \\
&= \big[\underbrace{(\b{I} - \frac{1}{n}\b{1}\b{1}^\top)\b{1}}_{=\,\b{0}} \b{g}^\top -2 (\b{I} - \frac{1}{n}\b{1}\b{1}^\top)\b{G}\\ 
&~~~~~ + (\b{I} - \frac{1}{n}\b{1}\b{1}^\top)\b{g}\b{1}^\top\big] (\b{I} - \frac{1}{n}\b{1}\b{1}^\top) \\
&= -2 (\b{I} - \frac{1}{n}\b{1}\b{1}^\top)\b{G}(\b{I} - \frac{1}{n}\b{1}\b{1}^\top) \\
&~~~~~ + (\b{I} - \frac{1}{n}\b{1}\b{1}^\top)\b{g}\underbrace{\b{1}^\top(\b{I} - \frac{1}{n}\b{1}\b{1}^\top)}_{=\,\b{0}} \\
&= -2 (\b{I} - \frac{1}{n}\b{1}\b{1}^\top)\b{G}(\b{I} - \frac{1}{n}\b{1}\b{1}^\top) = -2\,\b{HGH}.
\end{align*}
\begin{align}\label{equation_linearKernel_and_distanceMAtrix_1}
\therefore~~~~~~~~ \b{HGH} = \b{H}\b{X}^\top\b{X}\b{H} = -\frac{1}{2} \b{HDH}.
\end{align}
Note that $(\b{I} - \frac{1}{n}\b{1}\b{1}^\top)\b{1} = \b{0}$ and $\b{1}^\top(\b{I} - \frac{1}{n}\b{1}\b{1}^\top) = \b{0}$ because removing the row mean of $\b{1}$ and column mean of of $\b{1}^\top$ results in the zero vectors, respectively.

If data $\b{X}$ are already centered, i.e., the mean has been removed ($\b{X} \gets \b{X}\b{H}$), Eq. (\ref{equation_linearKernel_and_distanceMAtrix_1}) becomes:
\begin{align}\label{equation_linearKernel_and_distanceMAtrix_2}
\b{X}^\top\b{X} = -\frac{1}{2} \b{HDH}.
\end{align}

According to the kernel trick, Eq. (\ref{equation_kernel_trick_matrix}), we can write a general kernel matrix rather than the linear Gram matrix in Eq. (\ref{equation_linearKernel_and_distanceMAtrix_2}), to have \cite{cox2008multidimensional}:
\begin{align}\label{equation_generalKernel_and_distanceMAtrix}
\mathbb{R}^{n \times n} \ni \b{K} = \b{\Phi}(\b{X})^\top \b{\Phi}(\b{X}) = -\frac{1}{2} \b{HDH}.
\end{align}
This kernel is double-centered because of $\b{HDH}$.  
It is also noteworthy that Eq. (\ref{equation_generalKernel_and_distanceMAtrix}) can be used for unifying the spectral dimensionality reduction methods as special cases of kernel principal component analysis with different kernels. See \cite{ham2004kernel,bengio2004learning} and {\citep[Table 2.1]{strange2014open}} for more details. 

\begin{lemma}[Distance-based Kernel is a Mercer Kernel]
The kernel constructed from a valid distance metric, i.e. Eq. (\ref{equation_generalKernel_and_distanceMAtrix}), is a Mercer kernel. 
\end{lemma}
\begin{proof}
The kernel is symmetric because:
\begin{align*}
\b{K}^\top = -\frac{1}{2} \b{H}^\top \b{D}^\top \b{H}^\top \overset{(a)}{=} -\frac{1}{2} \b{HDH} = \b{K},
\end{align*}
where $(a)$ is because $\b{H}$ and $\b{D}$ are symmetric matrices. 
Moreover, the kernel is positive semi-definite because:
\begin{align*}
&\b{K} = -\frac{1}{2} \b{HDH} = \b{\Phi}(\b{X})^\top \b{\Phi}(\b{X}) \\
&\implies \b{v}^\top \b{K} \b{v} = \b{v}^\top \b{\Phi}(\b{X})^\top \b{\Phi}(\b{X}) \b{v} \\
&~~~~~~~~~~~~~~~~~~~~~~~ = \|\b{\Phi}(\b{X}) \b{v}\|_2^2 \geq 0, \quad \forall \b{v} \in \mathbb{R}^n. 
\end{align*}
Hence, according to Definition \ref{definition_Mercer_kernel}, this kernel is a Mercer kernel. Q.E.D.
\end{proof}

\begin{remark}[Kernel Construction from Metric]
One can use any valid distance metric, satisfying the following properties:
\begin{enumerate}[topsep=0pt,itemsep=-1ex,partopsep=1ex,parsep=1ex]
\item non-negativity: $\b{D}(\b{x}, \b{y}) \geq 0$, 
\item equal points: $\b{D}(\b{x}, \b{y}) = 0 \iff \b{x} = \b{y}$, 
\item symmetry: $\b{D}(\b{x}, \b{y}) = \b{D}(\b{y}, \b{x})$, 
\item triangular inequality: $\b{D}(\b{x}, \b{y}) \leq \b{D}(\b{x}, \b{z})\! + \b{D}(\b{z}, \b{y})$,
\end{enumerate}
to calculate elements of distance matrix $\b{D}$ in Eq. (\ref{equation_linearKernel_and_distanceMAtrix_2}). It is important that the used distance matrix should be a valid distance matrix. Using various distance metrics in Eq. (\ref{equation_linearKernel_and_distanceMAtrix_2}) results in various useful kernels.
\end{remark}

Some examples are the geodesic kernel and Structural Similarity Index (SSIM) kernel, used in Isomap \cite{tenenbaum2000global} and image structure subspace learning \cite{ghojogh2019image}, respectively. 
The geodesic kernel is defined as \cite{tenenbaum2000global,ghojogh2020multidimensional}:
\begin{align}
\b{K} = -\frac{1}{2} \b{H}\b{D}^{(g)}\b{H},
\end{align}
where the approximation of geodesic distances using piece-wise Euclidean distances is used in calculating the geodesic distance matrix $\b{D}^{(g)}$. 
The SSIM kernel is defined as \cite{ghojogh2019image}:
\begin{align}
\b{K} = -\frac{1}{2} \b{H} \b{D}^{(s)} \b{H},
\end{align}
where the distance matrix $\b{D}^{(s)}$ is calculated using the SSIM distance \cite{brunet2011mathematical}.

\subsection{Important Classes of Kernels}\label{section_universal_characteristic_kernels}

In the following, we introduce some of the important classes of kernels which are widely used in statistics and machine learning. A good survey on the classes of kernels is \cite{genton2001classes}.


\subsubsection{Bounded Kernels}

\begin{definition}[Bounded Kernel]
A kernel function $k$ is bounded if:
\begin{align}
\sup_{\b{x}, \b{y} \in \mathcal{X}} k(\b{x}, \b{y}) < \infty, 
\end{align}
where $\mathcal{X}$ is the input space. Likewise, the kernel matrix $\b{K}$ is bounded if $\sup_{\b{x}, \b{y} \in \mathcal{X}} \b{K}(\b{x}, \b{y}) < \infty$.
\end{definition}

\subsubsection{Integrally Positive Definite Kernels}\label{section_integrally_positive_definite_kernels}

\begin{definition}[Integrally Positive Definite Kernel]
A kernel matrix $\b{K}$ is integrally positive definite ($\int$ p.d.) on $\Omega \times \Omega$ if:
\begin{align}
\int_\Omega \int_\Omega \b{K}(\b{x}, \b{y}) f(\b{x}) f(\b{y}) \geq 0, \quad \forall f \in L_2(\Omega).
\end{align}
A kernel matrix $\b{K}$ is integrally strictly positive definite ($\int$ s.p.d.) on $\Omega \times \Omega$ if:
\begin{align}
\int_\Omega \int_\Omega \b{K}(\b{x}, \b{y}) f(\b{x}) f(\b{y}) > 0, \quad \forall f \in L_2(\Omega).
\end{align}
\end{definition}

\subsubsection{Universal Kernels}\label{section_universal_kernels}


\begin{definition}[Universal Kernel {\citep[Definition 4]{steinwart2001influence}}, {\citep[Definition 2]{steinwart2002support}}]
Let $C(\mathcal{X})$ denote the space of all continuous functions on space $\mathcal{X}$. A continuous kernel $k$ on a compact metric space $\mathcal{X}$ is called universal if the RKHS $\mathcal{H}$, with kernel function $k$, is dense in $C(\mathcal{X})$. In other words, for every function $g \in C(\mathcal{X})$ and all $\epsilon > 0$, there exists a function $f \in \mathcal{H}$ such that $\|f - g\|_\infty \leq \epsilon$.
\end{definition}

\begin{remark}
We can approximate any function, including continuous functions and functions which can be approximated by continuous functions, using a universal kernel. 
\end{remark}

\begin{lemma}[{\citep[Corollary 10]{steinwart2001influence}}]
Consider a function $f : (-r, r) \rightarrow \mathbb{R}$ where $0 < r \leq \infty$ and $f \in C^\infty$ ($C^\infty$ denotes the differentiable space for all degrees of differentiation). Let $\mathcal{X} := \{\b{x} \in \mathbb{R}^d\, |\, \|\b{x}\|_2 < \sqrt{r}\}$. If the function $f$ can be expanded by Taylor expansion in $0$ as:
\begin{align}
f(\b{x}) = \sum_{j=0}^\infty a_j\, \b{x}^j, \quad \forall \b{x} \in (-r, r),
\end{align}
and $a_j > 0$ for all $j \geq 0$, then $k(\b{x}, \b{y}) = f(\langle \b{x}, \b{y} \rangle)$ is a universal kernel on every compact subset of $\mathcal{X}$. 
\end{lemma}
\begin{proof}
For proof, see {\citep[proof of Corollary 10]{steinwart2001influence}}. Note that the Stone-Weierstrass theorem \cite{de1959stone} is used for the proof of this lemma.
\end{proof}
An example for universal kernel is RBF kernel {\citep[Example 1]{steinwart2001influence}} because its Taylor series expansion is:
\begin{align*}
\exp(-\gamma r) \approx 1 - \gamma r + \frac{\gamma^2}{2} r^2 - \frac{\gamma^3}{6} r^3 + \dots,
\end{align*}
where $r := \|\b{x} - \b{y}\|_2^2$.
Considering Eq. (\ref{equation_kernel_inner_product}) and noticing that this Taylor series expansion has infinite number of terms, we see that the RKHS for RBF kernel is infinite dimensional because $\phi(\b{x})$, although cannot be calculated explicitely for this kernel, will have infinite dimensions. 
Another example for universal kernel is the SSIM kernel \cite{ghojogh2020theoretical}, denoted by $\b{K}_s$, whose Taylor series expansion is \cite{ghojogh2019image}:
\begin{align*}
\b{K}_s \approx -\frac{5}{16} - \frac{15}{16} r + \frac{5}{16} r^2 - \frac{1}{16} r^3 + \dots,
\end{align*}
where $r$ is the squared SSIM distance \cite{brunet2011mathematical} between images. 
Note that polynomial kernels are not universal. 
Universal kernels have been widely used for kernel SVM. 
More detailed discussion and proofs for use of universal kernels in kernel SVM can be found in \cite{steinwart2008support}.



\begin{lemma}[\cite{borgwardt2006integrating}, {\citep[Theorem 10]{song2008learning}}]
A kernel is universal if for arbitrary sets of distinct points, it induces strictly positive definite kernel matrices \cite{borgwardt2006integrating,song2008learning}. Conversely, if a kernel matrix can be written as $\b{K} = \b{K}' + \epsilon \b{I}$ where $\b{K}' \succeq \b{0}$, $\epsilon > 0$, and $\b{I}$ is the identity matrix, the kernel function corresponding to $\b{K}$ is universal \cite{pan2008transfer}.
\end{lemma}

\subsubsection{Stationary Kernels}



\begin{definition}[Stationary Kernel \cite{genton2001classes,noack2021advanced}]
A kernel $k$ is stationary if it is a positive definite function of the form:
\begin{align}\label{equation_stationary_kernel}
k(\b{x}, \b{y}) = k(\|\b{x} - \b{y}\|),
\end{align}
where $\|.\|$ is some norm defined on the input space. 
\end{definition}
An example for stationary kernel is the RBF kernel defined in Eq. (\ref{equation_RBF_kernel}) which has the form of Eq. (\ref{equation_stationary_kernel}). 
Stationary kernels are used for Gaussian processes \cite{noack2021advanced}. 

\subsubsection{Characteristic Kernels}

The characteristic kernels, which are widely used for distribution embedding in the Hilbert space, will be defined and explained in Section \ref{section_kernel_embedding_distributions}.
Examples for characteristic kernels are RBF and Laplacian kernels. Polynomial kernels. however, are not characteristic.
Note that the relation between universal kernels, characteristic kernels, and integrally strictly positive definite kernels has been studied in \cite{sriperumbudur2011universality}.

\section{Kernel Centering and Normalization}\label{section_kernel_centering_normalization}

\subsection{Kernel Centering}\label{section_kernel_centering}

In some cases, there is a need to center the pulled data in the feature space. For this, the kernel matrix should be centered in a way that the mean of pulled dataset becomes zero. Note that this will restrict the place of pulled points in the feature space further (see Fig. \ref{figure_centered_pulled_data}); however, because of different possible rotations of pulled points around origin, the exact positions of pulled points are still unknown. 

\begin{figure}[!t]
\centering
\includegraphics[width=3.4in]{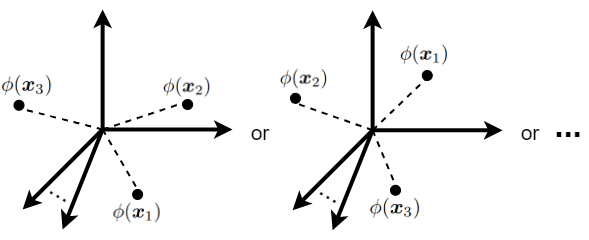}
\caption{Centered pulled data the feature space (RKHS). This happens after kernel centering where the mean of cloud of pulled data becomes zero in RKHS. Even by kernel centering, the explicit locations of pulled points are not necessarily known because of not knowing the rotation of pulled data in that space.}
\label{figure_centered_pulled_data}
\end{figure}

For kernel centering, one should follow the following theory, which is based on \cite{scholkopf1997kernel} and {\citep[Appendix A]{scholkopf1998nonlinear}}. An example of use of kernel centering in machine learning is kernel principal component analysis (see \cite{ghojogh2019unsupervised} for more details).

\subsubsection{Centering the Kernel of Training Data}

Assume we have some training data $\b{X} = [\b{x}_1, \dots, \b{x}_n] \in \mathbb{R}^{d \times n}$ and some out-of-sample data $\b{X}_t = [\b{x}_{t,1}, \dots, \b{x}_{t,n_t}] \in \mathbb{R}^{d \times n_t}$. 
Consider the kernel matrix for the training data $\mathbb{R}^{n \times n} \ni \b{K} := \b{\Phi}(\b{X})^\top \b{\Phi}(\b{X})$, whose $(i,j)$-th element is $\mathbb{R} \ni \b{K}(i,j) = \b{\phi}(\b{x}_i)^\top \b{\phi}(\b{x}_j)$. 
We want to center the pulled training data in the feature space:
\begin{align}\label{equation_appendix_centered_pulled_training}
\breve{\b{\phi}}(\b{x}_i) := \b{\phi}(\b{x}_i) - \frac{1}{n} \sum_{k=1}^n \b{\phi}(\b{x}_k).
\end{align}
If we center the pulled training data, the $(i,j)$-th element of kernel matrix becomes:
\begin{align}
&\breve{\b{K}}(i,j) := \breve{\b{\phi}}(\b{x}_i)^\top \breve{\b{\phi}}(\b{x}_j) \label{equation_centered_kernel_definition} \\
&\overset{(\ref{equation_appendix_centered_pulled_training})}{=}\! \big(\b{\phi}(\b{x}_i) - \frac{1}{n} \sum_{k_1=1}^n \b{\phi}(\b{x}_{k_1})\big)^\top \! \big(\b{\phi}(\b{x}_j) - \frac{1}{n} \sum_{k_2=1}^n \b{\phi}(\b{x}_{k_2})\big) \nonumber \\
&= \b{\phi}(\b{x}_i)^\top \b{\phi}(\b{x}_j) - \frac{1}{n} \sum_{k_1=1}^n \b{\phi}(\b{x}_{k_1})^\top \b{\phi}(\b{x}_j) \nonumber \\
&- \frac{1}{n} \sum_{k_2=1}^n \b{\phi}(\b{x}_i)^\top \b{\phi}(\b{x}_{k_2}) + \frac{1}{n^2}\! \sum_{k_1=1}^n  \sum_{k_2=1}^n \b{\phi}(\b{x}_{k_1})^\top \! \b{\phi}(\b{x}_{k_2}). \nonumber
\end{align}
Writing this in the matrix form gives:
\begin{align}
\mathbb{R}^{n \times n} \ni \breve{\b{K}} &= \b{K} - \frac{1}{n} \b{1}_{n \times n} \b{K} - \frac{1}{n} \b{K} \b{1}_{n \times n} \nonumber \\
&~~~~ + \frac{1}{n^2} \b{1}_{n \times n} \b{K} \b{1}_{n \times n} = \b{H} \b{K} \b{H}, \label{equation_appendix_doubleCentered_training_kernel}
\end{align}
where $\b{H}$ is the centering matrix (see Eq. (\ref{equation_centered_matrix})).
The Eq. (\ref{equation_appendix_doubleCentered_training_kernel}) is called the \textit{double-centered kernel}. 
This equation is the kernel matrix when the pulled training data in the feature space are centered. Also, double-centered kernel has zero row-wise and column-wise mean (so its row and column summations are zero). 
Therefore, after this kernel centering, we will have:
\begin{align}
&\frac{1}{n} \sum_{i=1}^n \breve{\b{\phi}}(\b{x}_i) = \b{0}, \\
&\sum_{i=1}^n \sum_{j=1}^n \breve{\b{K}}(i,j) = 0. \label{equation_appendix_doubleCentered_training_kernel_sum}
\end{align}

\subsubsection{Centering the Kernel between Training and Out-of-sample Data}

Now, consider the kernel matrix between the training data and the out-of-sample data $\mathbb{R}^{n \times n_t} \ni \b{K}_t := \b{\Phi}(\b{X})^\top \b{\Phi}(\b{X}_t)$.
whose $(i,j)$-th element is $\mathbb{R} \ni \b{K}_t(i,j) = \b{\phi}(\b{x}_{i})^\top \b{\phi}(\b{x}_{t,j})$.
We want to center the pulled training data in the feature space, i.e., Eq. (\ref{equation_appendix_centered_pulled_training}). Moreover, the out-of-sample data should be centered using the mean of training (and not out-of-sample) data:
\begin{align}\label{equation_appendix_centered_pulled_outOfSample}
\breve{\b{\phi}}(\b{x}_{t,i}) := \b{\phi}(\b{x}_{t,i}) - \frac{1}{n} \sum_{k=1}^n \b{\phi}(\b{x}_k).
\end{align}
If we center the pulled training and out-of-sample data, the $(i,j)$-th element of kernel matrix becomes:
\begin{align*}
&\breve{\b{K}}_t(i,j) := \breve{\b{\phi}}(\b{x}_i)^\top \breve{\b{\phi}}(\b{x}_{t,j}) \\
&\overset{(a)}{=} \! \big(\b{\phi}(\b{x}_i) - \frac{1}{n} \sum_{k_1=1}^n \b{\phi}(\b{x}_{k_1})\big)^\top \! \big(\b{\phi}(\b{x}_{t,j}) - \frac{1}{n} \sum_{k_2=1}^n \b{\phi}(\b{x}_{k_2})\big) \\
&= \b{\phi}(\b{x}_i)^\top \b{\phi}(\b{x}_{t,j}) - \frac{1}{n} \sum_{k_1=1}^n \b{\phi}(\b{x}_{k_1})^\top \b{\phi}(\b{x}_{t,j}) \\
& - \frac{1}{n} \sum_{k_2=1}^n \b{\phi}(\b{x}_i)^\top \b{\phi}(\b{x}_{k_2}) + \frac{1}{n^2} \! \sum_{k_1=1}^n  \sum_{k_2=1}^n \b{\phi}(\b{x}_{k_1})^\top \! \b{\phi}(\b{x}_{k_2}),
\end{align*}
where (a) is because of Eqs. (\ref{equation_appendix_centered_pulled_training}) and (\ref{equation_appendix_centered_pulled_outOfSample}).
Therefore, the double-centered kernel matrix over training and out-of-sample data is:
\begin{align}
\mathbb{R}^{n \times n_t} \ni \breve{\b{K}}_t &= \b{K}_t - \frac{1}{n} \b{1}_{n \times n} \b{K}_t - \frac{1}{n} \b{K} \b{1}_{n \times n_t} \nonumber \\
&~~~~ + \frac{1}{n^2} \b{1}_{n \times n} \b{K} \b{1}_{n \times n_t}, \label{equation_appendix_doubleCentered_outOfSample_kernel}
\end{align}
where $\mathbb{R}^{n \times n_t} \ni \b{1}_{n \times n_t} := \b{1}_n \b{1}_{n_t}^\top$ and $\mathbb{R}^{n_t} \ni \b{1}_{n_t} := [1, \dots, 1]^\top$.
The Eq. (\ref{equation_appendix_doubleCentered_outOfSample_kernel}) is the kernel matrix when the pulled training data in the feature space are centered and the pulled out-of-sample data are centered using the mean of pulled training data.

If we have one out-of-sample $\b{x}_t$, the Eq. (\ref{equation_appendix_doubleCentered_outOfSample_kernel}) becomes:
\begin{align}
\mathbb{R}^{n} \ni \breve{\b{k}}_t &= \b{k}_t - \frac{1}{n} \b{1}_{n \times n} \b{k}_t - \frac{1}{n} \b{K} \b{1}_{n} + \frac{1}{n^2} \b{1}_{n \times n} \b{K} \b{1}_{n}, \label{equation_appendix_doubleCentered_outOfSample_kernel_oneSample}
\end{align}
where:
\begin{align}
&\mathbb{R}^n \ni \b{k}_t = \b{k}_t(\b{X}, \b{x}_t) := \b{\Phi}(\b{X})^\top \b{\phi}(\b{x}_t) \label{equation_appendix_kernelVector_outOfSample} \\
&~~~~~~~~~~~~~~ =[\b{\phi}(\b{x}_1)^\top \b{\phi}(\b{x}_t), \dots, \b{\phi}(\b{x}_n)^\top \b{\phi}(\b{x}_t)]^\top, \nonumber \\
&\mathbb{R}^n \ni \breve{\b{k}}_t = \breve{\b{k}}_t(\b{X}, \b{x}_t) := \breve{\b{\Phi}}(\b{X})^\top \breve{\b{\phi}}(\b{x}_t), \label{equation_appendix_centered_kernelVector_outOfSample} \\
&~~~~~~~~~~~~~~ =[\breve{\b{\phi}}(\b{x}_1)^\top \breve{\b{\phi}}(\b{x}_t), \dots, \breve{\b{\phi}}(\b{x}_n)^\top \breve{\b{\phi}}(\b{x}_t)]^\top, \nonumber 
\end{align}
where $\breve{\b{\Phi}}(\b{X})$ and $\breve{\b{\phi}}(\b{x}_t)$ are according to Eqs. (\ref{equation_appendix_centered_pulled_training}) and (\ref{equation_appendix_centered_pulled_outOfSample}), respectively.

Note that Eq. (\ref{equation_appendix_doubleCentered_training_kernel}) or  (\ref{equation_appendix_doubleCentered_outOfSample_kernel}) can be restated as the following lemma.
\begin{lemma}[Kernel Centering \cite{bengio2003learning,bengio2003spectral}]\label{lemma_kernel_centering}
The pulled data to the feature space can be centered by kernel centering. The kernel matrix $\b{K}(\b{x},\b{y})$ is centered as:
\begin{align}
\breve{\b{K}}(\b{x},\b{y}) &= \big(\b{\phi}(\b{x}) - \mathbb{E}_x[\b{\phi}(\b{x})]\big)^\top \big(\b{\phi}(\b{x}) - \mathbb{E}_x[\b{\phi}(\b{x})]\big) \nonumber \\
&= \b{K}(\b{x}, \b{y}) - \mathbb{E}_x[\b{K}(\b{x}, \b{y})] - \mathbb{E}_y[\b{K}(\b{x}, \b{y})] \nonumber \\
&~~~ + \mathbb{E}_x[\mathbb{E}_y[\b{K}(\b{x}, \b{y})]]. \label{equation_kernel_centering_with_exptectation}
\end{align}
\end{lemma}
\begin{proof}
The explained derivations for Eqs. (\ref{equation_appendix_centered_pulled_training}) and (\ref{equation_appendix_centered_pulled_outOfSample}) and definition of expectation complete the proof.
\end{proof}
Note that in Eq. (\ref{equation_kernel_centering_with_exptectation}), $\mathbb{E}_x[\b{K}(\b{x}, \b{y})]$, $\mathbb{E}_y[\b{K}(\b{x}, \b{y})]$, and $\mathbb{E}_x[\mathbb{E}_y[\b{K}(\b{x}, \b{y})]]$ are average of rows, average of columns, and total average of rows and columns of the kernel matrix, respectively.

\subsection{Kernel Normalization}

According to Eq. (\ref{equation_kernel_inner_product}), kernel value can be large if the pulled vectors to the feature map have large length. Hence, in practical computations and optimization, it is sometimes required to normalize the kernel matrix. 

\begin{lemma}[Cosine Normalization of Kernel \cite{rennie2005how,ah2010normalized}]
The kernel matrix $\b{K} \in \mathbb{R}^{n \times n}$ can be normalized as:
\begin{align}
\b{K}(i,j) \gets \frac{\b{K}(i,j)}{\sqrt{\b{K}(i,i) \b{K}(j,j)}}, \quad \forall i,j \in \{1, \dots, n\}.
\end{align}
\end{lemma}
\begin{proof}
Cosine normalizes points onto a unit hyper-sphere and then computes the similarity of points using inner product. Cosine is computed by Eq. (\ref{equation_cosine_kernel}) and according to the relation of norm and inner product, it is:
\begin{align*}
\cos(\b{x}_i, \b{x}_j) &= \frac{\b{x}_i^\top \b{x}_j}{\|\b{x}_i\|_2\, \|\b{x}_j\|_2} = \frac{\b{x}_i^\top \b{x}_j}{\sqrt{\|\b{x}_i\|_2^2\, \|\b{x}_j\|_2^2}} \\
&= \frac{\b{x}_i^\top \b{x}_j}{\sqrt{\b{x}_i^\top \b{x}_i\, \b{x}_j^\top \b{x}_j}}.
\end{align*}
According to Remark \ref{remark_kernel_is_similarity}, kernel is also a measure of similarity. 
Using kernel trick, Eq. (\ref{equation_kernel_trick}), the cosine similarity (which is already normalized) is kernelized as:
\begin{align*}
\b{K}(i,j) &= \frac{\b{\phi}(\b{x}_i)^\top \b{\phi}(\b{x}_j)}{\sqrt{ \b{\phi}(\b{x}_i)^\top \b{\phi}(\b{x}_i)\, \b{\phi}(\b{x}_j)^\top \b{\phi}(\b{x}_j) }} \\
&\overset{(\ref{equation_kernel_trick})}{=} \frac{\b{K}(i,j)}{\sqrt{\b{K}(i,i) \b{K}(j,j)}}. 
\end{align*}
Q.E.D.
\end{proof}

\begin{definition}[generalized mean with exponent $t$ \cite{ah2010normalized}]
The generalized mean with exponent $t$ as:
\begin{align}\label{equation_generalized_mean}
m_t(a_1, \dots, a_n) := \Big(\frac{1}{p} \sum_{i=1}^p a_i^t\Big)^{\frac{1}{t}}.
\end{align}
The generalized mean becomes the harmonic, geometric, and arithmetic mean for $t=-1$, $t \rightarrow 0$, and $t=1$, respectively.
\end{definition}

\begin{definition}[Generalized Kernel Normalization of order $t$ \cite{ah2010normalized}]
The generalized kernel normalization of order $t>0$ normalizes the kernel $\b{K} \in \mathbb{R}^{n \times n}$ as:
\begin{align}
\b{K}(i,j) \gets \frac{\b{K}(i,j)}{m_t\big(\b{K}(i,i), \b{K}(j,j)\big)}, \quad \forall i,j \in \{1, \dots, n\}.
\end{align}
\end{definition}

Both cosine normalization and generalized normalization make the kernel of every point with itself one. In other words, after normalization, we have:
\begin{align}
k(\b{x}, \b{x}) = 1, \quad \forall \b{x} \in \mathcal{X}.
\end{align}
As the most similar point to a point is itself, the values of a normalized kernel will be less than or equal to one. In other words, after normalization, we have:
\begin{align}
k(\b{x}_i, \b{x}_j) \leq 1, \quad \forall i,j \in \{1, \dots, n\}.
\end{align}
This helps the values not explode to large values in algorithms, especially in the iterative algorithms (e.g., algorithms which use gradient descent for optimization).

\section{Eigenfunctions}\label{section_eigenfunctions}

\subsection{Inner Product in Hilbert Space}

\begin{lemma}[Inner Product in Hilbert Space]
If the domain of functions in a Hilbert space $\mathcal{H}$ is $[a,b]$, the inner product of two functions in the Hilbert space is calculated as:
\begin{align}\label{equation_Hilbert_inner_product}
\langle f(x), g(x)  \rangle_\mathcal{H} = \int_a^b f(x)\, g^*(x) dx \!\overset{(\ref{equation_symmetric_inner_product})}{=}\!\! \int_a^b f^*(x)\, g(x) dx,
\end{align}
where $g^*$ is the complex conjugate of function $g$. If functions are real, the inner product is simplified to $\int_a^b f(x)\, g(x) dx$.
\end{lemma}
\begin{proof}
If we discretize the domain $[a,b]$, for example by sampling, with step $\Delta x$, the function values become vectors as $\b{f}(x) = [f(x_1), f(x_2), \dots, f(x_n)]^\top$ and $\b{g}(x) = [g(x_1), g(x_2), \dots, g(x_n)]^\top$. According to the inner product of two vectors, we have:
\begin{align*}
\langle \b{f}(x), \b{g}(x)  \rangle = \b{g}^H \b{f} = \sum_{i=1}^n f(x_i) g(x_i),
\end{align*}
where $g^H$ denotes the conjugate transpose of $g$ (it is transpose if functions are real). Multiplying the sides of this equation by the setp $\Delta x$ gives:
\begin{align*}
\langle \b{f}(x), \b{g}(x)  \rangle \Delta x = \b{g}^H \b{f} = \sum_{i=1}^n f(x_i) g(x_i) \Delta x,
\end{align*}
which is a Riemann sum. This is the Riemann approximation of the Eq. (\ref{equation_Hilbert_inner_product}). 
This approximation gets more accurate by $\Delta x \rightarrow 0$ or $n \rightarrow \infty$.
Hence, that equation is a valid inner product in the Hilbert space. Q.E.D.
\end{proof}

\begin{remark}[Interpretation of Inner Product of Functions in Hilbert Space]
The inner product of two functions, i.e. Eq. (\ref{equation_Hilbert_inner_product}), measures how similar two functions are. The more similar they are in their domain $[a,b]$, the larger inner product they have. 
Note that this similarity is more about the pattern (or changes) of functions and not the exact value of functions. If the pattern of functions is very similar, they will have a large inner product. 
\end{remark}

\begin{corollary}[Weighted Inner Product in Hilbert Space \cite{williams2000effect}, {\citep[Section 2]{bengio2003spectral}}]
The Eq. (\ref{equation_Hilbert_inner_product}) is the inner product with uniform weighting. With density function $p(x)$, one can weight the inner product in the Hilbert space as (assuming the functions are real):
\begin{align}\label{equation_Hilbert_inner_product_weighted}
\langle f(x), g(x)  \rangle_\mathcal{H} = \int_a^b f(x)\, g(x)\, p(x)\, dx.
\end{align}
\end{corollary}

\subsection{Eigenfunctions}

Recall eigenvalue problem for a matrix $\b{A}$ \cite{ghojogh2019eigenvalue}:
\begin{align}
\b{A}\, \b{\phi}_i = \lambda_i\, \b{\phi}_i, \quad \forall i \in \{1, \dots, d\},
\end{align}
where $\b{\phi}_i$ and $\lambda_i$ are the $i$-th eigenvector and eigenvalue of $\b{A}$, respectively. 
In the following, we introduce the Eigenfunction problem which has a similar form but for an operator rather than a matrix. 

\begin{definition}[Eigenfunction {\citep[Chapter 11.2]{kusse2006mathematical}}]\label{definition_eigenfunction}
Consider a linear operator $O$ which can be applied on a function $f$. If applying this operator on the function results in a multiplication of function to a constant:
\begin{align}\label{equation_eigenfunction}
O f = \lambda f,
\end{align}
then the function $f$ is an eigenfunction for the operator $O$ and the constant $\lambda$ is the corresponding eigenvalue.
Note that the form of eigenfunction problem is:
\begin{align}
\text{Operator }( \text{function }f ) = \text{constant } \times \text{function }f.
\end{align}
\end{definition}
Some examples of operator are derivative, kernel function, etc. For example, $e^{\lambda x}$ is an eigenfunction of derivative because $\frac{d}{dx} e^{\lambda x} = \lambda e^{\lambda x}$.
Note that eigenfunctions have application in many fields of science including machine learning \cite{bengio2003spectral} and quantum mechanics \cite{reed1972methods}.

Recall that in eigenvalue problem, the eigenvectors show the most important or informative directions of matrix and the corresponding eigenvalue shows the amount of importance \cite{ghojogh2019eigenvalue}. Likewise, in eigenfunction problem of an operator, the eigenfunction is the most important function of the operator and the corresponding eigenvalue shows the amount of this importance. 
This connection between eigenfunction and eigenvalue problems is proved in the following theorem. 

\begin{theorem}[Connection of Eigenfunction and Eigenvalue Problems]\label{theorem_connection_eigenfunction_eigenvalue_problems}
If we assume that the operator and the function are a matrix and a vector, eigenfunction problem is converted to an eigenvalue problem where the vector is the eigenvector of the matrix. 
\end{theorem}
\begin{proof}
Consider any function space such as a Hilbert space. Let $\{\b{e}_j\}_{j=1}^n$ be the bases (basis functions) of this function space where $n$ may be infinite. The function $f$ in this space can be represented as a linear combination bases:
\begin{align}\label{equation_proof_eigenfunction_f}
f(\b{x}) = \sum_{j=1}^n \alpha_j\, \b{e}_j(\b{x}).
\end{align}
An example of this linear combination is Eq. (\ref{equation_RKHS}) in RKHS where the bases are kernels. 
Consider the operator $O$ which can be applied on the functions in this function space.
Applying this operator on Eq. (\ref{equation_proof_eigenfunction_f}) gives:
\begin{align}\label{equation_proof_eigenfunction_O_f_1}
O f(\b{x}) = O\sum_{j=1}^n \alpha_j\, \b{e}_j(\b{x}) \overset{(a)}{=} \sum_{j=1}^n \alpha_j\, O\b{e}_j(\b{x}),
\end{align}
where $(a)$ is because the operator $O$ is a linear operator according to Definition \ref{definition_eigenfunction}.
Also, we have: 
\begin{align}\label{equation_proof_eigenfunction_lambda_f}
O f(\b{x}) \overset{(\ref{equation_eigenfunction})}{=} \lambda f(\b{x}) \overset{(a)}{=} \sum_{j=1}^n \lambda\, \alpha_j\, \b{e}_j(\b{x}),
\end{align}
where $(a)$ is because $\lambda$ is a scalar. 

On the other hand, the output function from applying the operator on a function can also be written as a linear combination of the bases:
\begin{align}\label{equation_proof_eigenfunction_O_f_2}
O f(\b{x}) = \sum_{j=1}^n \beta_j\, \b{e}_j(\b{x}).
\end{align}
From Eqs. (\ref{equation_proof_eigenfunction_O_f_1}) and (\ref{equation_proof_eigenfunction_O_f_2}), we have:
\begin{align}\label{equation_proof_eigenfunction_O_f_3}
\sum_{j=1}^n \alpha_j\, O\b{e}_j(\b{x}) = \sum_{j=1}^n \beta_j\, \b{e}_j(\b{x}).
\end{align}

In parentheses, consider an $n \times n$ matrix $\b{A}$ whose $(i,j)$-th element is the inner product of $\b{e}_i$ and $O \b{e}_j$:
\begin{align}\label{equation_proof_eigenfunction_A}
\b{A}(i,j) := \langle \b{e}_i, O\b{e}_j \rangle_k \overset{(\ref{equation_Hilbert_inner_product})}{=} \int \b{e}_i^*(\b{x})\, O\b{e}_j(\b{x}) d\b{x},
\end{align}
where integral is over the domain of functions in the function space. 

Using Eq. (\ref{equation_Hilbert_inner_product}), we take the inner product of sides of Eq. (\ref{equation_proof_eigenfunction_O_f_3}) with an arbitrary basis function $\b{e}_i$:
\begin{align*}
\sum_{j=1}^n \alpha_j\, \int \b{e}_i^*(\b{x})\, O\b{e}_j(\b{x})\, d\b{x} = \sum_{j=1}^n \beta_j\, \int \b{e}_i^*(\b{x})\, \b{e}_j(\b{x})\, d\b{x}.
\end{align*}
According to Eq. (\ref{equation_proof_eigenfunction_A}), this equation is simplified to:
\begin{align}\label{equation_proof_eigenfunction_sum_alpha_A_beta}
\sum_{j=1}^n \alpha_j\, \b{A}(i,j) &= \sum_{j=1}^n \beta_j\, \int \b{e}_i^*(\b{x})\, \b{e}_j(\b{x})\, d\b{x} \overset{(a)}{=} \beta_i,
\end{align}
which is true for $\forall i \in \{1, \dots n\}$ and $(a)$ is because the bases are orthonormal, so:
\begin{align*}
\langle \b{e}_i, \b{e}_j \rangle_k \overset{(\ref{equation_Hilbert_inner_product})}{=}
\int \b{e}_i^*(\b{x})\, \b{e}_j(\b{x})\, d\b{x} =
\left\{
    \begin{array}{ll}
        1 & \mbox{if } i = j, \\
        0 & \mbox{Otherwise.}
    \end{array}
\right.
\end{align*}
The Eq. (\ref{equation_proof_eigenfunction_sum_alpha_A_beta}) can be written in matrix form:
\begin{align}\label{equation_proof_eigenfunction_A_alpha_equals_beta}
\b{A} \b{\alpha} = \b{\beta},
\end{align}
where $\b{\alpha} := [\alpha_1, \dots, \alpha_n]^\top$ and $\b{\beta} := [\beta_1, \dots, \beta_n]^\top$.

From Eqs. (\ref{equation_proof_eigenfunction_lambda_f}) and (\ref{equation_proof_eigenfunction_O_f_2}), we have:
\begin{align}\label{equation_proof_eigenfunction_lambda_alpha_equals_beta}
\sum_{j=1}^n \lambda\, \alpha_j\, \b{e}_j(\b{x}) = \sum_{j=1}^n \beta_j\, \b{e}_j(\b{x}) \implies \lambda\, \b{\alpha} = \b{\beta}.
\end{align}
Comparing Eqs. (\ref{equation_proof_eigenfunction_A_alpha_equals_beta}) and (\ref{equation_proof_eigenfunction_lambda_alpha_equals_beta}) shows:
\begin{align*}
\b{A} \b{\alpha} = \lambda\, \b{\alpha},
\end{align*}
which is an eigenvalue problem for matrix $\b{A}$ with eigenvector $\b{\alpha}$ and eigenvalue $\lambda$ \cite{ghojogh2019eigenvalue}. 
Note that, according to Eq. (\ref{equation_proof_eigenfunction_f}), the information of function $f$ is in the coefficients $\alpha_j$'s of the basis functions of space. Therefore, the function is converted to the eigenvector (vector of coefficients) and the operator $O$ is converted to the matrix $\b{A}$. Q.E.D.
\end{proof}

\subsection{Use of Eigenfunctions for Spectral Embedding}\label{section_eigenfunctions_for_spectral_embedding}

Consider a Hilbert space $\mathcal{H}$ of functions with the inner product defined by Eq. (\ref{equation_Hilbert_inner_product_weighted}). Let the data in the input space be $\mathcal{X} = \{\b{x}_i \in \mathbb{R}^d\}_{i=1}^n$.
In this space, we can consider an operator for the kernel function $K_p$ as \cite{williams2000effect}, {\citep[Section 3]{bengio2003out}}:
\begin{align}\label{equation_K_operator_integral}
(K_p f)(\b{x}) := \int k(\b{x},\b{y})\, f(\b{y})\, p(\b{y})\, d\b{y},
\end{align}
where $f \in \mathcal{H}$ and the density function $p(\b{y})$ can be approximated empirically. 
A discrete approximation of this operator is \cite{williams2000effect}:
\begin{align}\label{equation_discrete_kernel_operator}
(K_{p,n} f)(\b{x}) := \frac{1}{n} \sum_{i=1}^n k(\b{x},\b{x}_i)\, f(\b{x}_i),
\end{align}
which converges to Eq. (\ref{equation_K_operator_integral}) if $n \rightarrow \infty$. Note that this equation is also mentioned in {\citep[Section 2]{bengio2003spectral}}, {\citep[Section 4]{bengio2004learning}}, {\citep[Section 3.2]{bengio2006spectral}}.

\begin{lemma}[Relation of Eigenvalues of Eigenvalue Problem and Eigenfunction Problem for Kernel {\citep[Proposition 1]{bengio2003out}}, {\citep[Theorem 1]{bengio2003spectral}}, {\citep[Section 4]{bengio2004learning}}]
Assume $\lambda_k$ denotes the $k$-th eigenvalue for eigenfunction decomposition of the operator $K_p$ and $\delta_k$ denotes the $k$-th eigenvalue for eigenvalue problem of the matrix $\b{K} \in \mathbb{R}^{n \times n}$. We have:
\begin{align}\label{equation_connection_eigenvalues_for_kernel_operator}
\delta_k = n\, \lambda_k.
\end{align}
\end{lemma}
\begin{proof}
This proof gets help from {\citep[proof of Proposition 3]{bengio2003learning}}.
According to Eq. (\ref{equation_eigenfunction}), the eigenfunction problems for the operators $K_p$ and $K_{p,n}$ (discrete version) are:
\begin{equation}\label{equation_eigenfunction_for_discrete_kernel_operator}
\begin{aligned}
&(K_p f_k)(\b{x}) = \lambda_k f_k(\b{x}), \quad \forall k \in \{1, \dots, n\}, \\
&(K_{p,n} f_k)(\b{x}) = \lambda_k f_k(\b{x}), \quad \forall k \in \{1, \dots, n\}, 
\end{aligned}
\end{equation}
where $f_k(.)$ is the $k$-th eigenfunction and $\lambda_k$ is the corresponding eigenvalue.
Consider the kernel matrix defined by Definition \ref{definition_Gram_matrix}.
The eigenvalue problem for the kernel matrix is \cite{ghojogh2019eigenvalue}: 
\begin{align}\label{equation_eigenvalue_problem_kernel_operator}
\b{K} \b{v}_k = \delta_k \b{v}_k, \quad \forall k \in \{1, \dots, n\},
\end{align}
where $\b{v}_k$ is the $k$-th eigenvector and $\delta_k$ is the corresponding eigenvalue. 
According to Eqs. (\ref{equation_discrete_kernel_operator}) and (\ref{equation_eigenfunction_for_discrete_kernel_operator}), we have:
\begin{align*}
\frac{1}{n} \sum_{i=1}^n k(\b{x},\b{x}_i)\, f(\b{x}_i) = \lambda_k f_k(\b{x}), \quad \forall k \in \{1, \dots, n\}.
\end{align*}
When this equation is evaluated only at $\b{x}_i \in \mathcal{X}$, we have {\citep[Section 4]{bengio2004learning}}, {\citep[Section 3.2]{bengio2006spectral}}:
\begin{align*}
&\frac{1}{n} \b{K} f_k = \lambda_k f_k, \quad \forall k \in \{1, \dots, n\}, \\
&\implies \b{K} f_k = n \lambda_k f_k.
\end{align*}
According to Theorem \ref{theorem_connection_eigenfunction_eigenvalue_problems}, eigenfunction can be seen as an eigenvector. If so, we can say:
\begin{align}\label{equation_eigenvalue_problem_kernel_operator_2}
\b{K} f_k = n \lambda_k f_k \implies \b{K} \b{v}_k = n \lambda_k \b{v}_k,
\end{align}
Comparing Eqs. (\ref{equation_eigenvalue_problem_kernel_operator}) and (\ref{equation_eigenvalue_problem_kernel_operator_2}) results in Eq. (\ref{equation_connection_eigenvalues_for_kernel_operator}). Q.E.D.
\end{proof}

\begin{lemma}[Relation of Eigenvalues of Kernel and Covariance in the Feature Space \cite{scholkopf1998nonlinear}]
Consider the covariance of pulled data to the feature space:
\begin{align}\label{equation_covariance_in_feature_space}
\b{C}_H := \frac{1}{n} \sum_{i=1}^n \breve{\phi}(\b{x}_i) \breve{\phi}(\b{x}_i)^\top,
\end{align}
where $\breve{\phi}(\b{x}_i)$ is the centered pulled data defined by Eq. (\ref{equation_appendix_centered_pulled_outOfSample}).

which is $t \times t$ dimensional where $t$ may be infinite. 
Assume $\eta_k$ denotes the $k$-th eigenvalue $\b{C}_H$ and $\delta_k$ denotes the $k$-th eigenvalue of centered kernel $\breve{\b{K}}$. We have:
\begin{align}\label{equation_connection_eigenvalues_for_kernel_and_covariance}
\delta_k = n\, \eta_k.
\end{align}
\end{lemma}
\begin{proof}
This proof is based on {\citep[Section 2]{scholkopf1998nonlinear}}.
The eigenvalue problem for this covariance matrix is:
\begin{align*}
\eta_k\, \b{u}_k = \b{C}_H\, \b{u}_k, \quad \forall k \in \{1, \dots, n\},
\end{align*}
where $\b{u}_k$ is the $k$-th eigenvector and $\eta_k$ is its corresponding eigenvalue \cite{ghojogh2019eigenvalue}. 
Left multiplying this equation with $\breve{\phi}(\b{x}_j)^\top$ gives:
\begin{align}\label{equation_eigenvalue_problem_of_covariance_in_feature_space}
\eta_k\, \breve{\phi}(\b{x}_j)^\top \b{u}_k = \breve{\phi}(\b{x}_j)^\top \b{C}_H\, \b{u}_k, \quad \forall k \in \{1, \dots, n\}.
\end{align}
As $\b{u}_k$ is the eigenvector of the covariance matrix in the feature space, it lies in the feature space; hence, according to Lemma \ref{lemma_representation_function_bases} which will come later, we can represent it as:
\begin{align}\label{equation_eigenvector_of_covariance_in_feature_space}
\b{u}_k = \frac{1}{\sqrt{\delta_k}} \sum_{\ell=1}^n v_\ell\, \breve{\phi}(\b{x}_\ell), 
\end{align}
where pulled data to feature space are assumed to be centered, $v_\ell$'s are the coefficients in representation, and the normalization by $1/\sqrt{\delta_k}$ is because of a normalization used in {\citep[Section 4]{bengio2003spectral}}.
Substituting Eq. (\ref{equation_eigenvector_of_covariance_in_feature_space}) and Eq. (\ref{equation_covariance_in_feature_space}) in Eq. (\ref{equation_eigenvalue_problem_of_covariance_in_feature_space}) results in:
\begin{align*}
\eta_k\, \breve{\phi}(\b{x}_j)^\top &\sum_{\ell=1}^n v_\ell\, \breve{\phi}(\b{x}_\ell) \\
&= \breve{\phi}(\b{x}_j)^\top \frac{1}{n} \sum_{i=1}^n \breve{\phi}(\b{x}_i) \breve{\phi}(\b{x}_i)^\top\, \sum_{\ell=1}^n v_\ell\, \breve{\phi}(\b{x}_\ell),
\end{align*}
where normalization factors are simplified from sides. 
In the right-hand side, as the summations are finite, we are allowed to re-arrange them. Re-arranging the terms in this equation gives:
\begin{align*}
&\eta_k \sum_{\ell=1}^n v_\ell\, \breve{\phi}(\b{x}_j)^\top \breve{\phi}(\b{x}_\ell) \\
&~~~~~~ =\frac{1}{n} \sum_{\ell=1}^n v_\ell\, \Big( \breve{\phi}(\b{x}_j)^\top \sum_{i=1}^n \breve{\phi}(\b{x}_i) \Big) \Big( \breve{\phi}(\b{x}_i)^\top\, \breve{\phi}(\b{x}_\ell) \Big).
\end{align*}
Considering Eqs. (\ref{equation_kernel_inner_product_matrix}) and (\ref{equation_centered_kernel_definition}), we can write this equation in matrix form $\eta_k \breve{\b{K}} \b{v}_k = \frac{1}{n} \breve{\b{K}}^2 \b{v}_k$ where $\b{v}_k := [v_1, \dots, v_n]^\top$. As $\breve{\b{K}}$ is positive semi-definite (see Lemma \ref{lemma_kernel_is_positive_semidefinite}), it is often non-singular. For non-zero eigenvalues, we can left multiply this equation to $\breve{\b{K}}^{-1}$ to have:
\begin{align*}
n\, \eta_k\, \b{v}_k = \breve{\b{K}} \b{v}_k,
\end{align*}
which is the eigenvalue problem for $\breve{\b{K}}$ where $\b{v}$ is the eigenvector and $\delta_k = n\, \eta_k$ is the eigenvalue (cf. Eq. (\ref{equation_eigenvalue_problem_kernel_operator})). Q.E.D.
\end{proof}


\begin{lemma}[Relation of Eigenfunctions and Eigenvectors for Kernel {\citep[Proposition 1]{bengio2003out}}, {\citep[Theorem 1]{bengio2003spectral}}]\label{lemma_relation_eigenfunctions_eigenvectors_for_kernel}
Consider a training dataset $\{\b{x}_i \in \mathbb{R}^d\}_{i=1}^n$ and the eigenvalue problem (\ref{equation_eigenvalue_problem_kernel_operator}) where $\b{v}_k \in \mathbb{R}^n$ and $\delta_k$ are the $k$-th eigenvector and eigenvalue of matrix $\b{K} \in \mathbb{R}^{n \times n}$.
If $v_{ki}$ is the $i$-th element of vector $\b{v}_k$, the eigenfunction for the point $\b{x}$ and the $i$-th training point $\b{x}_i$ are:
\begin{align}
&f_k(\b{x}) = \frac{\sqrt{n}}{\delta_k} \sum_{i=1}^n v_{ki}\, \breve{k}(\b{x}_i, \b{x}), \label{equation_relation_eigenfunction_eigenvector_x} \\
&f_k(\b{x}_i) = \sqrt{n}\, v_{ki}, \label{equation_relation_eigenfunction_eigenvector_x_i}
\end{align}
respectively, where $\breve{k}(\b{x}_i, \b{x})$ is the centered kernel. 
If $\b{x}$ is a training point, $\breve{k}(\b{x}_i, \b{x})$ is the centered kernel over training data and if $\b{x}$ is an out-of-sample point, then $\breve{k}(\b{x}_i, \b{x}) = \breve{k}_t(\b{x}_i, \b{x})$ is between training set and the out-of-sample point (n.b. kernel centering is explained in Section \ref{section_kernel_centering}). 
\end{lemma}
\begin{proof}
For proof of Eq. (\ref{equation_relation_eigenfunction_eigenvector_x}), see {\citep[proof of Theorem 1]{bengio2003spectral}} or {\citep[Section 1.1]{williams2001using}}. 
The Eq. (\ref{equation_relation_eigenfunction_eigenvector_x_i}) is claimed in {\citep[Proposition 1]{bengio2003spectral}}. 
For proof of this equation, see {\citep[proof of Theorem 1, Eq. 7]{bengio2003spectral}}. 
\end{proof}
It is noteworthy that Eq. (\ref{equation_relation_eigenfunction_eigenvector_x}) is similar and related to the Nystr{\"o}m approximation of eigenfunctions of kernel operator which will be explained in Lemma \ref{lemma_Nystrom_approx_eigenfunctions}.

\begin{theorem}[Embedding from Eigenfunctions of Kernel Operator {\citep[Proposition 1]{bengio2003out}}, {\citep[Section 4]{bengio2003spectral}}]\label{theorem_embedding_from_eigenfunctions}
Consider a dimensionality reduction algorithm which embeds data into a low-dimensional embedding space.
Let the embedding of the point $\b{x}$ be $\mathbb{R}^p \ni \b{y}(\b{x}) = [y_1(\b{x}), \dots, y_p(\b{x})]^\top$ where $p \leq n$. The $k$-th dimension of this embedding is:
\begin{align}\label{equation_embedding_eigenfunction}
y_k(\b{x}) &= \sqrt{\delta_k}\, \frac{f_k(\b{x})}{\sqrt{n}} = \frac{1}{\sqrt{\delta_k}} \sum_{i=1}^n v_{ki}\, \breve{k}(\b{x}_i, \b{x}),
\end{align}
where $\breve{k}(\b{x}_i, \b{x})$ is the centered training or out-of-sample kernel depending on whether $\b{x}$ is a training or an out-of-sample point (n.b. kernel centering will be explained in Section \ref{section_kernel_centering}). 
\end{theorem}
\begin{proof}
We can embed data point $\b{x}$ by pulling it to the feature space and centering the pulled dataset to have $\breve{\phi}(\b{x})$ and then projecting it onto the eigenvector of covariance matrix in the feature space {\citep[Section 2]{scholkopf1998nonlinear}}, {\citep[Section 4]{bengio2003spectral}}:
\begin{align*}
y_k(\b{x}) &= \b{u}_k^\top \breve{\phi}(\b{x}) \overset{(\ref{equation_eigenvector_of_covariance_in_feature_space})}{=} \frac{1}{\sqrt{\delta_k}} \sum_{i=1}^n v_i\, \breve{\phi}(\b{x}_i)^\top \breve{\phi}(\b{x}) \\
&\overset{(\ref{equation_centered_kernel_definition})}{=} \frac{1}{\sqrt{\delta_k}} \sum_{i=1}^n v_{ki}\, \breve{k}(\b{x}_i, \b{x}).
\end{align*}
Q.E.D.
\end{proof}
The Theorem \ref{theorem_embedding_from_eigenfunctions} has been widely used for out-of-sample (test data) embedding in many spectral dimensionality reduction algorithms \cite{bengio2003out}.

\begin{corollary}[Embedding from Eigenvectors of Kernel Matrix]
Consider the eigenvalue problem for the kernel matrix, i.e. Eq. (\ref{equation_eigenvalue_problem_kernel_operator}), where $\b{v}_k = [v_{k1}, \dots, v_{kn}]^\top$ and $\delta_k$ are the $k$-th eigenvector and eigenvalue of kernel, respectively. According to Eqs. (\ref{equation_relation_eigenfunction_eigenvector_x_i}) and (\ref{equation_embedding_eigenfunction}), we can compute the embedding of point $\b{x}$, denoted by $\b{y}(\b{x}) = [y_1(\b{x}), \dots, y_p(\b{x})]^\top$ (where $p \leq n$) using the eigenvector of kernel as:
\begin{align}\label{equation_embedding_eigenvector_of_kernel}
y_k(\b{x}) = \sqrt{\delta_k}\, \frac{1}{\sqrt{n}} (\sqrt{n}) v_{ki} = \sqrt{\delta_k}\, v_{ki}. 
\end{align}
\end{corollary}
The Eq. (\ref{equation_embedding_eigenvector_of_kernel}) is used in several dimensionality reduction methods such as maximum variance unfolding (or semidefinite embedding) \cite{weinberger2005nonlinear,weinberger2006unsupervised,weinberger2006introduction}. We will introduce this method in Section \ref{section_kernel_learning}. 

\section{Kernelization Techniques}\label{section_kernelization_techniques}

Linear algorithms cannot properly handle nonlinear patterns of data obviously. 
When dealing with nonlinear data, if the algorithm is linear, two solutions exist to have acceptable performance:
\begin{enumerate}
\item Either the linear method should be modified to become nonlinear or a completely new nonlinear algorithm should be proposed to be able to handle nonlinear data. Some examples of this category are nonlinear dimensionality methods such as locally linear embedding \cite{ghojogh2020locally} and Isomap \cite{ghojogh2020multidimensional}. 
\item Or the nonlinear data should be modified in a way to become more linear in pattern. In other words, a transformation should be applied on data so that the pattern of data becomes roughly linear or easier to process by the linear algorithm. Some examples of this category are kernel versions of linear methods such as kernel Principal Component Analysis (PCA) \cite{scholkopf1997kernel,scholkopf1998nonlinear,ghojogh2019unsupervised}, kernel Fisher Discriminant Analysis (FDA) \cite{mika1999fisher,ghojogh2019fisher}, and kernel Support Vector Machine (SVM) \cite{boser1992training,vapnik1995nature}. 
\end{enumerate}
The second approach is called kernelization in machine learning which we define in the following. 
Figure \ref{figure_separating_classes_kernel} shows how kernelization for transforming data can help separate classes for better classification. 

\begin{figure}[!t]
\centering
\includegraphics[width=3.4in]{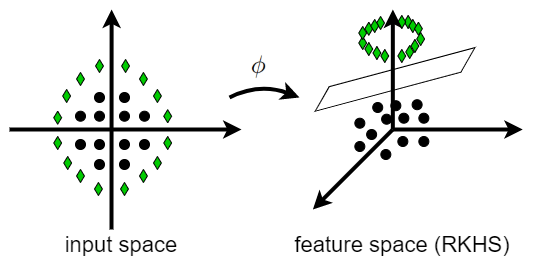}
\caption{Transforming data to RKHS using kernels to make the nonlinear pattern of data more linear. For example, here the classes have become linearly separable (by a linear hyperplane) after kernelization.}
\label{figure_separating_classes_kernel}
\end{figure}

\begin{definition}[Kernelization]
In machine learning and data science, kernelization means a slight change in algorithm formulation (without any modification in the idea of algorithm) so that the pulled data to the RKHS, rather than the raw data, are used as input of algorithm. 
\end{definition}
Note that kernelization can be useful for enabling linear algorithms to handle nonlinear data better. Nevertheless, it should be noted that nonlinear algorithms can also be kernelized to be able to handle nonlinear data perhaps better by transforming data. 

Generally, there exist two main approaches for kernelization in machine learning. These two approaches are related in theory but have two ways for kernelization. 
In the following, we explain these methods which are kernel trick and kernelization using representation theory. 

\subsection{Kernelization by Kernel Trick}\label{section_kernel_trick}

Recall Eqs. (\ref{equation_kernel_inner_product}) and (\ref{equation_kernel_inner_product_matrix}) where kernel can be computed by inner product between pulled data instances to the RKHS. 
One technique to kernelize an algorithm is kernel trick. In this technique, we first try to formulate the algorithm formulas or optimization in a way that data always appear as inner product of data instances and not a data instance alone. In other words, the formulation of algorithm should only have $\b{x}^\top \b{x}$, $\b{x}^\top \b{X}$, $\b{X}^\top \b{x}$, or $\b{X}^\top \b{X}$ and not a lonely $\b{x}$ or $\b{X}$. In this way, kernel trick replaces $\b{x}^\top \b{x}$ with $\b{\phi}(\b{x})^\top \b{\phi}(\b{x})$ and uses Eq. (\ref{equation_kernel_inner_product}) or (\ref{equation_kernel_inner_product_matrix}). To better explain, kernel trick applies the following mapping \cite{burges1998tutorial}:
\begin{align}\label{equation_kernel_trick}
\b{x}^\top \b{x} \mapsto \b{\phi}(\b{x})^\top \b{\phi}(\b{x}) \overset{(\ref{equation_kernel_inner_product})}{=} k(\b{x}, \b{x}).
\end{align}
Therefore, the inner products of points are all replaced with the kernel between points. The matrix form of kernel trick is:
\begin{align}\label{equation_kernel_trick_matrix}
\b{X}^\top \b{X} \mapsto \b{\Phi}(\b{X})^\top \b{\Phi}(\b{X}) \overset{(\ref{equation_kernel_inner_product_matrix})}{=} \b{K}(\b{X}, \b{X}) \in \mathbb{R}^{n \times n}.
\end{align}

Most often, kernel matrix is computed over one dataset; hence, its dimensionality is $n \times n$. 
However, in some cases, the kernel matrix is computed between two sets of data instances with sample sizes $n_1$ and $n_2$ for example, i.e. datasets $\b{X}_1 := [\b{x}_{1,1}, \dots, \b{x}_{1,n_1}]$ and $\b{X}_2 := [\b{x}_{2,1}, \dots, \b{x}_{2,n_2}]$. In this case, the kernel matrix has size $n_1 \times n_2$ and the kernel trick is:
\begin{align}
&\b{x}_{1,i}^\top \b{x}_{1,j} \mapsto \b{\phi}(\b{x}_{1,i})^\top \b{\phi}(\b{x}_{1,j}) \overset{(\ref{equation_kernel_inner_product})}{=} k(\b{x}_{1,i}, \b{x}_{1,j}), \label{equation_kernel_trick_two_sets} \\
&\b{X}_1^\top \b{X}_2 \mapsto \b{\Phi}(\b{X}_1)^\top \b{\Phi}(\b{X}_2) \overset{(\ref{equation_kernel_inner_product_matrix})}{=} \b{K}(\b{X}_1, \b{X}_2) \in \mathbb{R}^{n_1 \times n_2}.
\end{align}
An example for kernel between two sets of data is the kernel between training data and out-of-sample (test) data. 
As stated in \cite{scholkopf2001kernel}, the kernel trick is proved to work for Mercer kernels in \cite{boser1992training,vapnik1995nature} or equivalently for the positive definite kernels \cite{berg1984harmonic,wahba1990spline}.

Some examples of using kernel trick in machine learning are kernel PCA \cite{scholkopf1997kernel,scholkopf1998nonlinear,ghojogh2019unsupervised} and kernel SVM \cite{boser1992training,vapnik1995nature}.
More examples will be provided in Section \ref{section_kernel_in_machine_learning}.
Note that in some algorithms, data do not not appear only by inner product which is required for the kernel trick. In these cases, if possible, a ``dual" method for the algorithm is proposed which only uses the inner product of data. Then, the dual algorithm is kernelized using kernel trick. Some examples for this are kernelization of dual PCA \cite{ghojogh2019unsupervised} and dual SVM \cite{burges1998tutorial}. 
As an additional point, it is noteworthy that it is possible to replace kernel trick with function replacement (see \cite{ma2003function} for more details on this).

\subsection{Kernelization by Representation Theory}

As was explained in Section \ref{section_kernel_trick}, if the formulation of algorithm has data only as inner products of points, kernel trick can be used. In some cases, a dual version of algorithm is used to have only inner products. 
Some algorithms, however, cannot be formulated in a way to have data only in inner product form, nor does their dual have this form. An example is Fisher Discriminant Analysis (FDA) \cite{ghojogh2019fisher} which uses another technique for kernelization \cite{mika1999fisher}. In the following, we explain this technique. 

\begin{lemma}[Representation of Function Using Bases \cite{mika1999fisher}]\label{lemma_representation_function_bases}
Consider a RKHS denoted by $\mathcal{H}$. Any function $f \in \mathcal{H}$ lies in the span of all points in the RKHS, i.e.,
\begin{align}
f = \sum_{i=1}^n \alpha_i\, \phi(\b{x}_i).
\end{align}
\end{lemma}
\begin{proof}
Consider Eq. (\ref{equation_RKHS}) which can be restated as:
\begin{align*}
&f(\b{y}) \overset{(\ref{equation_RKHS})}{=} \sum_{i=1}^n \alpha_i\, k(\b{x}_i, \b{y}) \overset{(\ref{equation_kernel_inner_product})}{=} \sum_{i=1}^n \alpha_i\, \phi(\b{x}_i)^\top \phi(\b{y}) \\
&\implies f(.) = \sum_{i=1}^n \alpha_i\, \phi(\b{x}_i).
\end{align*}
Q.E.D.
\end{proof}

\begin{remark}[Justification by Representation Theory]
According to representation theory \cite{alperin1993local}, any function in the space can be represented as a linear combination of bases of the space. This makes sense because the function is in the space and the space is spanned by the bases. Now, assume the space is RKHS. Hence, any function should lie in the RKHS spanned by the pulled data points to the feature space. This justifies Lemma \ref{lemma_representation_function_bases} using representation theory. 
\end{remark}

\subsubsection{Kernelization for Vector Solution}

Now, consider an algorithm whose optimization variable or solution is the vector/direction $\b{u} \in \mathbb{R}^d$ in the input space. 
For kernelization, we pull this solution to RKHS by Eq. (\ref{equation_pulling_mapping}) to have $\phi(\b{u})$. 
According to Lemma \ref{lemma_representation_function_bases}, this pulled solution must lie in the span of all pulled training points $\{\phi(\b{x}_i)\}_{i=1}^n$ as:
\begin{align}\label{equation_representation_theory_phi_u}
\phi(\b{u}) = \sum_{i=1}^n \alpha_i\, \phi(\b{x}_i) = \b{\Phi}(\b{X})\, \b{\alpha},
\end{align}
which is $t$ dimensional and $t$ may be infinite. Note that $\b{\Phi}(\b{X})$ is defined by Eq. (\ref{equation_Phi_X_pulled_matrix}) and is $t \times n$ dimensional. The vector $\b{\alpha} := [\alpha_1, \dots, \alpha_n]^\top \in \mathbb{R}^n$ contains the coefficients. 
According to Eq. (\ref{equation_pulling_mapping}), we can replace $\b{u}$ with $\phi(\b{u})$ in the algorithm. If, by this replacement, the terms $\phi(\b{x}_i)^\top \phi(\b{x}_i)$ or $\b{\Phi}(\b{X})^\top \b{\Phi}(\b{X})$ appear, then we can use Eq. (\ref{equation_kernel_inner_product}) and replace $\phi(\b{x}_i)^\top \phi(\b{x}_i)$ with $k(\b{x}_i, \b{x}_i)$ or use Eq. (\ref{equation_kernel_inner_product_matrix}) to replace $\b{\Phi}(\b{X})^\top \b{\Phi}(\b{X})$ with $\b{K}(\b{X}, \b{X})$. 
This kernelizes the method. The steps of kernelization by representation theory are summarized below:
\begin{itemize}[topsep=0pt,itemsep=-1ex,partopsep=1ex,parsep=1ex]
\item Step 1: $\b{u} \rightarrow \phi(\b{u})$
\item Step 2: Replace $\phi(\b{u})$ with Eq. (\ref{equation_representation_theory_phi_u}) in the algorithm formulation
\item Step 3: Some $\phi(\b{x}_i)^\top \phi(\b{x}_i)$ or $\b{\Phi}(\b{X})^\top \b{\Phi}(\b{X})$ terms appear in the formulation
\item Step 4: Use Eq. (\ref{equation_kernel_inner_product}) or (\ref{equation_kernel_inner_product_matrix})
\item Step 5: Solve (optimize) the algorithm where the variable to find is $\b{\alpha}$ rather than $\b{u}$
\end{itemize}
Usually, the goal of algorithm results in kernel. For example, if $\b{u}$ is a projection direction, the desired projected data are obtained as:
\begin{align}
\b{u}^\top \b{x}_i \overset{(\ref{equation_pulling_mapping})}{\mapsto} \phi(\b{u})^\top \phi(\b{x}_i) &\overset{(\ref{equation_representation_theory_phi_u})}{=} \b{\alpha}^\top \b{\Phi}(\b{X})^\top \phi(\b{x}_i) \nonumber \\
&\overset{(\ref{equation_kernel_inner_product_matrix})}{=} \b{\alpha}^\top \b{k}(\b{X}, \b{x}_i),
\end{align}
where $\b{k}(\b{X}, \b{x}_i) \in \mathbb{R}^n$ is the kernel between all $n$ training points with the point $\b{x}_i$.
As this equation shows, the desired goal is based on kernel. 

\subsubsection{Kernelization for Matrix Solution}

Usually, the algorithm has multiple directions/vectors as its solution. In other words, its solution is a matrix $\b{U} = [\b{u}_1, \dots, \b{u}_p] \in \mathbb{R}^{d \times p}$. In this case, Eq. (\ref{equation_representation_theory_phi_u}) is used for all $p$ vectors and in a matrix form, we have:
\begin{align}\label{equation_representation_theory_phi_u_matrix}
\b{\Phi}(\b{U}) = \b{\Phi}(\b{X})\, \b{A}, 
\end{align}
where $\mathbb{R}^{n \times p} \ni \b{A} := [\b{\alpha}_1, \dots, \b{\alpha}_p]^\top$ and $\b{\Phi}(\b{U})$ is $t \times p$ dimensional where $t$ may be infinite. Similarly, the following steps should be performed to kernelize the algorithm:
\begin{itemize}[topsep=0pt,itemsep=-1ex,partopsep=1ex,parsep=1ex]
\item Step 1: $\b{U} \rightarrow \phi(\b{U})$
\item Step 2: Replace $\phi(\b{U})$ with Eq. (\ref{equation_representation_theory_phi_u_matrix}) in the algorithm formulation
\item Step 3: Some $\b{\Phi}(\b{X})^\top \b{\Phi}(\b{X})$ terms appear in the formulation
\item Step 4: Use Eq. (\ref{equation_kernel_inner_product_matrix})
\item Step 5: Solve (optimize) the algorithm where the variable to find is $\b{A}$ rather than $\b{U}$
\end{itemize}
Again the goal of algorithm usually results in kernel. For example, if $\b{U}$ is a projection matrix onto its column space, we have:
\begin{align}
\b{U}^\top \b{x}_i \overset{(\ref{equation_pulling_mapping})}{\mapsto} \phi(\b{U})^\top \phi(\b{x}_i) &\overset{(\ref{equation_representation_theory_phi_u_matrix})}{=} \b{A}^\top \b{\Phi}(\b{X})^\top \phi(\b{x}_i) \nonumber \\
&\overset{(\ref{equation_kernel_inner_product_matrix})}{=} \b{A}^\top \b{k}(\b{X}, \b{x}_i),
\end{align}
where $\b{k}(\b{X}, \b{x}_i) \in \mathbb{R}^n$ is the kernel between all $n$ training points with the point $\b{x}_i$.
As this equation shows, the desired goal is based on kernel. 

As was mentioned, in many machine learning algorithms, the solution $\b{U} \in \mathbb{R}^{d \times p}$ is a projection matrix for projecting $d$-dimensional data onto a $p$-dimensional subspace. Some example methods which have used kernelization by representation theory are kernel Fisher discriminant analysis (FDA) \cite{mika1999fisher,ghojogh2019fisher},
kernel supervised principal component analysis (PCA) \cite{barshan2011supervised,ghojogh2019unsupervised}, and direct kernel Roweis discriminant analysis (RDA) \cite{ghojogh2020generalized}. 

\begin{remark}[Linear Kernel in Kernelization]\label{remark_linear_kernel_equivalent_to_non_kernelized}
If we use kernel trick, the kernelized algorithm with a linear kernel is equivalent to the non-kernelized algorithm. This is because in linear kernel, we have $\phi(\b{x}) = \b{x}$ and $k(\b{x}, \b{y}) = \b{x}^\top \b{y}$ according to Eq. (\ref{equation_linear_kernel}). So, the kernel trick, which is Eq. (\ref{equation_kernel_trick_two_sets}), maps data as $\b{x}^\top \b{y} \mapsto \phi(\b{x})^\top \phi(\b{y}) = \b{x}^\top \b{y}$ for linear kernel. Therefore, linear kernel does not have any effect when using kernel trick. Examples for this are kernel PCA \cite{scholkopf1997kernel,scholkopf1998nonlinear,ghojogh2019unsupervised} and kernel SVM \cite{boser1992training,vapnik1995nature} which are equivalent to PCA and SVM, respectively, if linear kernel is used. 

However, kernel trick does have impact when using kernelization by representation theory because it finds the inner products of pulled data points after pulling the solution and representation as a span of bases. Hence, kernelized algorithm using representation theory with linear kernel is not equivalent to non-kernelized algorithm. Examples of this are kernel FDA \cite{mika1999fisher,ghojogh2019fisher} and kernel supervised PCA \cite{barshan2011supervised,ghojogh2019unsupervised} which are different from FDA and supervised PCA, respectively, even if linear kernel is used. 
\end{remark}

\section{Types of Use of Kernels in Machine Learning}\label{section_kernel_in_machine_learning}

There are several types of using kernels in machine learning. In the following, we explain these types of usage of kernels. 

\subsection{Kernel Methods}

The first type of using kernels in machine learning is kernelization of algorithms using either kernel trick or representation theory. As was discussed in Section \ref{section_kernelization_techniques}, linear methods can be kernelized to handle nonlinear data better. Even nonlinear algorithms can be kernelized to perform in the feature space rather than the input space.
In machine learning both kernel trick and kernelization by representation theory have been used. We provide some examples for each of these categories: 
\begin{itemize}[topsep=0pt,itemsep=-1ex,partopsep=1ex,parsep=1ex]
\item Examples for kernelization by \textbf{kernel trick}: kernel Principal Component Analysis (PCA) \cite{scholkopf1997kernel,scholkopf1998nonlinear,ghojogh2019unsupervised}, kernel Support Vector Machine (SVM) \cite{boser1992training,vapnik1995nature}.
\item Examples for kernelization by \textbf{representation theory}: kernel supervised PCA \cite{barshan2011supervised,ghojogh2019unsupervised}, kernel Fisher Discriminant Analysis (FDA) \cite{mika1999fisher,ghojogh2019fisher}, direct kernel Roweis Discriminant Analysis (RDA) \cite{ghojogh2020generalized}. 
\end{itemize}

As was discussed in Section \ref{section_introduction}, using kernels was widely noticed when linear SVM \cite{vapnik1974theory} was kernelized in \cite{boser1992training,vapnik1995nature}.
More discussions on kernel SVM can be found in \cite{scholkopf1997comparing,hastie2009elements}. A tutorial on kernel SVM is \cite{burges1998tutorial}. 
Universal kernels, introduced in Section \ref{section_universal_characteristic_kernels}, are widely used in kernel SVM. More detailed discussions and proofs for use of universal kernels in kernel SVM can be read in \cite{steinwart2008support}.

As was discussed in Section \ref{section_kernel_trick}, many machine learning algorithms are developed to have dual versions because inner products of points usually appear in the dual algorithms and kernel trick can be applied on them. Some examples of these are dual PCA \cite{scholkopf1997kernel,scholkopf1998nonlinear,ghojogh2019unsupervised} and dual SVM \cite{boser1992training,vapnik1995nature} yielding to kernel PCA and kernel SVM, respectively. 
In some algorithms, however, either a dual version does not exist or formulation does not allow for merely having inner products of points. In those algorithms, kernel trick cannot be used and representation theory should be used. An example for this is FDA \cite{ghojogh2019fisher}.
Moreover, some algorithms, such as kernel reinforcement learning \cite{ormoneit2002kernel}, use kernel as a measure of similarity (see Remark \ref{remark_kernel_is_similarity}).

\subsection{Kernel Learning}\label{section_kernel_learning}

After development of many spectral dimensionality reduction methods in machine learning, it was found out that many of them are actually special cases of kernel Principal Component Analysis (PCA) \cite{bengio2003spectral,bengio2004learning}. 
Paper \cite{ham2004kernel} has shown that PCA, multidimensional scaling, Isomap, locally linear embedding, and Laplacian eigenmap are special cases of kernel PCA with kernels in the formulation of Eq. (\ref{equation_generalKernel_and_distanceMAtrix}). A list of these kernels can be seen in {\citep[Chapter 2]{strange2014open}} and \cite{ghojogh2019feature}. 

Because of this, some generalized dimensionality reduction methods, such as graph embedding \cite{yan2005graph}, were proposed.
In addition, as many spectral methods are cases of kernel PCA, some researchers tried to learn the best kernel for manifold unfolding. 
Maximum Variance Unfolding (MVU) or Semidefinite Embedding (SDE) \cite{weinberger2005nonlinear,weinberger2006unsupervised,weinberger2006introduction} is a method for kernel learning using Semidefinite Programming (SDP) \cite{vandenberghe1996semidefinite}. MVU is used for manifold unfolding and dimensionality reduction.
Note that kernel learning by SDP has also been used for labeling a not completely labeled dataset and is also used for kernel SVM \cite{lanckriet2004learning,karimi2017summary}. Our focus here is on MVU. 
In the following, we briefly introduce the MVU (or SDE) algorithm. 

\begin{lemma}[Distance in RKHS \cite{scholkopf2001kernel}]\label{lemma_distance_in_RKHS}
The squared Euclidean distance between points in the feature space is:
\begin{align}\label{equation_distance_in_RKHS}
\|\phi(\b{x}_i) - \phi(\b{x}_j)\|_k^2 = k(\b{x}_i, \b{x}_i) + k(\b{x}_j, \b{x}_j) - 2 k(\b{x}_i, \b{x}_j).
\end{align}
\end{lemma}
\begin{proof}
\begin{align*}
&\|\phi(\b{x}_i) - \phi(\b{x}_j)\|_k^2 = \big( \phi(\b{x}_i) - \phi(\b{x}_j) \big)^\top \big( \phi(\b{x}_i) - \phi(\b{x}_j) \big) \\
&= \phi(\b{x}_i)^\top \phi(\b{x}_i) + \phi(\b{x}_j)^\top \phi(\b{x}_j) -  \phi(\b{x}_i)^\top \phi(\b{x}_j) \\
&- \phi(\b{x}_j)^\top \phi(\b{x}_i) \overset{(\ref{equation_symmetric_inner_product})}{=} \phi(\b{x}_i)^\top \phi(\b{x}_i) + \phi(\b{x}_j)^\top \phi(\b{x}_j) \\
&- 2\phi(\b{x}_i)^\top \phi(\b{x}_j) \overset{(\ref{equation_kernel_inner_product})}{=} k(\b{x}_i, \b{x}_i) + k(\b{x}_j, \b{x}_j) - 2 k(\b{x}_i, \b{x}_j).
\end{align*}
Q.E.D.
\end{proof}

MVU desires to unfolds the manifold of data in its maximum variance direction. For example, consider a Swiss roll which can be unrolled to have maximum variance after being unrolled. As trace of matrix is the summation of eigenvalues and kernel matrix is a measure of similarity between points (see Remark \ref{remark_kernel_is_similarity}), the trace of kernel can be used to show the summation of variance of data. Hence, we should maximize $\textbf{tr}(\b{K})$ where $\textbf{tr}(.)$ denotes the trace of matrix. 
MVU pulls data to the RKHS and then unfolds the manifold. This unfolding should not ruin the local distances between points after pulling data to the feature space. Hence, we should preserve the local distances as:
\begin{align}
\|\b{x}_i - &\b{x}_j\|_2^2 \overset{\text{set}}{=} \|\phi(\b{x}_i) - \phi(\b{x}_j)\|_k^2 \nonumber \\
&\overset{(\ref{equation_distance_in_RKHS})}{=} k(\b{x}_i, \b{x}_i) + k(\b{x}_j, \b{x}_j) - 2 k(\b{x}_i, \b{x}_j).
\end{align}
Moreover, according to Lemma \ref{lemma_kernel_is_positive_semidefinite}, the kernel should be positive semidefinite, i.e. $\b{K} \succeq \b{0}$. 
MVU also centers kernel to have zero mean for the pulled dataset in the feature space (see Section \ref{section_kernel_centering}). According to Eq. (\ref{equation_appendix_doubleCentered_training_kernel_sum}), we will have $\sum_{i=1}^n \sum_{j=1}^n \b{K}(i,j) = 0$. In summary, the optimization of MVU is \cite{weinberger2005nonlinear,weinberger2006unsupervised,weinberger2006introduction}:
\begin{equation}
\begin{aligned}
& \underset{\b{K}}{\text{maximize}}
& & \textbf{tr}(\b{K}) \\
& \text{subject to}
& & \|\b{x}_i - \b{x}_j\|_2^2 = \b{K}(i,i) + \b{K}(j,j) - 2 \b{K}(i,j), \\
& & & ~~~~~~~~~~\qquad\qquad\qquad\qquad \; \forall i,j \in \{1, \ldots, n\}, \\
& & & \sum_{i=1}^n \sum_{j=1}^n \b{K}(i,j) = 0, \\
& & & \b{K} \succeq \b{0},
\end{aligned}
\end{equation}
which is a SDP problem \cite{vandenberghe1996semidefinite}. 
Solving this optimization gives the best kernel for maximum variance unfolding of manifold. Then, MVU considers the eigenvalue problem for kernel, i.e. Eq. (\ref{equation_eigenvalue_problem_kernel_operator}), and finds the embedding using Eq. (\ref{equation_embedding_eigenvector_of_kernel}).

\subsection{Use of Kernels for Difference of Distributions}

There exist several different measures for difference of distributions (i.e., PDFs). Some of them make use of kernels and some do not. 
A measure of difference of distributions can be used for (1) calculating the divergence (difference) of a distribution from another reference distribution or (2) convergence of a distribution to another reference distribution using optimization. We will explain the second use of this measure better in Corollary \ref{corollary_characteristic_kernel_convergence_of_distributions}.

For the information of reader, we first enumerate some of the measures without kernels and then introduce the kernel-based measures for difference of distributions. 
One of methods which do not use kernels is the Kullback-Leibler (KL) divergence \cite{kullback1951information}. KL-divergence, which is a relative entropy from one distribution to the other one and has been widely used in deep learning \cite{goodfellow2016deep}.
Another measure is the Wasserstein metric which has been used in generative models \cite{arjovsky2017wasserstein}. 
The integral probability metric \cite{muller1997integral} is another measure for difference of distributions. 

In the following, we introduce some well-known measures for difference of distributions, using kernels. 

\subsubsection{Hilbert-Schmidt Independence Criterion (HSIC)}

Suppose we want to measure the dependence of two random variables. Measuring the correlation between them is easier because correlation is just ``linear'' dependence. 

According to \cite{hein2004kernels}, two random variables $X$ and $Y$ are independent if and only if any bounded continuous functions of them are uncorrelated. Therefore, if we map the samples of two random variables $\{\b{x}\}_{i=1}^n$ and $\{\b{y}\}_{i=1}^n$ to two different (``separable'') RKHSs and have $\phi(\b{x})$ and $\phi(\b{y})$, we can measure the correlation of $\phi(\b{x})$ and $\phi(\b{y})$ in Hilbert space to have an estimation of dependence of $\b{x}$ and $\b{y}$ in the input space. 

The correlation of $\phi(\b{x})$ and $\phi(\b{y})$ can be computed by the Hilbert-Schmidt norm of the cross-covariance of them \cite{gretton2005measuring}. Note that the squared Hilbert-Schmidt norm of a matrix $\b{A}$ is \cite{bell2016trace}:
\begin{align}
||\b{A}||_{HS}^2 := \textbf{tr}(\b{A}^\top \b{A}),
\end{align}
and the cross-covariance matrix of two vectors $\b{x}$ and $\b{y}$ is \cite{gubner2006probability,gretton2005measuring}:
\begin{align}
\mathbb{C}\text{ov}(\b{x}, \b{y}) := \mathbb{E}\Big[&\big(\b{x} - \mathbb{E}(\b{x})\big) \big(\b{y} - \mathbb{E}(\b{y})\big) \Big].
\end{align}
Using the explained intuition, an empirical estimation of the Hilbert-Schmidt Independence Criterion (HSIC) is introduced \cite{gretton2005measuring}:
\begin{align}\label{equation_HSIC}
\text{HSIC}(X,Y) := \frac{1}{(n-1)^2}\, \textbf{tr}(\b{K}_x\b{H}\b{K}_y\b{H}),
\end{align}
where $\b{K}_x := \b{\phi}(\b{x})^\top \b{\phi}(\b{x})$ and $\b{K}_y := \b{\phi}(\b{y})^\top \b{\phi}(\b{y})$ are the kernels over $\b{x}$ and $\b{y}$, respectively. 
The term $1/(n-1)^2$ is used for normalization.
The matrix $\b{H}$ is the centering matrix (see Eq. (\ref{equation_centered_matrix})).
Note that HSIC double-centers one of the kernels and then computes the Hilbert-Schmidt norm between kernels. 

HSIC measures the dependence of two random variable vectors $\b{x}$ and $\b{y}$. Note that $\text{HSIC}=0$ and $\text{HSIC}>0$ mean that $\b{x}$ and $\b{y}$ are independent and dependent, respectively. The greater the HSIC, the greater dependence they have.

\begin{lemma}[Independence of Random Variables Using Cross-Covariance {\citep[Theorem 5]{gretton2010consistent}}]
Two random variables $X$ and $Y$ are independent if and only if $\mathbb{C}\text{ov}(f(\b{x}), f(\b{y})) = 0$ for any pair of bounded continuous functions $(f,g)$. Because of relation of HSIC with the cross-covariance of variables, two random variables are independent if and only if $\text{HSIC}(X,Y) = 0$.
\end{lemma}

\subsubsection{Maximum Mean Discrepancy (MMD)}

MMD, also known as the kernel two sample test and proposed in \cite{gretton2006kernel,gretton2012kernel}, is a measure for difference of distributions. 
For comparison of two distributions, one can find the difference of all moments of the two distributions. However, as the number of moments is infinite, it is intractable to calculate the difference of all moments. One idea to do this tractably is to pull both distributions to the feature space and then compute the distance of all pulled data points from distributions in RKHS. This difference is a suitable estimate for the difference of all moments in the input space. This is the idea behind MMD. 

MMD is a semi-metric \cite{simon2020metrizing} and uses distance in the RKHS \cite{scholkopf2001kernel} (see Lemma \ref{lemma_distance_in_RKHS}). 
Consider PDFs $\mathbb{P}$ and $\mathbb{Q}$ and samples $\{\b{x}_i\}_{i=1}^n \sim \mathbb{P}$ and $\{\b{y}_i\}_{i=1}^n \sim \mathbb{Q}$.
The squared MMD between these PDFs is:
\begin{align}
&\text{MMD}^2(\mathbb{P}, \mathbb{Q}) := \Big\|\frac{1}{n} \sum_{i=1}^n \b{\phi}(\b{x}_i) - \frac{1}{n} \sum_{i=1}^n \b{\phi}(\b{y}_i)\Big\|_k^2 \nonumber \\
&\overset{(\ref{equation_distance_in_RKHS})}{=} \frac{1}{n^2} \sum_{i=1}^n \sum_{j=1}^n k(\b{x}_i, \b{x}_j) + \frac{1}{n^2} \sum_{i=1}^n \sum_{j=1}^n k(\b{y}_i, \b{y}_j) \nonumber \\
&~~~~~~~~~~~~ - \frac{2}{n^2} \sum_{i=1}^n \sum_{j=1}^n k(\b{x}_i, \b{y}_j) \nonumber \\
&= \mathbb{E}_{x}[\b{K}(\b{x}, \b{y})] + \mathbb{E}_{y}[\b{K}(\b{x}, \b{y})] - 2\mathbb{E}_{x}[\mathbb{E}_{y}[\b{K}(\b{x}, \b{y})]], 
\end{align}
where $\mathbb{E}_x[\b{K}(\b{x}, \b{y})]$, $\mathbb{E}_y[\b{K}(\b{x}, \b{y})]$, and $\mathbb{E}_x[\mathbb{E}_y[\b{K}(\b{x}, \b{y})]]$ are average of rows, average of columns, and total average of rows and columns of the kernel matrix, respectively.
Note that MMD $\geq 0$ where MMD $=0$ means the two distributions are equivalent if the used kernel is characteristic (see Corollary \ref{corollary_characteristic_kernel_convergence_of_distributions} which will be provided later). 
MMD has been widely used in machine learning such as generative moment matching networks \cite{li2015generative}.

\begin{remark}[Equivalence of HSIC and MMD \cite{sejdinovic2013equivalence}]
After development of HSIC and MMD measures, it was found out that they are equivalent.
\end{remark}

\subsection{Kernel Embedding of Distributions (Kernel Mean Embedding)}\label{section_kernel_embedding_distributions}


\begin{definition}[Kernel Embedding of Distributions \cite{smola2007hilbert}]
Kernel embedding of distributions, also called the Kernel Mean Embedding (KME) or mean map, represents (or embeds) Probability Density Functions (PDFs) in a RKHS. 
\end{definition}

\begin{corollary}[Distribution Embedding in Hilbert Space]
Inspired by Eq. (\ref{equation_RKHS}) or (\ref{equation_representer_theorem}), if we map a PDF $\mathbb{P}$ from its space $\mathcal{X}$ to the Hilbert space $\mathcal{H}$, its mapped PDF, denoted by $\phi(\mathbb{P})$, can be represented as:
\begin{align}\label{equation_mapping_PDF_to_RKHS}
\mathbb{P} \mapsto \phi(\mathbb{P}) = \int_\mathcal{X} k(\b{x}, .)\, d\mathbb{P}(\b{x}).
\end{align}
This integral is called the Bochner integral.
\end{corollary}

KME and MMD were first proposed in the field of pure mathematics \cite{guilbart1978etude}.  
Later on, KME and MMD were used in machine learning, first in \cite{smola2007hilbert}.
KME is a family on methods which use Eq. (\ref{equation_mapping_PDF_to_RKHS}) for embedding PDFs in RKHS. This family of methods is more discussed in \cite{sriperumbudur2010hilbert}. 
A survey on KME is \cite{muandet2016kernel}. 

Universal kernels, introduced in Section \ref{section_universal_kernels}, can be used for KME \cite{sriperumbudur2011universality,simon2018kernel}. In addition to universal kernels, characteristic kernels and integrally strictly positive definite are useful for KME \cite{sriperumbudur2011universality,simon2018kernel}. 
The integrally strictly positive definite kernel was introduced in Section \ref{section_integrally_positive_definite_kernels}. In the following, we introduce the characteristic kernels. 

\begin{definition}[Characteristic Kernel \cite{fukumizu2008characteristic}]\label{definition_characteristic_kernel}
A kernel is characteristic if the mapping (\ref{equation_mapping_PDF_to_RKHS}) is injective. In other words, for a characteristic kernel $k$, we have:
\begin{align}
\mathbb{E}_{X \sim \mathbb{P}}[k(., X)] = \mathbb{E}_{Y \sim \mathbb{Q}}[k(., Y)] \iff \mathbb{P} = \mathbb{Q},
\end{align}
where $\mathbb{P}$ and $\mathbb{Q}$ are two PDFs and $X$ and $Y$ are random variables from these distributions, respectively. 
\end{definition}
Some examples for characteristic kernels are RBF and Laplacian kernels. Polynomial kernels are not characteristic kernels. 

\begin{corollary}[Convergence of Distributions to Each Other Using Characteristic Kernels \cite{simon2020metrizing}]\label{corollary_characteristic_kernel_convergence_of_distributions}
Let $\mathbb{Q}$ be the PDF for a theoretical or sample reference distribution. Following Definition \ref{definition_characteristic_kernel}, if the kernel used in measures for difference of distributions is characteristic, the measure can be used in an optimization framework to converge a PDF $\mathbb{P}$ to the reference distribution $\mathbb{Q}$ as:
\begin{align}
d_k(\mathbb{P}, \mathbb{Q}) = 0 \iff \mathbb{P} = \mathbb{Q}.
\end{align}
where $d_k(.,.)$ denotes a measure for difference of distributions such as MMD. 
\end{corollary}

Characteristic kernels have been used for dimensionality reduction in machine learning. For example, see \cite{fukumizu2004dimensionality,fukumizu2009kernel}.

So far, we have introduced three different types of embedding in RKHS. In the following, we summarize these three types available in the literature of kernels. 
\begin{remark}[Types of Embedding in Hilbert Space]
There are three types of embeddings in Hilbert space:
\begin{enumerate}[topsep=0pt,itemsep=-1ex,partopsep=1ex,parsep=1ex]
\item Embedding of points in the Hilbert space: This embedding maps $\b{x} \mapsto k(\b{x},.)$ as stated in Sections \ref{section_RKHS} and \ref{section_kernelization_techniques}. 
\item Embedding of functions in the Hilbert space: This embedding maps $f(\b{x}) \mapsto \int K(\b{x},\b{y})\, f(\b{y})\, p(\b{y})\, d\b{y}$ as stated in Section \ref{section_eigenfunctions}.
\item Embedding of distributions (PDF's) in the Hilbert space: This embedding maps $\mathbb{P} \mapsto \int k(\b{x}, .)\, d\mathbb{P}(\b{x})$ as stated in Section \ref{section_kernel_embedding_distributions}.
\end{enumerate}
Researchers are expecting that a combination of these types of embedding might appear in the future. 
\end{remark}

\subsection{Kernel Dimensionality Reduction for Sufficient Dimensionality Reduction}

Kernels can also be used directly for dimensionality reduction.
Assume $\b{X}$ is the random variable of data and $Y$ is the random variables of labels of data. The labels can be discrete finite for classification or continuous for regression. 
Sufficient Dimensionality Reduction (SDR) \cite{adragni2009sufficient} is a family of methods which find a transformation of data to a lower dimensional space, denoted by $R(\b{x})$, which does not change the conditional of labels given data:
\begin{align}
\mathbb{P}_{Y | X}(\b{y}\, |\, \b{x}) = \mathbb{P}_{Y | R(X)}(\b{y}\, |\, R(\b{x})).
\end{align}

Kernel Dimensionality Reduction (KDR) \cite{fukumizu2004dimensionality,fukumizu2009kernel,wang2010unsupervised} is a SDR method with linear projection for transformation, i.e. $R(\b{x}): \b{x} \mapsto \b{U}^\top \b{x}$ which projects data onto the column space of $\b{U}$. The goal of KDR is:
\begin{align}
\mathbb{P}_{Y | X}(\b{y}\, |\, \b{x}) = \mathbb{P}_{Y | U, X}(\b{y}\, |\, \b{U}, \b{x}).
\end{align}

\begin{definition}[Dual Space]
A dual space of a vector space $\mathcal{V}$, denoted by $\mathcal{V}^*$, is the set of all linear functionals $\phi: \mathcal{V} \rightarrow \mathcal{F}$ where $\mathcal{F}$ is the field on which vector space is defined. 
\end{definition}

\begin{theorem}[Riesz (or Riesz–Fr{\'e}chet) representation theorem \cite{garling1973short}]\label{theorem_Riesz}
Let $\mathcal{H}$ be a Hilbert space with norm $\langle.,.\rangle_\mathcal{H}$. Suppose $\phi \in \mathcal{H}^*$ (e.g., $\phi: f \rightarrow \mathbb{R}$). Then, there exists a unique $f \in \mathcal{H}$ such that for any $\b{x} \in \mathcal{H}$, we have $\phi(\b{x}) = \langle f,g \rangle$:
\begin{align}
\exists f \in \mathcal{H} : \phi \in \mathcal{H}^*, \forall \b{x} \in \mathcal{H}, \quad  \phi(\b{x}) = \langle f,\b{x} \rangle_\mathcal{H}.
\end{align}
\end{theorem}


\begin{corollary}\label{corollary_Riesz_expectation}
According to Theorem \ref{theorem_Riesz}, we have:
\begin{align}
& \mathbb{E}_X[f(\b{x})] = \langle f, \phi(\mathbb{P}) \rangle_{\mathcal{H}}, \quad \forall f \in \mathcal{H},
\end{align}
where $\phi(\mathbb{P})$ is defined by Eq. (\ref{equation_mapping_PDF_to_RKHS}).
\end{corollary}

KDR uses Theorem \ref{theorem_Riesz} and Corollary \ref{corollary_Riesz_expectation} in its formualtions. 
Note that characteristic kernels \cite{fukumizu2008characteristic}, introduced in Definition \ref{definition_characteristic_kernel}, are used in KDR.

\section{Rank and Factorization of Kernel and the Nystr{\"o}m Method}\label{section_factorization_and_Nystrom_method}

\subsection{Rank and Factorization of Kernel Matrix}

Usually, the rank of kernel is small. This is because the manifold hypothesis which states that data points often do not cover the whole space but lie on a sub-manifold. Nystr{\"o}m approximation of the kernel matrix also works well because kernels are often low-rank matrices (we will discuss it in Corollary \ref{corollary_Nystron_rank_discussion}). 
Because of low rank of kernels, they can be approximated \cite{kishore2017literature}, learned \cite{kulis2006learning,kulis2009low}, and factorized. 
Kernel factorization has also be used for the sake of clustering \cite{wang2010low}.
In the following, we introduce some of the most well-known decompositions for the kernel matrix. 

\subsubsection{Singular Value and Eigenvalue Decompositions}

The Singular Value Decomposition (SVD) of the pulled data to the feature space is:
\begin{align}\label{equation_SVD_Phi}
\mathbb{R}^{t \times n} \ni \b{\Phi}(\b{X}) = \b{U}\b{\Sigma}\b{V}^\top,
\end{align}
where $\b{U} \in \mathbb{R}^{t \times n}$ and $\b{V} \in \mathbb{R}^{n \times n}$ are orthogonal matrices and contain the left and right singular vectors, respectively, and $\b{\Sigma} \in \mathbb{R}^{n \times n}$ is a diagonal matrix with singular values.  
Note that here, we are using notations such as $\mathbb{R}^{t \times n}$ for showing the dimensionality of matrices and this notation does not imply a Euclidean space. 

As mentioned before, the pulled data are not necessarily available so Eq. (\ref{equation_SVD_Phi}) cannot necessarily be done.
The kernel, however, is available.
Using the SVD of pulled data in the formulation of kernel gives:
\begin{align}
&\b{K} \overset{(\ref{equation_kernel_inner_product_matrix})}{=} \b{\Phi}(\b{X})^\top \b{\Phi}(\b{X}) \overset{(\ref{equation_SVD_Phi})}{=} (\b{U}\b{\Sigma}\b{V}^\top)^\top (\b{U}\b{\Sigma}\b{V}^\top) \nonumber \\
&~~~~~~ = \b{V}\b{\Sigma}\underbrace{\b{U}^\top \b{U}}_{=\,\b{I}}\b{\Sigma}\b{V}^\top 
= \b{V}\b{\Sigma}\b{\Sigma}\b{V}^\top = \b{V}\b{\Sigma}^2\b{V}^\top, \nonumber \\
&\implies \b{K} \b{V} = \b{V} \b{\Sigma}^2 \underbrace{\b{V}^\top \b{V}}_{=\,\b{I}} \implies \b{K} \b{V} = \b{V} \b{\Sigma}^2. \label{equation_eigenvalue_problem_kernel_factorization_Sigma2}
\end{align}
If we take $\b{\Delta} = \textbf{diag}([\delta_1, \dots, \delta_n]^\top) := \b{\Sigma}^2$, this equation becomes:
\begin{align}\label{equation_eigenvalue_problem_kernel_factorization}
\b{K} \b{V} = \b{V} \b{\Delta},
\end{align}
which is the matrix of Eq. (\ref{equation_eigenvalue_problem_kernel_operator}), i.e. the eigenvalue problem \cite{ghojogh2019eigenvalue} for the kernel matrix with $\b{V}$ and $\b{\Delta}$ as eigenvectors and eigenvalues, respectively.
Hence, for SVD on the pulled dataset, one can apply Eigenvalue Decomposition (EVD) on the kernel where the eigenvectors of kernel are equal to right singular vectors of pulled dataset and the eigenvalues of kernel are the squared singular values of pulled dataset. 
This technique has been used in kernel PCA \cite{ghojogh2019unsupervised}. 

\subsubsection{Cholesky and QR Decompositions}

The kernel matrix can be factorized using LU decomposition; however, as the kernel matrix is symmetric positive semi-definite matrix (see Lemmas \ref{lemma_kernel_is_symmetric} and \ref{lemma_kernel_is_positive_semidefinite}), it can be decomposed using Cholesky decomposition which is much faster than LU decomposition. The Cholesky decomposition of kernel is in the form $\mathbb{R}^{n \times n} \ni \b{K} = \b{L} \b{L}^\top$ where $\b{L} \in \mathbb{R}^{n \times n}$ is a lower-triangular matrix. 
The kernel matrix can also be factorized using QR decomposition as $\mathbb{R}^{n \times n} \ni \b{K} = \b{Q} \b{R}$ where $\b{Q} \in \mathbb{R}^{n \times n}$ is an orthogonal matrix and $\b{R} \in \mathbb{R}^{n \times n}$ is an upper-triangular matrix. 
The paper \cite{bach2005predictive} has incorporated side information, such as class labels, in the Cholesky and QR decompositions of kernel matrix.

\subsection{Nystr{\"o}m Method for Approximation of Eigenfunctions}

Nystr{\"o}m method, first proposed in \cite{nystrom1930praktische}, was initially used for approximating the eigenfunctions of an operator (or of a matrix corresponding to an operator). 
The following lemma provides the Nystr{\"o}m approximation for the eigenfunctions of the kernel operator defined by Eq. (\ref{equation_K_operator_integral}).
For more discussion on this, reader can refer to \cite{baker1977numerical}, {\citep[Section 1.1]{williams2001using}}, and \cite{williams2000effect}.

\begin{lemma}[Nystr{\"o}m Approximation of Eigenfunction {\citep[Chapter 3]{baker1977numerical}}, {\citep[Section 1.1]{williams2001using}}]\label{lemma_Nystrom_approx_eigenfunctions}
Consider a training dataset $\{\b{x}_i \in \mathbb{R}^d\}_{i=1}^n$ and the eigenfunction problem (\ref{equation_eigenfunction}) where $f_k \in \mathcal{H}$ and $\lambda_k$ are the $k$-th eigenfunction and eigenvalue of kernel operator defined by Eq. (\ref{equation_K_operator_integral}) or (\ref{equation_discrete_kernel_operator}).
The eigenfunction can be approximated by Nystr{\"o}m method as:
\begin{align}
&f_k(\b{x}) \approx \frac{1}{n \lambda_k} \sum_{i=1}^n k(\b{x}_i, \b{x})\, f_k(\b{x}_i), \label{equation_relation_eigenfunction_eigenvector_x_2} 
\end{align}
where $k(\b{x}_i, \b{x})$ is the kernel (or centered kernel) corresponding to the kernel operator.  
\end{lemma}
\begin{proof}
This Lemma is somewhat similar and related to Lemma \ref{lemma_relation_eigenfunctions_eigenvectors_for_kernel}. 
Consider the kernel operator, defined by Eq. (\ref{equation_K_operator_integral}), in Eq. (\ref{equation_eigenfunction}): $K f = \lambda f$. Combining this with the discrete version of operator, Eq. (\ref{equation_discrete_kernel_operator}), gives $\lambda_k f_k(\b{x}) = (1/n) \sum_{i=1}^n k(\b{x}_i, \b{x}) f_k(\b{x}_i)$
Dividing the sides of this equation by $\lambda_k$ brings Eq. (\ref{equation_relation_eigenfunction_eigenvector_x_2}). Q.E.D.
\end{proof}

\subsection{Nystr{\"o}m Method for Kernel Completion and Approximation}\label{section_Nystrom_for_kernel}

As explained before, the kernel matrix has a low rank often. Because of its low rank, it can be approximated \cite{kishore2017literature}. This is important because in big data when $n \gg 1$, constructing the kernel matrix is both time-consuming and also intractable to store in computer; i.e., its computation will run forever and will raise a memory error finally. Hence, it is desired to compute the kernel function between a subset of data points (called landmarks) and then approximate the rest of kernel matrix using this subset of kernel matrix. Nystr{\"o}m approximation can be used for this goal. 

The Nystr{\"o}m method can be used for kernel approximation. 
It is a technique used to approximate a positive semi-definite matrix using merely a subset of its columns (or rows) {\citep[Section 1.2]{williams2001using}}, \cite{drineas2005nystrom}. 
Hence, it can be used for kernel completion in big data where computation of the entire kernel matrix is time consuming and intractable. One can compute some of the important important columns or rows of a kernel matrix, called landmarks, and approximate the rest of columns or rows by Nystr{\"o}m approximation.

Consider a positive semi-definite matrix $\mathbb{R}^{n \times n} \ni \b{K} \succeq 0$ whose parts are:
\begin{align}\label{equation_Nystrom_partions}
\mathbb{R}^{n \times n} \ni \b{K} = 
\left[
\begin{array}{c|c}
\b{A} & \b{B} \\
\hline
\b{B}^\top & \b{C}
\end{array}
\right],
\end{align}
where $\b{A} \in \mathbb{R}^{m \times m}$, $\b{B} \in \mathbb{R}^{m \times (n-m)}$, and $\b{C} \in \mathbb{R}^{(n-m) \times (n-m)}$ in which $m \ll n$. 
This positive semi-definite matrix can be a kernel (or Gram) matrix. 

The Nystr{\"o}m approximation says if we have the small parts of this matrix, i.e. $\b{A}$ and $\b{B}$, we can approximate $\b{C}$ and thus the whole matrix $\b{K}$. The intuition is as follows. Assume $m=2$ (containing two points, a and b) and $n=5$ (containing three other points, c, d, and e). If we know the similarity (or distance) of points a and b from one another, resulting in matrix $\b{A}$, as well as the similarity (or distance) of points c, d, and e from a and b, resulting in matrix $\b{B}$, we cannot have much freedom on the location of c, d, and e, which is the matrix $\b{C}$. This is because of the positive semi-definiteness of the matrix $\b{K}$. 
The points selected in submatrix $\b{A}$ are named \textit{landmarks}. Note that the landmarks can be selected randomly from the columns/rows of matrix $\b{K}$ and, without loss of generality, they can be put together to form a submatrix at the top-left corner of matrix. 
For Nystr{\"o}m approximation, some methods have been proposed for sampling more important columns/rows of matrix more wisely rather than randomly. We will mention some of these sampling methods in Section \ref{section_improvements_over_Nystrom}.

As the matrix $\b{K}$ is positive semi-definite, by definition, it can be written as $\b{K} = \b{O}^\top \b{O}$. If we take $\b{O} = [\b{R}, \b{S}]$ where $\b{R}$ are the selected columns (landmarks) of $\b{O}$ and $\b{S}$ are the other columns of $\b{O}$. We have:
\begin{align}
\b{K} &= \b{O}^\top \b{O} = 
\begin{bmatrix}
\b{R}^\top \\
\b{S}^\top
\end{bmatrix}
[\b{R}, \b{S}] \label{equation_Nystrom_kernel_OtransposeO} \\
&= 
\begin{bmatrix}
\b{R}^\top \b{R} & \b{R}^\top \b{S} \\
\b{S}^\top \b{R} & \b{S}^\top \b{S}
\end{bmatrix}
\overset{(\ref{equation_Nystrom_partions})}{=} 
\begin{bmatrix}
\b{A} & \b{B} \\
\b{B}^\top & \b{C}
\end{bmatrix}.
\end{align}
Hence, we have $\b{A} = \b{R}^\top \b{R}$. The eigenvalue decomposition \cite{ghojogh2019eigenvalue} of $\b{A}$ gives:
\begin{align}
&\b{A} = \b{U} \b{\Sigma} \b{U}^\top \label{equation_Nystrom_A_eig_decomposition} \\
&\implies \b{R}^\top \b{R} = \b{U} \b{\Sigma} \b{U}^\top \implies \b{R} = \b{\Sigma}^{(1/2)} \b{U}^\top. \label{equation_Nystrom_R}
\end{align}
Moreover, we have $\b{B} = \b{R}^\top \b{S}$ so we have:
\begin{align}
&\b{B} = (\b{\Sigma}^{(1/2)} \b{U}^\top)^\top \b{S} = \b{U} \b{\Sigma}^{(1/2)} \b{S} \nonumber \\
&\overset{(a)}{\implies} \b{U}^\top \b{B} = \b{\Sigma}^{(1/2)} \b{S} \implies \b{S} = \b{\Sigma}^{(-1/2)} \b{U}^\top \b{B}, \label{equation_Nystrom_S}
\end{align}
where $(a)$ is because $\b{U}$ is orthogonal (in the eigenvalue decomposition). 
Finally, we have:
\begin{align}
\b{C} &= \b{S}^\top \b{S} = \b{B}^\top \b{U} \b{\Sigma}^{(-1/2)} \b{\Sigma}^{(-1/2)} \b{U}^\top \b{B} \nonumber \\
&= \b{B}^\top \b{U} \b{\Sigma}^{-1} \b{U}^\top \b{B} \overset{(\ref{equation_Nystrom_A_eig_decomposition})}{=} \b{B}^\top \b{A}^{-1} \b{B}. \label{equation_Nystrom_C}
\end{align}
Therefore, Eq. (\ref{equation_Nystrom_partions}) becomes:
\begin{align}\label{equation_Nystrom_partions_withDetails}
\b{K} \approx 
\left[
\begin{array}{c|c}
\b{A} & \b{B} \\
\hline
\b{B}^\top & \b{B}^\top \b{A}^{-1} \b{B}
\end{array}
\right].
\end{align}

\begin{lemma}[Impact of Size of sub-matrix $\b{A}$ on Nystr{\"o}m approximation]
By increasing $m$, the approximation of Eq. (\ref{equation_Nystrom_partions_withDetails}) becomes more accurate. 
If rank of $\b{K}$ is at most $m$, this approximation is exact. 
\end{lemma}
\begin{proof}
In Eq. (\ref{equation_Nystrom_C}), we have the inverse of $\b{A}$. In order to have this inverse, the matrix $\b{A}$ must not be singular. For having a full-rank $\b{A} \in \mathbb{R}^{m \times m}$, the rank of $\b{A}$ should be $m$. This results in $m$ to be an upper bound on the rank of $\b{K}$ and a lower bound on the number of landmarks. In practice, it is recommended to use more number of landmarks for more accurate approximation but there is a trade-off with the speed. 
\end{proof}

\begin{corollary}\label{corollary_Nystron_rank_discussion}
As we usually have $m \ll n$, the Nystr{\"o}m approximation works well especially for the low-rank matrices \cite{kishore2017literature} because we will need a small $\b{A}$ (so small number of landmarks) for approximation. Usually, because of the manifold hypothesis, data fall on a submanifold; hence, usually, the kernel (similarity) matrix  or the distance matrix has a low rank. Therefore, the Nystr{\"o}m approximation works well for many kernel-based or distance-based manifold learning methods. 
\end{corollary}

\subsection{Use of Nystr{\"o}m Approximation for Landmark Spectral Embedding}

The spectral dimensionality reduction methods \cite{saul2006spectral} are based on geometry of data and their solutions often follow an eigenvalue problem \cite{ghojogh2019eigenvalue}. 
Therefore, they cannot handle big data where $n \gg 1$. To tackle this issue, there exist some landmark methods which approximate the embedding of all points using the embedding of some landmarks. 
Big data, i.e. $n \gg 1$, results in large kernel matrices. Selecting some most informative columns or rows of the kernel matrix, called landmarks, can reduce computations. This technique is named the Nystr{\"o}m approximation which is used for kernel approximation and completion. 

Nystr{\"o}m approximation, introduced below, can be used to make the spectral methods such as locally linear embedding \cite{ghojogh2020locally} and Multidimensional Scaling (MDS) \cite{ghojogh2020multidimensional} scalable and suitable for big data embedding. 
It is shown in \cite{platt2005fastmap} that all the landmark MDS methods are Nystr{\"o}m approximations.
For more details on usage of Nystr{\"o}m approximation in spectral embedding, refer to \cite{ghojogh2020locally,ghojogh2020multidimensional}.

\subsection{Other Improvements over Nystr{\"o}m Approximation of Kernels}\label{section_improvements_over_Nystrom}

The Nystr{\"o}m method has been improved for kernel approximation in a line of research. For example, it has been used for clustering \cite{fowlkes2004spectral} and regularization \cite{rudi2015less}. 
Greedy Nystr{\"o}m \cite{farahat2011novel,farahat2015greedy} and large scale Nystr{\"o}m \cite{li2010making} are other examples. 
There is a trade-off between the approximation accuracy and computational efficiency and they are balanced in \cite{lim2015double,lim2018multi}. 
The error analysis of Nystr{\"o}m method can be found in \cite{zhang2008improved,zhang2010clustered}.
It is better to sample the landmarks wisely rather than randomly. This field of research is named ``column subset selection" or ``landmark selection" for Nystron approximation. Some of these methods are \cite{kumar2009sampling,kumar2012sampling} and landmark selection on a sparse manifold \cite{silva2006selecting}.

\section{Conclusion}\label{section_conclusion}

This paper was a tutorial and survey paper on kernels and kernel methods. We covered various topics including Mercer kernels, Mercer's theorem, RKHS, eigenfunctions, Nystr{\"o}m methods, kernelization techniques, and use of kernels in machine learning. This paper can be useful for different fields of science such as machine learning, functional analysis, and quantum mechanics. 

\section*{Acknowledgement}

Some parts of this tutorial, particularly some parts of the RKHS, are covered by Prof. Larry Wasserman's statistical machine learning course at the Department of Statistics and Data Science, Carnegie Mellon University (watch his course on YouTube). 
The video of RKHS in Statistical Machine Learning course by Prof. Ulrike von Luxburg at the University of T{\"u}bingen is also great (watch on YouTube, T{\"u}bingen Machine Learning channel).
Also, some parts such as the Nystr{\"o}m approximation and MVU are covered by Prof. Ali Ghodsi's course, at University of Waterloo, available on YouTube. 
There are other useful videos on this field which can be found on YouTube. 


\bibliography{References}

\begin{thebibliography}{156}
\providecommand{\natexlab}[1]{#1}
\providecommand{\url}[1]{\texttt{#1}}
\expandafter\ifx\csname urlstyle\endcsname\relax
  \providecommand{\doi}[1]{doi: #1}\else
  \providecommand{\doi}{doi: \begingroup \urlstyle{rm}\Url}\fi

\bibitem[Abu-Khzam et~al.(2004)Abu-Khzam, Collins, Fellows, Langston, Suters,
  and Symons]{abu2004kernelization}
Abu-Khzam, Faisal~N, Collins, Rebecca~L, Fellows, Micheal~R, Langston,
  Micheal~A, Suters, W~Henry, and Symons, Christopher~T.
\newblock Kernelization algorithms for the vertex cover problem: Theory and
  experiments.
\newblock Technical report, University of Tennessee, 2004.

\bibitem[Adragni \& Cook(2009)Adragni and Cook]{adragni2009sufficient}
Adragni, Kofi~P and Cook, R~Dennis.
\newblock Sufficient dimension reduction and prediction in regression.
\newblock \emph{Philosophical Transactions of the Royal Society A:
  Mathematical, Physical and Engineering Sciences}, 367\penalty0
  (1906):\penalty0 4385--4405, 2009.

\bibitem[Ah-Pine(2010)]{ah2010normalized}
Ah-Pine, Julien.
\newblock Normalized kernels as similarity indices.
\newblock In \emph{Pacific-Asia Conference on Knowledge Discovery and Data
  Mining}, pp.\  362--373. Springer, 2010.

\bibitem[Aizerman et~al.(1964)Aizerman, Braverman, and
  Rozonoer]{aizerman1964theoretical}
Aizerman, Mark~A, Braverman, E.~M., and Rozonoer, L.~I.
\newblock Theoretical foundations of the potential function method in pattern
  recognition learning.
\newblock \emph{Automation and remote control}, 25:\penalty0 821--837, 1964.

\bibitem[Akbari et~al.(2006)Akbari, Ghareghani, Khosrovshahi, and
  Maimani]{akbari2006kernels}
Akbari, Saieed, Ghareghani, Narges, Khosrovshahi, Gholamreza~B, and Maimani,
  Hamidreza.
\newblock The kernels of the incidence matrices of graphs revisited.
\newblock \emph{Linear algebra and its applications}, 414\penalty0
  (2-3):\penalty0 617--625, 2006.

\bibitem[Alperin(1993)]{alperin1993local}
Alperin, Jonathan~L.
\newblock \emph{Local representation theory: Modular representations as an
  introduction to the local representation theory of finite groups}.
\newblock Cambridge University Press, 1993.

\bibitem[Anderson \& Dahlin(2014)Anderson and Dahlin]{anderson2014operating}
Anderson, Thomas and Dahlin, Michael.
\newblock \emph{Operating Systems: Principles and Practice}.
\newblock Recursive Books, 2014.

\bibitem[Arjovsky et~al.(2017)Arjovsky, Chintala, and
  Bottou]{arjovsky2017wasserstein}
Arjovsky, Martin, Chintala, Soumith, and Bottou, L{\'e}on.
\newblock Wasserstein generative adversarial networks.
\newblock In \emph{International conference on machine learning}, pp.\
  214--223. PMLR, 2017.

\bibitem[Aronszajn(1950)]{aronszajn1950theory}
Aronszajn, Nachman.
\newblock Theory of reproducing kernels.
\newblock \emph{Transactions of the American mathematical society}, 68\penalty0
  (3):\penalty0 337--404, 1950.

\bibitem[Arzel{\`a}(1895)]{arzela1895sulle}
Arzel{\`a}, Cesare.
\newblock \emph{Sulle Funzioni Di Linee}.
\newblock Tipografia Gamberini e Parmeggiani, 1895.

\bibitem[Bach \& Jordan(2005)Bach and Jordan]{bach2005predictive}
Bach, Francis~R and Jordan, Michael~I.
\newblock Predictive low-rank decomposition for kernel methods.
\newblock In \emph{Proceedings of the 22nd international conference on machine
  learning}, pp.\  33--40, 2005.

\bibitem[Baker(1978)]{baker1977numerical}
Baker, Christopher~TH.
\newblock \emph{The numerical treatment of integral equations}.
\newblock Clarendon press, 1978.

\bibitem[Barshan et~al.(2011)Barshan, Ghodsi, Azimifar, and
  Jahromi]{barshan2011supervised}
Barshan, Elnaz, Ghodsi, Ali, Azimifar, Zohreh, and Jahromi, Mansoor~Zolghadri.
\newblock Supervised principal component analysis: Visualization,
  classification and regression on subspaces and submanifolds.
\newblock \emph{Pattern Recognition}, 44\penalty0 (7):\penalty0 1357--1371,
  2011.

\bibitem[Beauzamy(1982)]{beauzamy1982introduction}
Beauzamy, Bernard.
\newblock \emph{Introduction to {Banach} spaces and their geometry}.
\newblock North-Holland, 1982.

\bibitem[Bell(2016)]{bell2016trace}
Bell, Jordan.
\newblock Trace class operators and {Hilbert}-{Schmidt} operators.
\newblock \emph{Department of Mathematics, University of Toronto, Technical
  Report}, 2016.

\bibitem[Bengio et~al.(2003{\natexlab{a}})Bengio, Paiement, Vincent, Delalleau,
  Roux, and Ouimet]{bengio2003out}
Bengio, Yoshua, Paiement, Jean-fran{\c{c}}cois, Vincent, Pascal, Delalleau,
  Olivier, Roux, Nicolas, and Ouimet, Marie.
\newblock Out-of-sample extensions for {LLE}, {Isomap}, {MDS}, eigenmaps, and
  spectral clustering.
\newblock \emph{Advances in neural information processing systems},
  16:\penalty0 177--184, 2003{\natexlab{a}}.

\bibitem[Bengio et~al.(2003{\natexlab{b}})Bengio, Vincent, Paiement, Delalleau,
  Ouimet, and LeRoux]{bengio2003learning}
Bengio, Yoshua, Vincent, Pascal, Paiement, Jean-Fran{\c{c}}ois, Delalleau, O,
  Ouimet, M, and LeRoux, N.
\newblock Learning eigenfunctions of similarity: linking spectral clustering
  and kernel {PCA}.
\newblock Technical report, Departement d’Informatique et Recherche
  Operationnelle, 2003{\natexlab{b}}.

\bibitem[Bengio et~al.(2003{\natexlab{c}})Bengio, Vincent, Paiement, Delalleau,
  Ouimet, and Le~Roux]{bengio2003spectral}
Bengio, Yoshua, Vincent, Pascal, Paiement, Jean-Fran{\c{c}}ois, Delalleau,
  Olivier, Ouimet, Marie, and Le~Roux, Nicolas.
\newblock Spectral clustering and kernel {PCA} are learning eigenfunctions.
\newblock Technical report, Departement d’Informatique et Recherche
  Operationnelle, Technical Report 1239, 2003{\natexlab{c}}.

\bibitem[Bengio et~al.(2004)Bengio, Delalleau, Roux, Paiement, Vincent, and
  Ouimet]{bengio2004learning}
Bengio, Yoshua, Delalleau, Olivier, Roux, Nicolas~Le, Paiement,
  Jean-Fran{\c{c}}ois, Vincent, Pascal, and Ouimet, Marie.
\newblock Learning eigenfunctions links spectral embedding and kernel {PCA}.
\newblock \emph{Neural computation}, 16\penalty0 (10):\penalty0 2197--2219,
  2004.

\bibitem[Bengio et~al.(2006)Bengio, Delalleau, Le~Roux, Paiement, Vincent, and
  Ouimet]{bengio2006spectral}
Bengio, Yoshua, Delalleau, Olivier, Le~Roux, Nicolas, Paiement,
  Jean-Fran{\c{c}}ois, Vincent, Pascal, and Ouimet, Marie.
\newblock Spectral dimensionality reduction.
\newblock In \emph{Feature Extraction}, pp.\  519--550. Springer, 2006.

\bibitem[Berg et~al.(1984)Berg, Christensen, and Ressel]{berg1984harmonic}
Berg, Christian, Christensen, Jens Peter~Reus, and Ressel, Paul.
\newblock \emph{Harmonic analysis on semigroups: theory of positive definite
  and related functions}, volume 100.
\newblock Springer, 1984.

\bibitem[Bergman(2011)]{bergman2011universal}
Bergman, Clifford.
\newblock \emph{Universal algebra: Fundamentals and selected topics}.
\newblock CRC Press, 2011.

\bibitem[Berlinet \& Thomas-Agnan(2011)Berlinet and
  Thomas-Agnan]{berlinet2011reproducing}
Berlinet, Alain and Thomas-Agnan, Christine.
\newblock \emph{Reproducing kernel {Hilbert} spaces in probability and
  statistics}.
\newblock Springer Science \& Business Media, 2011.

\bibitem[Bhatia(2009)]{bhatia2009positive}
Bhatia, Rajendra.
\newblock \emph{Positive definite matrices}.
\newblock Princeton University Press, 2009.

\bibitem[Borgwardt et~al.(2006)Borgwardt, Gretton, Rasch, Kriegel,
  Sch{\"o}lkopf, and Smola]{borgwardt2006integrating}
Borgwardt, Karsten~M, Gretton, Arthur, Rasch, Malte~J, Kriegel, Hans-Peter,
  Sch{\"o}lkopf, Bernhard, and Smola, Alex~J.
\newblock Integrating structured biological data by kernel maximum mean
  discrepancy.
\newblock \emph{Bioinformatics}, 22\penalty0 (14):\penalty0 e49--e57, 2006.

\bibitem[Boser et~al.(1992)Boser, Guyon, and Vapnik]{boser1992training}
Boser, Bernhard~E, Guyon, Isabelle~M, and Vapnik, Vladimir~N.
\newblock A training algorithm for optimal margin classifiers.
\newblock In \emph{Proceedings of the fifth annual workshop on Computational
  learning theory}, pp.\  144--152, 1992.

\bibitem[Bourbaki(1950)]{bourbaki1950certains}
Bourbaki, Nicolas.
\newblock Sur certains espaces vectoriels topologiques.
\newblock In \emph{Annales de l'institut Fourier}, volume~2, pp.\  5--16, 1950.

\bibitem[Brunet et~al.(2011)Brunet, Vrscay, and Wang]{brunet2011mathematical}
Brunet, Dominique, Vrscay, Edward~R, and Wang, Zhou.
\newblock On the mathematical properties of the structural similarity index.
\newblock \emph{IEEE Transactions on Image Processing}, 21\penalty0
  (4):\penalty0 1488--1499, 2011.

\bibitem[Burges(1998)]{burges1998tutorial}
Burges, Christopher~JC.
\newblock A tutorial on support vector machines for pattern recognition.
\newblock \emph{Data mining and knowledge discovery}, 2\penalty0 (2):\penalty0
  121--167, 1998.

\bibitem[Camps-Valls(2006)]{camps2006kernel}
Camps-Valls, Gustavo.
\newblock \emph{Kernel methods in bioengineering, signal and image processing}.
\newblock Igi Global, 2006.

\bibitem[Conway(2007)]{conway2007course}
Conway, John~B.
\newblock \emph{A course in functional analysis}.
\newblock Springer, 2 edition, 2007.

\bibitem[Cox \& Cox(2008)Cox and Cox]{cox2008multidimensional}
Cox, Michael~AA and Cox, Trevor~F.
\newblock Multidimensional scaling.
\newblock In \emph{Handbook of data visualization}, pp.\  315--347. Springer,
  2008.

\bibitem[De~Branges(1959)]{de1959stone}
De~Branges, Louis.
\newblock The {Stone}-{Weierstrass} theorem.
\newblock \emph{Proceedings of the American Mathematical Society}, 10\penalty0
  (5):\penalty0 822--824, 1959.

\bibitem[Drineas et~al.(2005)Drineas, Mahoney, and
  Cristianini]{drineas2005nystrom}
Drineas, Petros, Mahoney, Michael~W, and Cristianini, Nello.
\newblock On the {N}ystr{\"o}m method for approximating a {Gram} matrix for
  improved kernel-based learning.
\newblock \emph{journal of machine learning research}, 6\penalty0 (12), 2005.

\bibitem[Farahat et~al.(2011)Farahat, Ghodsi, and Kamel]{farahat2011novel}
Farahat, Ahmed, Ghodsi, Ali, and Kamel, Mohamed.
\newblock A novel greedy algorithm for {N}ystr{\"o}m approximation.
\newblock In \emph{Proceedings of the Fourteenth International Conference on
  Artificial Intelligence and Statistics}, pp.\  269--277. JMLR Workshop and
  Conference Proceedings, 2011.

\bibitem[Farahat et~al.(2015)Farahat, Elgohary, Ghodsi, and
  Kamel]{farahat2015greedy}
Farahat, Ahmed~K, Elgohary, Ahmed, Ghodsi, Ali, and Kamel, Mohamed~S.
\newblock Greedy column subset selection for large-scale data sets.
\newblock \emph{Knowledge and Information Systems}, 45\penalty0 (1):\penalty0
  1--34, 2015.

\bibitem[Fausett(1994)]{fausett1994fundamentals}
Fausett, Laurene~V.
\newblock \emph{Fundamentals of neural networks: architectures, algorithms and
  applications}.
\newblock Prentice-Hall, Inc., 1994.

\bibitem[Fomin et~al.(2019)Fomin, Lokshtanov, Saurabh, and
  Zehavi]{fomin2019kernelization}
Fomin, Fedor~V, Lokshtanov, Daniel, Saurabh, Saket, and Zehavi, Meirav.
\newblock \emph{Kernelization: theory of parameterized preprocessing}.
\newblock Cambridge University Press, 2019.

\bibitem[Fowlkes et~al.(2004)Fowlkes, Belongie, Chung, and
  Malik]{fowlkes2004spectral}
Fowlkes, Charless, Belongie, Serge, Chung, Fan, and Malik, Jitendra.
\newblock Spectral grouping using the {N}ystr{\"o}m method.
\newblock \emph{IEEE transactions on pattern analysis and machine
  intelligence}, 26\penalty0 (2):\penalty0 214--225, 2004.

\bibitem[Fukumizu et~al.(2004)Fukumizu, Bach, and
  Jordan]{fukumizu2004dimensionality}
Fukumizu, Kenji, Bach, Francis~R, and Jordan, Michael~I.
\newblock Dimensionality reduction for supervised learning with reproducing
  kernel {Hilbert} spaces.
\newblock \emph{Journal of Machine Learning Research}, 5\penalty0
  (Jan):\penalty0 73--99, 2004.

\bibitem[Fukumizu et~al.(2008)Fukumizu, Sriperumbudur, Gretton, and
  Sch{\"o}lkopf]{fukumizu2008characteristic}
Fukumizu, Kenji, Sriperumbudur, Bharath~K, Gretton, Arthur, and Sch{\"o}lkopf,
  Bernhard.
\newblock Characteristic kernels on groups and semigroups.
\newblock In \emph{Advances in neural information processing systems}, pp.\
  473--480, 2008.

\bibitem[Fukumizu et~al.(2009)Fukumizu, Bach, Jordan,
  et~al.]{fukumizu2009kernel}
Fukumizu, Kenji, Bach, Francis~R, Jordan, Michael~I, et~al.
\newblock Kernel dimension reduction in regression.
\newblock \emph{The Annals of Statistics}, 37\penalty0 (4):\penalty0
  1871--1905, 2009.

\bibitem[Garling(1973)]{garling1973short}
Garling, DJH.
\newblock A ‘short’ proof of the {Riesz} representation theorem.
\newblock In \emph{Mathematical Proceedings of the Cambridge Philosophical
  Society}, volume~73, pp.\  459--460. Cambridge University Press, 1973.

\bibitem[Genton(2001)]{genton2001classes}
Genton, Marc~G.
\newblock Classes of kernels for machine learning: a statistics perspective.
\newblock \emph{Journal of machine learning research}, 2\penalty0
  (Dec):\penalty0 299--312, 2001.

\bibitem[Ghojogh \& Crowley(2019)Ghojogh and Crowley]{ghojogh2019unsupervised}
Ghojogh, Benyamin and Crowley, Mark.
\newblock Unsupervised and supervised principal component analysis: Tutorial.
\newblock \emph{arXiv preprint arXiv:1906.03148}, 2019.

\bibitem[Ghojogh et~al.(2019{\natexlab{a}})Ghojogh, Karray, and
  Crowley]{ghojogh2019eigenvalue}
Ghojogh, Benyamin, Karray, Fakhri, and Crowley, Mark.
\newblock Eigenvalue and generalized eigenvalue problems: Tutorial.
\newblock \emph{arXiv preprint arXiv:1903.11240}, 2019{\natexlab{a}}.

\bibitem[Ghojogh et~al.(2019{\natexlab{b}})Ghojogh, Karray, and
  Crowley]{ghojogh2019fisher}
Ghojogh, Benyamin, Karray, Fakhri, and Crowley, Mark.
\newblock Fisher and kernel {Fisher} discriminant analysis: Tutorial.
\newblock \emph{arXiv preprint arXiv:1906.09436}, 2019{\natexlab{b}}.

\bibitem[Ghojogh et~al.(2019{\natexlab{c}})Ghojogh, Karray, and
  Crowley]{ghojogh2019image}
Ghojogh, Benyamin, Karray, Fakhri, and Crowley, Mark.
\newblock Image structure subspace learning using structural similarity index.
\newblock In \emph{International Conference on Image Analysis and Recognition},
  pp.\  33--44. Springer, 2019{\natexlab{c}}.

\bibitem[Ghojogh et~al.(2019{\natexlab{d}})Ghojogh, Samad, Mashhadi, Kapoor,
  Ali, Karray, and Crowley]{ghojogh2019feature}
Ghojogh, Benyamin, Samad, Maria~N, Mashhadi, Sayema~Asif, Kapoor, Tania, Ali,
  Wahab, Karray, Fakhri, and Crowley, Mark.
\newblock Feature selection and feature extraction in pattern analysis: A
  literature review.
\newblock \emph{arXiv preprint arXiv:1905.02845}, 2019{\natexlab{d}}.

\bibitem[Ghojogh et~al.(2020{\natexlab{a}})Ghojogh, Ghodsi, Karray, and
  Crowley]{ghojogh2020locally}
Ghojogh, Benyamin, Ghodsi, Ali, Karray, Fakhri, and Crowley, Mark.
\newblock Locally linear embedding and its variants: Tutorial and survey.
\newblock \emph{arXiv preprint arXiv:2011.10925}, 2020{\natexlab{a}}.

\bibitem[Ghojogh et~al.(2020{\natexlab{b}})Ghojogh, Ghodsi, Karray, and
  Crowley]{ghojogh2020multidimensional}
Ghojogh, Benyamin, Ghodsi, Ali, Karray, Fakhri, and Crowley, Mark.
\newblock Multidimensional scaling, {Sammon} mapping, and {Isomap}: Tutorial
  and survey.
\newblock \emph{arXiv preprint arXiv:2009.08136}, 2020{\natexlab{b}}.

\bibitem[Ghojogh et~al.(2020{\natexlab{c}})Ghojogh, Karray, and
  Crowley]{ghojogh2020generalized}
Ghojogh, Benyamin, Karray, Fakhri, and Crowley, Mark.
\newblock Generalized subspace learning by {Roweis} discriminant analysis.
\newblock In \emph{International Conference on Image Analysis and Recognition},
  pp.\  328--342. Springer, 2020{\natexlab{c}}.

\bibitem[Ghojogh et~al.(2020{\natexlab{d}})Ghojogh, Karray, and
  Crowley]{ghojogh2020theoretical}
Ghojogh, Benyamin, Karray, Fakhri, and Crowley, Mark.
\newblock Theoretical insights into the use of structural similarity index in
  generative models and inferential autoencoders.
\newblock In \emph{International Conference on Image Analysis and Recognition},
  pp.\  112--117. Springer, 2020{\natexlab{d}}.

\bibitem[Gonzalez \& Woods(2002)Gonzalez and Woods]{gonzalez2002digital}
Gonzalez, Rafael~C and Woods, Richard~E.
\newblock \emph{Digital image processing}.
\newblock Prentice hall Upper Saddle River, NJ, 2002.

\bibitem[Goodfellow et~al.(2016)Goodfellow, Bengio, Courville, and
  Bengio]{goodfellow2016deep}
Goodfellow, Ian, Bengio, Yoshua, Courville, Aaron, and Bengio, Yoshua.
\newblock \emph{Deep learning}, volume~1.
\newblock MIT press Cambridge, 2016.

\bibitem[Gretton \& Gy{\"o}rfi(2010)Gretton and
  Gy{\"o}rfi]{gretton2010consistent}
Gretton, Arthur and Gy{\"o}rfi, L{\'a}szl{\'o}.
\newblock Consistent nonparametric tests of independence.
\newblock \emph{The Journal of Machine Learning Research}, 11:\penalty0
  1391--1423, 2010.

\bibitem[Gretton et~al.(2005)Gretton, Bousquet, Smola, and
  Sch{\"o}lkopf]{gretton2005measuring}
Gretton, Arthur, Bousquet, Olivier, Smola, Alex, and Sch{\"o}lkopf, Bernhard.
\newblock Measuring statistical dependence with {Hilbert}-{Schmidt} norms.
\newblock In \emph{International conference on algorithmic learning theory},
  pp.\  63--77. Springer, 2005.

\bibitem[Gretton et~al.(2006)Gretton, Borgwardt, Rasch, Sch{\"o}lkopf, and
  Smola]{gretton2006kernel}
Gretton, Arthur, Borgwardt, Karsten, Rasch, Malte, Sch{\"o}lkopf, Bernhard, and
  Smola, Alex.
\newblock A kernel method for the two-sample-problem.
\newblock \emph{Advances in neural information processing systems},
  19:\penalty0 513--520, 2006.

\bibitem[Gretton et~al.(2012)Gretton, Borgwardt, Rasch, Sch{\"o}lkopf, and
  Smola]{gretton2012kernel}
Gretton, Arthur, Borgwardt, Karsten~M, Rasch, Malte~J, Sch{\"o}lkopf, Bernhard,
  and Smola, Alexander.
\newblock A kernel two-sample test.
\newblock \emph{The Journal of Machine Learning Research}, 13\penalty0
  (1):\penalty0 723--773, 2012.

\bibitem[Gubner(2006)]{gubner2006probability}
Gubner, John~A.
\newblock \emph{Probability and random processes for electrical and computer
  engineers}.
\newblock Cambridge University Press, 2006.

\bibitem[Guilbart(1978)]{guilbart1978etude}
Guilbart, Christian.
\newblock \emph{Etude des produits scalaires sur l'espace des mesures:
  estimation par projections}.
\newblock PhD thesis, Universit{\'e} des Sciences et Techniques de Lille, 1978.

\bibitem[Ham et~al.(2004)Ham, Lee, Mika, and Sch{\"o}lkopf]{ham2004kernel}
Ham, Jihun, Lee, Daniel~D, Mika, Sebastian, and Sch{\"o}lkopf, Bernhard.
\newblock A kernel view of the dimensionality reduction of manifolds.
\newblock In \emph{Proceedings of the twenty-first international conference on
  Machine learning}, pp.\ ~47, 2004.

\bibitem[Hastie et~al.(2009)Hastie, Tibshirani, and
  Friedman]{hastie2009elements}
Hastie, Trevor, Tibshirani, Robert, and Friedman, Jerome.
\newblock \emph{The elements of statistical learning: data mining, inference,
  and prediction}.
\newblock Springer Science \& Business Media, 2009.

\bibitem[Hawkins(1975)]{hawkins1975cauchy}
Hawkins, Thomas.
\newblock Cauchy and the spectral theory of matrices.
\newblock \emph{Historia mathematica}, 2\penalty0 (1):\penalty0 1--29, 1975.

\bibitem[Hein \& Bousquet(2004)Hein and Bousquet]{hein2004kernels}
Hein, Matthias and Bousquet, Olivier.
\newblock Kernels, associated structures and generalizations.
\newblock \emph{Max-Planck-Institut fuer biologische Kybernetik, Technical
  Report}, 2004.

\bibitem[Hilbert(1904)]{hilbert1904grundzuge}
Hilbert, David.
\newblock Grundz{\"u}ge einer allgemeinen theorie der linearen
  integralrechnungen {I}.
\newblock \emph{Nachrichten von der Gesellschaft der Wissenschaften zu
  G{\"o}ttingen, Mathematisch-Physikalische Klasse}, pp.\  49--91, 1904.

\bibitem[Hinton \& Salakhutdinov(2006)Hinton and
  Salakhutdinov]{hinton2006reducing}
Hinton, Geoffrey~E and Salakhutdinov, Ruslan~R.
\newblock Reducing the dimensionality of data with neural networks.
\newblock \emph{Science}, 313\penalty0 (5786):\penalty0 504--507, 2006.

\bibitem[Hofmann et~al.(2006)Hofmann, Sch{\"o}lkopf, and
  Smola]{hofmann2006review}
Hofmann, Thomas, Sch{\"o}lkopf, Bernhard, and Smola, Alexander~J.
\newblock A review of kernel methods in machine learning.
\newblock \emph{Max-Planck-Institute Technical Report}, 156, 2006.

\bibitem[Hofmann et~al.(2008)Hofmann, Sch{\"o}lkopf, and
  Smola]{hofmann2008kernel}
Hofmann, Thomas, Sch{\"o}lkopf, Bernhard, and Smola, Alexander~J.
\newblock Kernel methods in machine learning.
\newblock \emph{The annals of statistics}, pp.\  1171--1220, 2008.

\bibitem[Icking \& Klein(1995)Icking and Klein]{icking1995searching}
Icking, Christian and Klein, Rolf.
\newblock Searching for the kernel of a polygon—a competitive strategy.
\newblock In \emph{Proceedings of the eleventh annual symposium on
  Computational geometry}, pp.\  258--266, 1995.

\bibitem[Karimi(2017)]{karimi2017summary}
Karimi, Amir-Hossein.
\newblock A summary of the kernel matrix, and how to learn it effectively using
  semidefinite programming.
\newblock \emph{arXiv preprint arXiv:1709.06557}, 2017.

\bibitem[Kimeldorf \& Wahba(1971)Kimeldorf and Wahba]{kimeldorf1971some}
Kimeldorf, George and Wahba, Grace.
\newblock Some results on {Tchebycheffian} spline functions.
\newblock \emph{Journal of mathematical analysis and applications}, 33\penalty0
  (1):\penalty0 82--95, 1971.

\bibitem[Kishore~Kumar \& Schneider(2017)Kishore~Kumar and
  Schneider]{kishore2017literature}
Kishore~Kumar, N and Schneider, Jan.
\newblock Literature survey on low rank approximation of matrices.
\newblock \emph{Linear and Multilinear Algebra}, 65\penalty0 (11):\penalty0
  2212--2244, 2017.

\bibitem[Kulis et~al.(2006)Kulis, Sustik, and Dhillon]{kulis2006learning}
Kulis, Brian, Sustik, M{\'a}ty{\'a}s, and Dhillon, Inderjit.
\newblock Learning low-rank kernel matrices.
\newblock In \emph{Proceedings of the 23rd international conference on Machine
  learning}, pp.\  505--512, 2006.

\bibitem[Kulis et~al.(2009)Kulis, Sustik, and Dhillon]{kulis2009low}
Kulis, Brian, Sustik, M{\'a}ty{\'a}s~A, and Dhillon, Inderjit~S.
\newblock Low-rank kernel learning with {Bregman} matrix divergences.
\newblock \emph{Journal of Machine Learning Research}, 10\penalty0 (2), 2009.

\bibitem[Kullback \& Leibler(1951)Kullback and
  Leibler]{kullback1951information}
Kullback, Solomon and Leibler, Richard~A.
\newblock On information and sufficiency.
\newblock \emph{The annals of mathematical statistics}, 22\penalty0
  (1):\penalty0 79--86, 1951.

\bibitem[Kumar et~al.(2009)Kumar, Mohri, and Talwalkar]{kumar2009sampling}
Kumar, Sanjiv, Mohri, Mehryar, and Talwalkar, Ameet.
\newblock Sampling techniques for the {N}ystr{\"o}m method.
\newblock In \emph{Artificial Intelligence and Statistics}, pp.\  304--311.
  PMLR, 2009.

\bibitem[Kumar et~al.(2012)Kumar, Mohri, and Talwalkar]{kumar2012sampling}
Kumar, Sanjiv, Mohri, Mehryar, and Talwalkar, Ameet.
\newblock Sampling methods for the {N}ystr{\"o}m method.
\newblock \emph{The Journal of Machine Learning Research}, 13\penalty0
  (1):\penalty0 981--1006, 2012.

\bibitem[Kung(2014)]{kung2014kernel}
Kung, Sun~Yuan.
\newblock \emph{Kernel methods and machine learning}.
\newblock Cambridge University Press, 2014.

\bibitem[Kusse \& Westwig(2006)Kusse and Westwig]{kusse2006mathematical}
Kusse, Bruce~R and Westwig, Erik~A.
\newblock \emph{Mathematical physics: applied mathematics for scientists and
  engineers}.
\newblock Wiley-VCH, 2 edition, 2006.

\bibitem[Lanckriet et~al.(2004)Lanckriet, Cristianini, Bartlett, Ghaoui, and
  Jordan]{lanckriet2004learning}
Lanckriet, Gert~RG, Cristianini, Nello, Bartlett, Peter, Ghaoui, Laurent~El,
  and Jordan, Michael~I.
\newblock Learning the kernel matrix with semidefinite programming.
\newblock \emph{Journal of Machine learning research}, 5\penalty0
  (Jan):\penalty0 27--72, 2004.

\bibitem[Li et~al.(2010)Li, Kwok, and L{\"u}]{li2010making}
Li, Mu, Kwok, James Tin-Yau, and L{\"u}, Baoliang.
\newblock Making large-scale {N}ystr{\"o}m approximation possible.
\newblock In \emph{ICML 2010-Proceedings, 27th International Conference on
  Machine Learning}, pp.\  631, 2010.

\bibitem[Li et~al.(2015)Li, Swersky, and Zemel]{li2015generative}
Li, Yujia, Swersky, Kevin, and Zemel, Rich.
\newblock Generative moment matching networks.
\newblock In \emph{International Conference on Machine Learning}, pp.\
  1718--1727. PMLR, 2015.

\bibitem[Lim et~al.(2015)Lim, Kim, Park, and Jung]{lim2015double}
Lim, Woosang, Kim, Minhwan, Park, Haesun, and Jung, Kyomin.
\newblock Double {N}ystr{\"o}m method: An efficient and accurate {N}ystr{\"o}m
  scheme for large-scale data sets.
\newblock In \emph{International Conference on Machine Learning}, pp.\
  1367--1375. PMLR, 2015.

\bibitem[Lim et~al.(2018)Lim, Du, Dai, Jung, Song, and Park]{lim2018multi}
Lim, Woosang, Du, Rundong, Dai, Bo, Jung, Kyomin, Song, Le, and Park, Haesun.
\newblock Multi-scale {N}ystr{\"o}m method.
\newblock In \emph{International Conference on Artificial Intelligence and
  Statistics}, pp.\  68--76. PMLR, 2018.

\bibitem[Ma(2003)]{ma2003function}
Ma, Junshui.
\newblock Function replacement vs. kernel trick.
\newblock \emph{Neurocomputing}, 50:\penalty0 479--483, 2003.

\bibitem[Mercer(1909)]{mercer1909functions}
Mercer, J.
\newblock Functions of positive and negative type and their connection with the
  theory of integral equations.
\newblock \emph{Philosophical Transactions of the Royal Society}, A\penalty0
  (209):\penalty0 415--446, 1909.

\bibitem[Mika et~al.(1999)Mika, Ratsch, Weston, Scholkopf, and
  Mullers]{mika1999fisher}
Mika, Sebastian, Ratsch, Gunnar, Weston, Jason, Scholkopf, Bernhard, and
  Mullers, Klaus-Robert.
\newblock Fisher discriminant analysis with kernels.
\newblock In \emph{Neural networks for signal processing IX: Proceedings of the
  1999 IEEE signal processing society workshop}, pp.\  41--48. Ieee, 1999.

\bibitem[Minh et~al.(2006)Minh, Niyogi, and Yao]{minh2006mercer}
Minh, Ha~Quang, Niyogi, Partha, and Yao, Yuan.
\newblock Mercer’s theorem, feature maps, and smoothing.
\newblock In \emph{International Conference on Computational Learning Theory},
  pp.\  154--168. Springer, 2006.

\bibitem[Muandet et~al.(2016)Muandet, Fukumizu, Sriperumbudur, and
  Sch{\"o}lkopf]{muandet2016kernel}
Muandet, Krikamol, Fukumizu, Kenji, Sriperumbudur, Bharath, and Sch{\"o}lkopf,
  Bernhard.
\newblock Kernel mean embedding of distributions: A review and beyond.
\newblock \emph{arXiv preprint arXiv:1605.09522}, 2016.

\bibitem[M{\"u}ller(1997)]{muller1997integral}
M{\"u}ller, Alfred.
\newblock Integral probability metrics and their generating classes of
  functions.
\newblock \emph{Advances in Applied Probability}, pp.\  429--443, 1997.

\bibitem[M{\"u}ller et~al.(2018)M{\"u}ller, Mika, Tsuda, and
  Sch{\"o}lkopf]{muller2018introduction}
M{\"u}ller, Klaus-Robert, Mika, Sebastian, Tsuda, Koji, and Sch{\"o}lkopf,
  Bernhard.
\newblock An introduction to kernel-based learning algorithms.
\newblock \emph{Handbook of Neural Network Signal Processing}, 2018.

\bibitem[Narici \& Beckenstein(2010)Narici and
  Beckenstein]{narici2010topological}
Narici, Lawrence and Beckenstein, Edward.
\newblock \emph{Topological vector spaces}.
\newblock CRC Press, 2010.

\bibitem[Noack \& Sethian(2021)Noack and Sethian]{noack2021advanced}
Noack, Marcus~M and Sethian, James~A.
\newblock Advanced stationary and non-stationary kernel designs for
  domain-aware {Gaussian} processes.
\newblock \emph{arXiv preprint arXiv:2102.03432}, 2021.

\bibitem[Novak et~al.(2018)Novak, Ullrich, Wo{\'z}niakowski, and
  Zhang]{novak2018reproducing}
Novak, Erich, Ullrich, Mario, Wo{\'z}niakowski, Henryk, and Zhang, Shun.
\newblock Reproducing kernels of {Sobolev} spaces on $\mathbb{R}^d$ and
  applications to embedding constants and tractability.
\newblock \emph{Analysis and Applications}, 16\penalty0 (05):\penalty0
  693--715, 2018.

\bibitem[Nystr{\"o}m(1930)]{nystrom1930praktische}
Nystr{\"o}m, Evert~J.
\newblock {\"U}ber die praktische aufl{\"o}sung von integralgleichungen mit
  anwendungen auf randwertaufgaben.
\newblock \emph{Acta Mathematica}, 54\penalty0 (1):\penalty0 185--204, 1930.

\bibitem[Oldford(2018)]{oldford2018lecture}
Oldford, Wayne.
\newblock Lecture: Recasting principal components.
\newblock Lecture notes for Data Visualization, Department of Statistics and
  Actuarial Science, University of Waterloo, 2018.

\bibitem[Ormoneit \& Sen(2002)Ormoneit and Sen]{ormoneit2002kernel}
Ormoneit, Dirk and Sen, {\'S}aunak.
\newblock Kernel-based reinforcement learning.
\newblock \emph{Machine learning}, 49\penalty0 (2):\penalty0 161--178, 2002.

\bibitem[Orr(1996)]{orr1996introduction}
Orr, Mark J.~L.
\newblock Introduction to radial basis function networks.
\newblock Technical report, Center for Cognitive Science, University of
  Edinburgh, 1996.

\bibitem[Pan et~al.(2008)Pan, Kwok, and Yang]{pan2008transfer}
Pan, Sinno~Jialin, Kwok, James~T, and Yang, Qiang.
\newblock Transfer learning via dimensionality reduction.
\newblock In \emph{AAAI}, volume~8, pp.\  677--682, 2008.

\bibitem[Parseval~des Chenes(1806)]{parseval1806memoires}
Parseval~des Chenes, MA.
\newblock M{\'e}moires pr{\'e}sent{\'e}s {\`a} l’institut des sciences,
  lettres et arts, par divers savans, et lus dans ses assembl{\'e}es.
\newblock \emph{Sciences, math{\'e}matiques et physiques (Savans
  {\'e}trangers)}, 1:\penalty0 638, 1806.

\bibitem[Perlibakas(2004)]{perlibakas2004distance}
Perlibakas, Vytautas.
\newblock Distance measures for {PCA}-based face recognition.
\newblock \emph{Pattern recognition letters}, 25\penalty0 (6):\penalty0
  711--724, 2004.

\bibitem[Platt(2005)]{platt2005fastmap}
Platt, John.
\newblock {FastMap}, {MetricMap}, and landmark {MDS} are all {Nystrom}
  algorithms.
\newblock In \emph{AISTATS}, 2005.

\bibitem[Prugovecki(1982)]{prugovecki1982quantum}
Prugovecki, Eduard.
\newblock \emph{Quantum mechanics in {Hilbert} space}.
\newblock Academic Press, 1982.

\bibitem[Reed \& Simon(1972)Reed and Simon]{reed1972methods}
Reed, Michael and Simon, Barry.
\newblock \emph{Methods of modern mathematical physics: Functional analysis}.
\newblock Academic Press, 1972.

\bibitem[Renardy \& Rogers(2006)Renardy and Rogers]{renardy2006introduction}
Renardy, Michael and Rogers, Robert~C.
\newblock \emph{An introduction to partial differential equations}, volume~13.
\newblock Springer Science \& Business Media, 2006.

\bibitem[Rennie(2005)]{rennie2005how}
Rennie, Jason.
\newblock How to normalize a kernel matrix.
\newblock Technical report, MIT Computer Science \& Artificial Intelligence
  Lab, 2005.

\bibitem[Rojo-{\'A}lvarez et~al.(2018)Rojo-{\'A}lvarez,
  Mart{\'\i}nez-Ram{\'o}n, Mar{\'\i}, and Camps-Valls]{rojo2018digital}
Rojo-{\'A}lvarez, Jos{\'e}~Luis, Mart{\'\i}nez-Ram{\'o}n, Manel, Mar{\'\i},
  Jordi~Mu{\~n}oz, and Camps-Valls, Gustavo.
\newblock \emph{Digital signal processing with Kernel methods}.
\newblock Wiley Online Library, 2018.

\bibitem[Rudi et~al.(2015)Rudi, Camoriano, and Rosasco]{rudi2015less}
Rudi, Alessandro, Camoriano, Raffaello, and Rosasco, Lorenzo.
\newblock Less is more: {N}ystr{\"o}m computational regularization.
\newblock In \emph{Advances in neural information processing systems}, pp.\
  1657--1665, 2015.

\bibitem[Rudin(2012)]{rudin2012prediction}
Rudin, Cynthia.
\newblock Prediction: Machine learning and statistics ({MIT} 15.097), lecture
  on kernels.
\newblock Technical report, Massachusetts Institute of Technology, 2012.

\bibitem[Rupp(2015)]{rupp2015machine}
Rupp, Matthias.
\newblock Machine learning for quantum mechanics in a nutshell.
\newblock \emph{International Journal of Quantum Chemistry}, 115\penalty0
  (16):\penalty0 1058--1073, 2015.

\bibitem[Saul et~al.(2006)Saul, Weinberger, Sha, Ham, and
  Lee]{saul2006spectral}
Saul, Lawrence~K, Weinberger, Kilian~Q, Sha, Fei, Ham, Jihun, and Lee,
  Daniel~D.
\newblock Spectral methods for dimensionality reduction.
\newblock \emph{Semi-supervised learning}, 3, 2006.

\bibitem[Saxe(2002)]{saxe2002beginning}
Saxe, Karen.
\newblock \emph{Beginning functional analysis}.
\newblock Springer, 2002.

\bibitem[Schlichtharle(2011)]{schlichtharle2011digital}
Schlichtharle, Dietrich.
\newblock \emph{Digital Filters: Basics and Design}.
\newblock Springer, 2 edition, 2011.

\bibitem[Schmidt(1908)]{schmidt1908auflosung}
Schmidt, Erhard.
\newblock {\"U}ber die aufl{\"o}sung linearer gleichungen mit unendlich vielen
  unbekannten.
\newblock \emph{Rendiconti del Circolo Matematico di Palermo (1884-1940)},
  25\penalty0 (1):\penalty0 53--77, 1908.

\bibitem[Sch{\"o}lkopf(2001)]{scholkopf2001kernel}
Sch{\"o}lkopf, Bernhard.
\newblock The kernel trick for distances.
\newblock \emph{Advances in neural information processing systems}, pp.\
  301--307, 2001.

\bibitem[Sch{\"o}lkopf \& Smola(2002)Sch{\"o}lkopf and
  Smola]{scholkopf2002learning}
Sch{\"o}lkopf, Bernhard and Smola, Alexander~J.
\newblock \emph{Learning with kernels: support vector machines, regularization,
  optimization, and beyond}.
\newblock MIT press, 2002.

\bibitem[Sch{\"o}lkopf et~al.(1997{\natexlab{a}})Sch{\"o}lkopf, Smola, and
  M{\"u}ller]{scholkopf1997kernel}
Sch{\"o}lkopf, Bernhard, Smola, Alexander, and M{\"u}ller, Klaus-Robert.
\newblock Kernel principal component analysis.
\newblock In \emph{International conference on artificial neural networks},
  pp.\  583--588. Springer, 1997{\natexlab{a}}.

\bibitem[Sch{\"o}lkopf et~al.(1997{\natexlab{b}})Sch{\"o}lkopf, Sung, Burges,
  Girosi, Niyogi, Poggio, and Vapnik]{scholkopf1997comparing}
Sch{\"o}lkopf, Bernhard, Sung, Kah-Kay, Burges, Christopher~JC, Girosi,
  Federico, Niyogi, Partha, Poggio, Tomaso, and Vapnik, Vladimir.
\newblock Comparing support vector machines with {Gaussian} kernels to radial
  basis function classifiers.
\newblock \emph{IEEE transactions on Signal Processing}, 45\penalty0
  (11):\penalty0 2758--2765, 1997{\natexlab{b}}.

\bibitem[Sch{\"o}lkopf et~al.(1998)Sch{\"o}lkopf, Smola, and
  M{\"u}ller]{scholkopf1998nonlinear}
Sch{\"o}lkopf, Bernhard, Smola, Alexander, and M{\"u}ller, Klaus-Robert.
\newblock Nonlinear component analysis as a kernel eigenvalue problem.
\newblock \emph{Neural computation}, 10\penalty0 (5):\penalty0 1299--1319,
  1998.

\bibitem[Sch{\"o}lkopf et~al.(1999{\natexlab{a}})Sch{\"o}lkopf, Burges, and
  Smola]{scholkopf1999advances}
Sch{\"o}lkopf, Bernhard, Burges, Christopher~JC, and Smola, Alexander~J.
\newblock \emph{Advances in kernel methods: support vector learning}.
\newblock MIT press, 1999{\natexlab{a}}.

\bibitem[Sch{\"o}lkopf et~al.(1999{\natexlab{b}})Sch{\"o}lkopf, Mika, Burges,
  Knirsch, Muller, Ratsch, and Smola]{scholkopf1999input}
Sch{\"o}lkopf, Bernhard, Mika, Sebastian, Burges, Chris~JC, Knirsch, Philipp,
  Muller, K-R, Ratsch, Gunnar, and Smola, Alexander~J.
\newblock Input space versus feature space in kernel-based methods.
\newblock \emph{IEEE transactions on neural networks}, 10\penalty0
  (5):\penalty0 1000--1017, 1999{\natexlab{b}}.

\bibitem[Scott(1992)]{scott1992multivariate}
Scott, David~W.
\newblock \emph{Multivariate density estimation: theory, practice, and
  visualization}.
\newblock John Wiley \& Sons, 1992.

\bibitem[Sejdinovic et~al.(2013)Sejdinovic, Sriperumbudur, Gretton, and
  Fukumizu]{sejdinovic2013equivalence}
Sejdinovic, Dino, Sriperumbudur, Bharath, Gretton, Arthur, and Fukumizu, Kenji.
\newblock Equivalence of distance-based and {RKHS}-based statistics in
  hypothesis testing.
\newblock \emph{The Annals of Statistics}, pp.\  2263--2291, 2013.

\bibitem[Shawe-Taylor \& Cristianini(2004)Shawe-Taylor and
  Cristianini]{shawe2004kernel}
Shawe-Taylor, John and Cristianini, Nello.
\newblock \emph{Kernel methods for pattern analysis}.
\newblock Cambridge university press, 2004.

\bibitem[Silva et~al.(2006)Silva, Marques, and Lemos]{silva2006selecting}
Silva, Jorge, Marques, Jorge, and Lemos, Jo{\~a}o.
\newblock Selecting landmark points for sparse manifold learning.
\newblock In \emph{Advances in neural information processing systems}, pp.\
  1241--1248, 2006.

\bibitem[Simon-Gabriel \& Sch{\"o}lkopf(2018)Simon-Gabriel and
  Sch{\"o}lkopf]{simon2018kernel}
Simon-Gabriel, Carl-Johann and Sch{\"o}lkopf, Bernhard.
\newblock Kernel distribution embeddings: Universal kernels, characteristic
  kernels and kernel metrics on distributions.
\newblock \emph{The Journal of Machine Learning Research}, 19\penalty0
  (1):\penalty0 1708--1736, 2018.

\bibitem[Simon-Gabriel et~al.(2020)Simon-Gabriel, Barp, and
  Mackey]{simon2020metrizing}
Simon-Gabriel, Carl-Johann, Barp, Alessandro, and Mackey, Lester.
\newblock Metrizing weak convergence with maximum mean discrepancies.
\newblock \emph{arXiv preprint arXiv:2006.09268}, 2020.

\bibitem[Smola et~al.(2007)Smola, Gretton, Song, and
  Sch{\"o}lkopf]{smola2007hilbert}
Smola, Alex, Gretton, Arthur, Song, Le, and Sch{\"o}lkopf, Bernhard.
\newblock A {Hilbert} space embedding for distributions.
\newblock In \emph{International Conference on Algorithmic Learning Theory},
  pp.\  13--31. Springer, 2007.

\bibitem[Smola \& Sch{\"o}lkopf(1998)Smola and
  Sch{\"o}lkopf]{smola1998learning}
Smola, Alex~J and Sch{\"o}lkopf, Bernhard.
\newblock \emph{Learning with kernels}, volume~4.
\newblock Citeseer, 1998.

\bibitem[Song(2008)]{song2008learning}
Song, Le.
\newblock \emph{Learning via Hilbert space embedding of distributions}.
\newblock PhD thesis, The University of Sydney, 2008.

\bibitem[Sriperumbudur et~al.(2010)Sriperumbudur, Gretton, Fukumizu,
  Sch{\"o}lkopf, and Lanckriet]{sriperumbudur2010hilbert}
Sriperumbudur, Bharath~K, Gretton, Arthur, Fukumizu, Kenji, Sch{\"o}lkopf,
  Bernhard, and Lanckriet, Gert~RG.
\newblock Hilbert space embeddings and metrics on probability measures.
\newblock \emph{The Journal of Machine Learning Research}, 11:\penalty0
  1517--1561, 2010.

\bibitem[Sriperumbudur et~al.(2011)Sriperumbudur, Fukumizu, and
  Lanckriet]{sriperumbudur2011universality}
Sriperumbudur, Bharath~K, Fukumizu, Kenji, and Lanckriet, Gert~RG.
\newblock Universality, characteristic kernels and {RKHS} embedding of
  measures.
\newblock \emph{Journal of Machine Learning Research}, 12\penalty0 (7), 2011.

\bibitem[Steinwart(2001)]{steinwart2001influence}
Steinwart, Ingo.
\newblock On the influence of the kernel on the consistency of support vector
  machines.
\newblock \emph{Journal of machine learning research}, 2\penalty0
  (Nov):\penalty0 67--93, 2001.

\bibitem[Steinwart(2002)]{steinwart2002support}
Steinwart, Ingo.
\newblock Support vector machines are universally consistent.
\newblock \emph{Journal of Complexity}, 18\penalty0 (3):\penalty0 768--791,
  2002.

\bibitem[Steinwart \& Christmann(2008)Steinwart and
  Christmann]{steinwart2008support}
Steinwart, Ingo and Christmann, Andreas.
\newblock \emph{Support vector machines}.
\newblock Springer Science \& Business Media, 2008.

\bibitem[Strang(1993)]{strang1993fundamental}
Strang, Gilbert.
\newblock The fundamental theorem of linear algebra.
\newblock \emph{The American Mathematical Monthly}, 100\penalty0 (9):\penalty0
  848--855, 1993.

\bibitem[Strange \& Zwiggelaar(2014)Strange and Zwiggelaar]{strange2014open}
Strange, Harry and Zwiggelaar, Reyer.
\newblock \emph{Open Problems in Spectral Dimensionality Reduction}.
\newblock Springer, 2014.

\bibitem[Tenenbaum et~al.(2000)Tenenbaum, De~Silva, and
  Langford]{tenenbaum2000global}
Tenenbaum, Joshua~B, De~Silva, Vin, and Langford, John~C.
\newblock A global geometric framework for nonlinear dimensionality reduction.
\newblock \emph{Science}, 290\penalty0 (5500):\penalty0 2319--2323, 2000.

\bibitem[Vandenberghe \& Boyd(1996)Vandenberghe and
  Boyd]{vandenberghe1996semidefinite}
Vandenberghe, Lieven and Boyd, Stephen.
\newblock Semidefinite programming.
\newblock \emph{SIAM review}, 38\penalty0 (1):\penalty0 49--95, 1996.

\bibitem[Vapnik(1995)]{vapnik1995nature}
Vapnik, Vladimir.
\newblock \emph{The nature of statistical learning theory}.
\newblock Springer science \& business media, 1995.

\bibitem[Vapnik \& Chervonenkis(1974)Vapnik and Chervonenkis]{vapnik1974theory}
Vapnik, Vladimir and Chervonenkis, Alexey.
\newblock \emph{Theory of pattern recognition}.
\newblock Nauka, Moscow, 1974.

\bibitem[Wahba(1990)]{wahba1990spline}
Wahba, Grace.
\newblock \emph{Spline models for observational data}.
\newblock SIAM, 1990.

\bibitem[Wand \& Jones(1994)Wand and Jones]{wand1994kernel}
Wand, Matt~P and Jones, M~Chris.
\newblock \emph{Kernel smoothing}.
\newblock CRC press, 1994.

\bibitem[Wang et~al.(2010{\natexlab{a}})Wang, Rege, Dong, and
  Ding]{wang2010low}
Wang, Lijun, Rege, Manjeet, Dong, Ming, and Ding, Yongsheng.
\newblock Low-rank kernel matrix factorization for large-scale evolutionary
  clustering.
\newblock \emph{IEEE Transactions on Knowledge and Data Engineering},
  24\penalty0 (6):\penalty0 1036--1050, 2010{\natexlab{a}}.

\bibitem[Wang et~al.(2010{\natexlab{b}})Wang, Sha, and
  Jordan]{wang2010unsupervised}
Wang, Meihong, Sha, Fei, and Jordan, Michael.
\newblock Unsupervised kernel dimension reduction.
\newblock \emph{Advances in neural information processing systems},
  23:\penalty0 2379--2387, 2010{\natexlab{b}}.

\bibitem[Weinberger \& Saul(2006{\natexlab{a}})Weinberger and
  Saul]{weinberger2006introduction}
Weinberger, Kilian~Q and Saul, Lawrence~K.
\newblock An introduction to nonlinear dimensionality reduction by maximum
  variance unfolding.
\newblock In \emph{Proceedings of the AAAI Conference on Artificial
  Intelligence}, volume~6, pp.\  1683--1686, 2006{\natexlab{a}}.

\bibitem[Weinberger \& Saul(2006{\natexlab{b}})Weinberger and
  Saul]{weinberger2006unsupervised}
Weinberger, Kilian~Q and Saul, Lawrence~K.
\newblock Unsupervised learning of image manifolds by semidefinite programming.
\newblock \emph{International journal of computer vision}, 70\penalty0
  (1):\penalty0 77--90, 2006{\natexlab{b}}.

\bibitem[Weinberger et~al.(2005)Weinberger, Packer, and
  Saul]{weinberger2005nonlinear}
Weinberger, Kilian~Q, Packer, Benjamin, and Saul, Lawrence~K.
\newblock Nonlinear dimensionality reduction by semidefinite programming and
  kernel matrix factorization.
\newblock In \emph{AISTATS}, 2005.

\bibitem[Williams \& Seeger(2000)Williams and Seeger]{williams2000effect}
Williams, Christopher and Seeger, Matthias.
\newblock The effect of the input density distribution on kernel-based
  classifiers.
\newblock In \emph{Proceedings of the 17th international conference on machine
  learning}, 2000.

\bibitem[Williams \& Seeger(2001)Williams and Seeger]{williams2001using}
Williams, Christopher and Seeger, Matthias.
\newblock Using the {N}ystr{\"o}m method to speed up kernel machines.
\newblock In \emph{Proceedings of the 14th annual conference on neural
  information processing systems}, number CONF, pp.\  682--688, 2001.

\bibitem[Williams \& Barber(1998)Williams and Barber]{williams1998bayesian}
Williams, Christopher~KI and Barber, David.
\newblock Bayesian classification with {Gaussian} processes.
\newblock \emph{IEEE Transactions on Pattern Analysis and Machine
  Intelligence}, 20\penalty0 (12):\penalty0 1342--1351, 1998.

\bibitem[Yan et~al.(2005)Yan, Xu, Zhang, and Zhang]{yan2005graph}
Yan, Shuicheng, Xu, Dong, Zhang, Benyu, and Zhang, Hong-Jiang.
\newblock Graph embedding: A general framework for dimensionality reduction.
\newblock In \emph{2005 IEEE Computer Society Conference on Computer Vision and
  Pattern Recognition (CVPR'05)}, volume~2, pp.\  830--837. IEEE, 2005.

\bibitem[Zhang et~al.(2007)Zhang, Marsza{\l}ek, Lazebnik, and
  Schmid]{zhang2007local}
Zhang, Jianguo, Marsza{\l}ek, Marcin, Lazebnik, Svetlana, and Schmid, Cordelia.
\newblock Local features and kernels for classification of texture and object
  categories: A comprehensive study.
\newblock \emph{International journal of computer vision}, 73\penalty0
  (2):\penalty0 213--238, 2007.

\bibitem[Zhang \& Kwok(2010)Zhang and Kwok]{zhang2010clustered}
Zhang, Kai and Kwok, James~T.
\newblock Clustered {N}ystr{\"o}m method for large scale manifold learning and
  dimension reduction.
\newblock \emph{IEEE Transactions on Neural Networks}, 21\penalty0
  (10):\penalty0 1576--1587, 2010.

\bibitem[Zhang et~al.(2008)Zhang, Tsang, and Kwok]{zhang2008improved}
Zhang, Kai, Tsang, Ivor~W, and Kwok, James~T.
\newblock Improved {N}ystr{\"o}m low-rank approximation and error analysis.
\newblock In \emph{Proceedings of the 25th international conference on machine
  learning}, pp.\  1232--1239, 2008.

\end{thebibliography}
\bibliographystyle{icml2016}

\end{document}